\documentclass[11pt]{article}
\input{macros}
\usepackage{mathrsfs}
\allowdisplaybreaks

\begin{document}

\title{Overparametrized linear dimensionality reductions:\\
From projection pursuit to two-layer neural networks\footnote{An earlier version of this paper was accepted for presentation at the Conference on Learning Theory (COLT) 2022, with the title ``High-Dimensional Projection Pursuit:
		Outer Bounds and Applications to Interpolation in Neural Networks."}}
\author{Andrea Montanari\thanks{Department of Electrical Engineering and Department of Statistics, 
		Stanford University} \;\; and \;\; Kangjie Zhou\thanks{Department of Statistics, 
		Stanford University}}
\date{\today}
\maketitle

\begin{abstract}
	Given a cloud of $n$ data points in $\R^d$, consider
	all projections onto $m$-dimensional subspaces of $\R^d$ and, 
	for each such projection, the empirical distribution of the projected points.
	What does this collection of probability distributions look like when $n,d$ grow large?
	
	We consider this question under the null model in which the points are i.i.d. 
	standard Gaussian vectors, focusing on the asymptotic regime in which $n,d\to\infty$,
	with $n/d\to\alpha\in (0,\infty)$, while $m$ is fixed. Denoting by $\cuF_{m, \alpha}$
	the set of probability distributions in $\R^m$ that arise as low-dimensional projections
	in this limit, we establish new inner and outer bounds on $\cuF_{m, \alpha}$. In
	particular, we characterize the Wasserstein radius of $\cuF_{m,\alpha}$ up to constant multiplicative factors, and determine it exactly for $m=1$. We 
	also prove sharp bounds in terms of Kullback-Leibler divergence 
	and R\'{e}nyi information dimension.
	
	The previous question has application to unsupervised learning methods, such as projection pursuit
	and independent component analysis. We introduce a version of the same problem that is relevant for 
	supervised learning, and prove a sharp Wasserstein radius bound. As an application, we
	establish an upper bound on the interpolation threshold of two-layers neural networks with $m$
	hidden neurons.
\end{abstract}


\tableofcontents

\section{Introduction and main results}\label{sec:intro}

\subsection{A null model for unsupervised learning}

Given $n$ data points $\xx_1,\dots,\xx_n\in\R^d$, Friedman and Tukey \cite{friedman1974projection} proposed to look for interesting structures by plotting histograms of their projections onto $m$-dimensional subspaces (with $m\in\{1,2\}$). They suggested to seek for 
projections that maximize a certain index of clustering of the resulting histograms. Diaconis and Freedman \cite{diaconis1984asymptotics} showed that under incoherence conditions 
on the points $\{\xx_i\}_{i\le n}$, most one-dimensional projections are nearly Gaussian.
They also studied the null model in which $\{\xx_i\}_{i\le n}\sim_{\iid}\normal(\bzero,\id_d)$,
in the high-dimensional setting $n,d\to\infty$. 
They proved that, if $n/d\to \infty$ then the `least Gaussian' one-dimensional projection converges 
to a standard Gaussian (in Kolmogorov-Smirnov distance), while, if $n/d\to \alpha\in(0,\infty)$,
it does not.

We will study the (Gaussian) null model of  \cite{diaconis1984asymptotics},
under the proportional asymptotics  $n/d\to \alpha\in(0,\infty)$. 
Let $\cuP (\R^{m})$ denote the space of all probability measures on $\R^{m}$, and 
$\{\xx_i\}_{i\le n}\sim_{\iid}\normal(\bzero,\id_d)$. Denote by $\XX\in\R^{n\times d}$
the matrix with rows $\xx_i^\top$, $i\le n$.  
We say that $P \in \cuP (\R^{m})$ is $(\alpha, m)$-feasible if there exists a 
sequence of random orthogonal matrices $\WW = \WW_n (\XX,\omega) \in \R^{d \times m}$ ($\WW^\top \WW = \id_m$)
such that the empirical distribution of the projections $\{\WW^{\top}\xx_i\}_{i\le n}$
converges weakly to $P$, in probability with respect to the randomness in $\XX,\WW$.
(Here $\omega$ denotes additional randomness that can be used in the construction of $\WW$.)
In formulas:
\begin{align}
	\cuF_{m,\alpha}:= \Big\{P \in \cuP (\R^{m}):\; & \exists
	\WW = \WW_n (\XX,\omega), \ \WW^\top \WW = \id_m \\
	& \mbox{ such that }
	\frac{1}{n} \sum_{i=1}^{n} \delta_{\WW^\top \xx_i} \stackrel{w}{\Rightarrow} P\; \mbox{ in probability }
	\Big\}\, .\label{eq:FeasibleFirst}
\end{align}
(See Section \ref{sec:MainUnsupervised} for clarifications about this definition.)

Understanding $\cuF_{m,\alpha}$ is  relevant for a broad array of
unsupervised learning methods. Indeed, non-Gaussian 
projections are sought by independent component analysis \cite{hyvarinen2000independent}, 
blind deconvolution \cite{levin2011understanding}, and related methods
\cite{blanchard2006search,sasaki2016non,loperfido2018skewness}.

The problem of characterizing $\cuF_{m,\alpha}$ was recently studied by Bickel, Kur and Nadler \cite{bickel2018projection}, for the case $m=1$. 
Denoting by $\mu_2(P) = \int x^2\, P(\de x)$ the second moment of 
$P$, and by $d_{\sKS}(P_1,P_2)$ the Kolmogorov-Smirnov distance between $P_1$ and $P_2$,
\cite{bickel2018projection} proved the following:
\begin{align}
	\alpha\le 1&\;\; \Rightarrow \big\{P:\, \mu_2(P)\le \alpha^{-1}-1\big\}
	\subseteq \cuF_{1, \alpha}\subseteq 
	\big\{P:\, \mu_2(P)\le (\alpha^{-1/2}+1)^2\big\}\, ,\label{eq:Bickel1}\\
	\alpha> 1&\;\; \Rightarrow 
	\cuF_{1, \alpha}\subseteq 
	\big\{P:\, d_{\sKS}(P,\normal(0,1))\le C\sqrt{\alpha^{-1}\log\alpha}\big\}\, .
	\label{eq:Bickel2}
\end{align}
The paper \cite{bickel2018projection} also proved that certain mixtures of a Gaussian and a non-Gaussian component are
feasible for $\alpha>1$, leading the authors to conjecture that this is the case for all 
the feasible distributions. 

We will establish several new results on this model:
\begin{description}
	\item[Wasserstein radius for $m=1$.] Denoting by $W_2(P_1,P_2)$ the second Wasserstein distance between two probability measures $P_1$ and $P_2$,
	we prove that $\sup\{W_2(P,\normal(0,1)): P\in \cuF_{1, \alpha}\}=1/\sqrt{\alpha}$.
	Note that this implies as a corollary the outer bound in Eq.~\eqref{eq:Bickel1}, but it is significantly
	stronger.
	\item[KL-Wasserstein outer bound.] We show that, for any $m \in \mathbb{N}$, $\cuF_{m,\alpha}$ is contained 
	in a (small) $W_2$ neighborhood of the set of distributions $Q$ such that
	$D_{\rm KL}(Q\|\normal(\bzero,\id_m))\le C m \alpha^{-1}$, with $D_{\rm KL}$ the 
	Kullback-Leibler (KL) divergence (relative entropy). As a corollary, this bound implies an
	upper bound on  the $W_2$ radius of $\cuF_{m,\alpha}$ that is tight within 
	a constant factor $C$. However, the KL bound is significantly tighter than the $W_2$ bound
	for distributions $P$ that are less regular than Gaussians.
	\item[Information dimension bound.] Denoting by  $\lod(P)$
	the lower information dimension of $P$ (see Definition~\ref{def:info_dim}), we prove that
	$\cuF_{m,\alpha}$ is contained in $\{P:  \lod(P) \ge m(1-1/\alpha)\}$ for $\alpha > 1$.
	
	For instance, if $P$ is supported on an $s$-dimensional smooth manifold in $\R^m$,
	then $\lod(P)\le s$, and therefore $P$  is $(\alpha,m)$-feasible only if $\alpha\le m/(m-s)$.
	\item[$\chi^2$-KL divergence inner bound.] We establish an inner bound for the feasibility 
	set $\cuF_{m, \alpha}$, which is expressed in terms  of $\chi^2$ and KL divergence.
	For probability distributions $P$ for which these two distances from $\normal(\bzero,\id_m)$
	are comparable, this inner bound implies a lower bound on the maximum $\alpha$
	for which $P$ is feasible, which is tight up to a constant factor. 
	As a comparison, the inner bound result in Eq.~\eqref{eq:Bickel1} only applies to the
	 case  $\alpha \le 1$.
	 \item[Satisfiability threshold of negative perceptron.] Interestingly, our general KL-Wasserstein outer bound can be applied to derive an asymptotically tight upper bound on the satisfiability threshold of the negative spherical perceptron problem. This derivation demonstrates that our KL-Wassertein outer bound is tight up to a constant factor.
\end{description}

\subsection{A null model for supervised learning}\label{sec:NullModel}

Our main motivation to revisit projection pursuit comes from supervised learning, and we establish their connection below.

To be definite, we consider a data model whereby
$\{ ( \xx_i, y_i ) \}_{i \in [n]}$ are i.i.d., with isotropic Gaussian covariates 
$\xx_i \sim \sN (\bzero, \id_d)$ and responses $y_i \in \{+1,-1\}$ depending on low-dimensional 
projections of the $\xx_i$'s. Namely, for $k\le d$, let $\VV \in \R^{d \times k}$ 
be an orthogonal matrix such that $\VV^\top \VV = \id_k$. 
We assume that the conditional distribution of $y_i$ given 
$\xx_i$ only depends on $\VV^\top \xx_i$:
\begin{align}\label{eq:SupervisedModel}
	\P \left( y_i =+1  \vert \xx_i \right) = \varphi(\VV^\top \xx_i )  \, ,
\end{align}
(and $\P( y_i = -1  \vert \xx_i) = 1-\varphi(\VV^\top \xx_i )$)
for a 
measurable function $\varphi: \R^{k} \to [0, 1]$. We notice in passing 
that this model can be easily generalized to continuous sub-Gaussian responses $y_i$,
and our proofs apply to this case as well.

In many supervised learning methods, one seeks a model that only depends on a 
low-dimensional projection of the covariates. Fitting such a model requires
to consider the possible distributions
over $\{+1,-1\}\times\R^m$ that can be obtained by projecting the
covariates onto an $m$-dimensional subspace of $\R^d$. 

This motivates the following definition.
We say that $P \in \cuP (\{+1,-1\}\times \R^{m})$ is $(\alpha, m)$-feasible if there exists a 
sequence of random orthogonal matrices $\WW = \WW_n (\XX,\yy,\omega) \in \R^{d \times m}$
($\WW^\top \WW = \id_m$)
such that the empirical distribution of the pairs $\{(y_i,\WW^{\top}\xx_i)\}_{i\le n}$
converges weakly to $P$ (in probability with respect to the randomness in $\XX$, $\yy$ and $\WW$).
In formulas:
\begin{align*}
	\cuF^{\varphi}_{m,\alpha}:= \Big\{P \in \cuP ( \{ \pm 1 \} \times \R^{m}):\; & \exists
	\WW= \WW_n (\XX,\yy,\omega), \ \WW^\top \WW = \id_m \\
	& \mbox{ such that }
	\frac{1}{n} \sum_{i=1}^{n} \delta_{\left(y_i,\xx_i^\top \WW \right)} \stackrel{w}{\Rightarrow} P \mbox{ in probability }
	\Big\}\, .
\end{align*}

Characterizing the set $\cuF^{\varphi}_{m,\alpha}$ gives access to a number of statistical
quantities of interest. In this paper, we present the following results.
\begin{description}
	\item[General ERM asymptotics.] We consider a class of empirical risk minimization 
	(ERM) problems over functions $f:\R^d\to \R$ of the form
	$f(\xx) = h(\WW^\top\xx)$, where both $h$ and $\WW$ are optimized over.
	We show that the asymptotics of the minimum empirical
	risk can be expressed in terms of a variational problem over the feasibility set $\cuF_{m,\alpha}^{\varphi}$.
	\item[Wasserstein bound for $m=1$.] We prove an outer bound on  $\cuF_{1,\alpha}^{\varphi}$
	for general $k=O(1)$,
	which generalizes the Wasserstein radius result obtained in the unsupervised setting.
	In fact this outer bound characterizes the maximum $W_2$ distance between the 
	empirical distribution of one-dimensional projections and the expected distribution.
	\item[KL-Wasserstein outer bound.] We extend our KL-Wasserstein outer bound result for unsupervised projection pursuit to the supervised setting, and similarly derive an upper bound on the Wassertein radius of the feasible set for general $m > 1$.
	\item[Interpolation for two-layer networks.]
	As a corollary to the previous result, we prove that a neural network with two-layers and $m$
	hidden neurons can separate $n$ data points in $d$ dimensions with margin $\kappa$ only if
	$md\ge C\kappa^2n$ (where the limit $n,d\to\infty$ with $n/d\to\alpha$ is understood).
	Earlier bounds only required $md\ge C n/\log(d/\kappa)$.
	\item[Margin distributions for linear classifier.] We demonstrate the tightness of our 
	$W_2$ bound by deriving the asymptotic distribution of the margins in linear max-margin 
	classification.
\end{description}
The rest of this paper is organized as follows. We formally state our results for 
unsupervised and supervised learning, respectively, in Section
\ref{sec:MainUnsupervised}, Section \ref{sec:2nd_moment}, and Section \ref{sec:MainSupervised}. We
describe some of the proof ideas in Section \ref{sec:ProofIdeas}, with actual proofs deferred to 
the appendices.

\subsection*{Notations}\label{sec:notation}
We denote by $\delta_{x}$ the Dirac measure at $x \in \cX$, where $\cX$ is a measurable space. The set of all probability measures on $\cX$ is denoted as $\cuP(\cX)$.
For a random variable $U$, $\Law (U)$
denotes the probability distribution of $U$. For a positive integer $n$, 
we let $[n]$ be the set $\{ 1, 2, \cdots, n \}$. For two measures $P$ and $Q$, we use 
$P \otimes Q$ to denote their product measure.

We consistently use lowercase letters to denote scalars, boldface lowercase 
letters to denote vectors, and boldface uppercase letters to denote matrices. 
For a scalar $a$, we write $a_+ = \max (a, 0)$ and $a_- = \max(- a, 0)$. For two vectors 
$\uu$ and $\vv$, $\langle \uu, \vv \rangle$ denotes their scalar product.
We use $\norm{\uu}_2$ to denote the Euclidean norm of a vector $\uu$. We denote by 
$\S^{d - 1}$ the unit sphere in $\R^d$.

We always use $\Phi$ and $\phi$ to denote the CDF and PDF of a standard normal variable,
respectively. We write $X \perp Y$ if $X$ and $Y$ are two independent random variables. 
We denote by $O(d, m)$ the set of all $d \times m$ orthogonal matrices $\WW$ such that 
$\WW^\top \WW = \id_m$.

Finally, whenever clear from the context, we identify a vector $\xx\in\R^k$
with its transpose $\xx^{\top}$: This reduces some notational burden, and amounts
to identifying $\R^k$ with its dual.

\section{Outer bounds: Unsupervised learning}\label{sec:MainUnsupervised}

Before stating our results, it is useful to recall the definition of feasibility set
\eqref{eq:FeasibleFirst}, and to clarify one element of this definition.
Note that $\hat{P}_{n,\WW}:=n^{-1}\sum_{i=1}^{n} \delta_{\WW^\top \xx_i}$
is a random probability distribution on $\R^m$. We say that $\hat{P}_{n,\WW}\stackrel{w}{\Rightarrow} P$
in probability if, for any $\veps>0$,  $\P(d_{W}(\hat{P}_{n,\WW},P)>\veps)\to 0$, where
$d_W$ is a distance that metrizes weak convergence. For instance $d_W$ can be taken to be
the bounded Lipschitz distance or the  L\'{e}vy-Prokhorov metric.

\subsection{Wasserstein radius for $m=1$}
Given two probability distributions $P,Q$ on $\R^m$, 
we denote by $W_2 (P,Q)$ the second Wasserstein distance between $P$ and $Q$.
Namely
\begin{align}
	W_2(P,Q) := \left( \inf_{\gamma\in \Gamma(P,Q)} \int\|\xx-\yy\|_2^2\gamma(\d \xx\times \d \yy) \right)^{1 / 2} \, ,
\end{align}
where the infimum is taken over the space $\Gamma(P,Q)$ of couplings of $(P,Q)$. In the following theorem, we show that the maximum $W_2$ distance between any $(\alpha, 1)$-feasible distribution and $\normal(0, 1)$ is $1 / \sqrt{\alpha}$.
\begin{thm}[Wasserstein radius for $m=1$]\label{thm:outer_bound_1_dim}
	Consider the case $m = 1$. Then for any $\alpha\in(0, \infty)$, we have
	\begin{equation}\label{eq:outer_bound_1_dim}
		\sup\big\{W_2 ( P, \sN (0, 1)):\; P\in\cuF_{1, \alpha}\big\} =\frac{1}{\sqrt{\alpha}}.
	\end{equation}
\end{thm}
\begin{rem}
	The supremum above is achieved by taking $P= \sN(0,(1+\alpha^{-1/2})^2)$.
	Indeed, as shown in the proof of Theorem~\ref{thm:outer_bound_1_dim}, this distribution is feasible by taking $\WW= (\vv_1(\XX))$,
	where $\vv_1(\XX)$ is the top right singular vector of $\XX$. 
\end{rem}

\subsection{KL-Wasserstein outer bound for general $m$}
In this section, we establish an outer bound on $\cuF_{m, \alpha}$, which is based on the $W_2$ metric and the KL divergence. Recall the definition of the KL divergence.
Given two probability measures $P$ and $Q$ on a measure space $\cX$, if $P$ is absolutely continuous 
with respect to $Q$, then
\begin{equation}\label{eq:KL_def}
	D_{\sKL} \left( P \Vert Q \right) = \int_{\cX} \log \left( \frac{\d P}{\d Q} \right) \d P,
\end{equation}
where $\d P / \d Q$ denotes the Radon-Nikodym derivative of $P$ with respect to $Q$. 
Otherwise, $D_{\rm KL}(P \Vert Q) = \infty$.
\begin{thm}[KL-Wasserstein outer bound]\label{thm:div_outer_bd}
	For $a, b > 0$, define the following neighborhood of $\sN (\bzero, \id_m)$:
	\begin{equation*}
		\cuS_m (a, b) = \Big\{ P \in \cuP (\R^m): \exists Q \in \cuP (\R^m), \ 
		\text{s.t.} \ W_2 (P, Q) \le a \ \text{and} \ D_{\rm KL}(Q \Vert \sN (\bzero, \id_m)) \le b \Big\}\, .
	\end{equation*}
	Then, there exist absolute constants $C > 0$ and $\veps_0 > 0$ such that for any  $\veps \in (0, \veps_0)$, we have
	\begin{equation*}
		\cuF_{m,\alpha} \subseteq \cuS_m \left( \sqrt{\frac{m}{\alpha}} \veps \sqrt{\log \left( \frac{C}{\veps} \right)}, \,\frac{m}{\alpha} \log \left( \frac{C}{\veps} \right) \right).
	\end{equation*}
\end{thm}

As a direct consequence of Theorem~\ref{thm:div_outer_bd}, we obtain the following:
\begin{thm}\label{thm:m_W2_outer_bd}
	There exists an absolute constant $C_1 > 0$ such that
	\begin{equation*}
		\cuF_{m,\alpha} \subseteq \left\{P \in \cuP (\R^m): \;\;  
		W_2 \left( P, \sN (\bzero, \id_m) \right) \le C_1 \sqrt{\frac{m}{\alpha}}
		\right\}.
	\end{equation*}
	Further, there exists $P\in \cuF_{m,\alpha}$ such that 
	$W_2( P, \sN (\bzero, \id_m))=\sqrt{m/\alpha}$.
\end{thm}

\begin{rem}
	This theorem gives an upper bound on the $W_2$ radius of $\cuF_{m,\alpha}$ for all $m$. A comparison with Theorem \ref{thm:outer_bound_1_dim} suggests that this upper bound is tight up to a constant factor.
	On the other hand, the lower bound $\sqrt{m / \alpha}$ shows that the factor $\sqrt{m}$ is necessary.
\end{rem}

Let us emphasize that even in the one-dimensional case,
Theorem~\ref{thm:div_outer_bd} is not a consequence of Theorem~\ref{thm:outer_bound_1_dim}.
There exist infeasible distributions $P$ that are excluded by Theorem~\ref{thm:div_outer_bd}
and satisfy $W_2 (P, \sN(0, 1)) \le 1 / \sqrt{\alpha}$. 
An interesting case is the one in which $P$ is supported on a set of lower dimension.
In particular for certain values of $\alpha$, probability measures supported on low-dimensional 
manifolds in $\R^m$ are not feasible, no matter how close they are to the standard normal distribution 
in $W_2$ distance. In the next section, we present such an example.

\subsection{Information dimension bound}
We first introduce the notion of \emph{information dimension}, a measurement of the fractal dimension of any probability distribution.
\begin{defn}[Information dimension \cite{renyi1959dimension}]\label{def:info_dim}
	Let $X = (X_1, \cdots, X_m)$ be an arbitrary random variable in $\R^m$, and denote for $\veps > 0$ the following discretization of $X$:
	\begin{equation*}
		\langle X \rangle_{\veps} = \veps \left\lfloor \frac{X}{\veps} \right\rfloor, \ \text{where} \ \left\lfloor \frac{X}{\veps} \right\rfloor = \left( \left\lfloor \frac{X_1}{\veps} \right\rfloor, \cdots, \left\lfloor \frac{X_m}{\veps} \right\rfloor \right).
	\end{equation*}
	Let $H(Z)$ denote the Shannon entropy of a discrete random variable $Z$, i.e., $H(Z) = - \sum_{z} \P(Z = z) \log \P(Z = z)$, and define
	\begin{equation*}
		\underline{d} (X) = \liminf_{\veps \to 0} \frac{H (\langle X \rangle_{\veps})}{\log (1 / \veps)}, \;\;\;\text{and} \;\;\; \overline{d} (X) = \limsup_{\veps \to 0} \frac{H (\langle X \rangle_{\veps})}{\log (1 / \veps)},
	\end{equation*}
	where $\underline{d} (X)$ and $\overline{d} (X)$ are called lower and upper information dimensions of $X$, respectively.
	With an abuse of notation, we will write $\underline{d} (P_X):=\underline{d} (X)$,
	$\overline{d} (P_X):=\overline{d} (X)$ when $X\sim P_X$. 
\end{defn}

\begin{thm}\label{thm:infeasible_info_dim}
	Assume $\alpha > 1$, and $P \in \cuP (\R^m)$ satisfies $\underline{d} (P) < m (1 - 1 / \alpha)$. Then, $P$ is not $(\alpha, m)$-feasible. As a consequence, if $P$ is supported on an $s$-dimensional smooth manifold in $\R^m$ (where $s < m$) such that $\alpha > m / (m - s)$, then $P$ is not $(\alpha, m)$-feasible.
\end{thm}

\begin{rem}\label{rem:feasible_set}
	Since any discrete distribution $P \in \cuP(\R^m)$ with finite entropy has 
	information dimension equal to $0$, we know that $P$ is not $(\alpha, m)$-feasible provided 
	$\alpha > 1$.
	As a consequence, for any $\alpha>1$, $\veps>0$, we can construct a distribution $P$ such 
	that $W_2(P,\sN(\bzero,\id_m))\le \veps$, and yet $P$ is infeasible. This is achieved by 
	discretizing $\sN(\bzero,\id_m)$ on a scale $\veps$, obtaining $P$ that is a countable combination
	of point masses at points in $\veps\cdot \Z^m$.
\end{rem}

A cartoon of the $W_2$ geometry of $\cuF_{1, \alpha}$  is given in Figure \ref{fig:CartoonW2}. Note that, since $\cuF_{m,\alpha}$ is closed under weak convergence (Lemma~\ref{lem:closure_prop}), the
last infeasibility result applies to distributions $P$ that have a density in $\R^m$ but
are sufficiently close ---say--- to a low-dimensional manifold.

\begin{figure}[t!]
	\centering
	\includegraphics[width=25em]{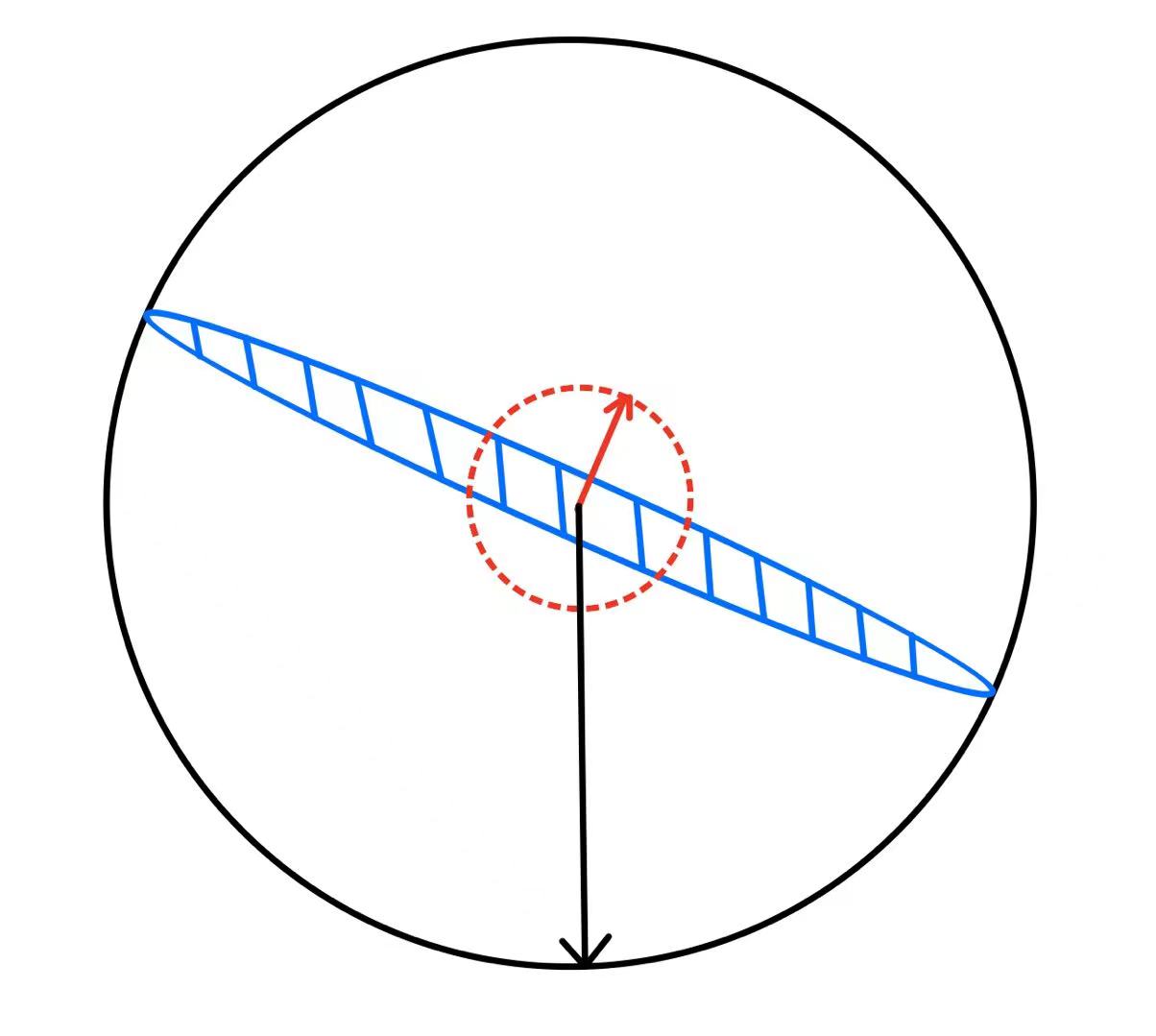}
	\put(-157.5, 100){$\sN(0, 1)$}
	\put(-120, 60){$1 / \sqrt{\alpha}$}
	\put(-117.5, 142.5){ $\veps$}
	\put(-225, 135){ $\cuF_{1, \alpha}$}
	\put(-60, 200){$W_2$ ball}
	\caption{A cartoon of the $W_2$ geometry of the feasibility set $\cuF_{1, \alpha}$ (blue shaded area).
	 The outer $W_2$ radius of $\cuF_{1, \alpha}$ 
	 (with respect to center $\sN (0, 1)$) is equal to $1 / \sqrt{\alpha}$,
	 but the  inner radius is zero. Namely, for any $\veps$, the  $W_2$ ball centered at $\sN (0, 1)$ 
	 with radius 
	 $\veps$ is not contained in $\cuF_{1, \alpha}$ for any $\alpha > 1$.}\label{fig:CartoonW2}
\end{figure}
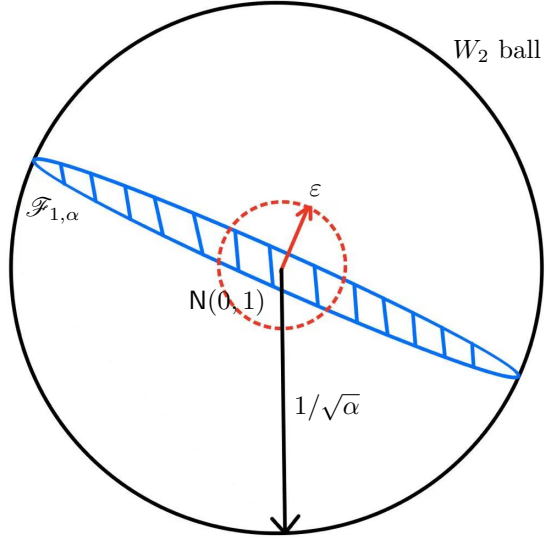

\subsection{Application to the negative spherical perceptron}
In this section we apply our KL-Wasserstein outer bound to derive an upper bound on the satisfiability threshold of the negative spherical perceptron, which is a continuous random constraint satisfaction problem (CSP) studied in \cite{stojnic2013negative,montanari2021tractability}. Our bound matches the upper bound in \cite{montanari2021tractability} that is tight up to an $1 + o_{\kappa} (1)$ factor as the margin $\kappa \to - \infty$. In particular, we show that
\begin{thm}\label{thm:negperc_upper_bd}
    For any $\eta > 0$, there exists $\underline{\kappa} (\eta) < 0$, such that as long as $\kappa < \underline{\kappa} (\eta)$ and
    \begin{equation}
    	\alpha > \frac{(1 + \eta) \log \vert \kappa \vert}{\Phi (\kappa)},
    \end{equation}
    we have
    \begin{equation}\label{eq:negperc_upper_bd}
    	\lim_{n \to \infty} \P \left( \exists \btheta \in \S^{d-1} \ \mbox{s.t.} \ \langle \btheta, \xx_i \rangle \ge \kappa \ \forall i \in [n] \right) = 0.
    \end{equation}
\end{thm}

\begin{proof}
	Note that, if $\btheta$ is a $\kappa$-margin solution, then $\hat{P}_{n, \btheta}$ is a probability distribution supported on $[\kappa, +\infty)$. Following the same argument as that in the proof of Theorem~\ref{thm:div_outer_bd}, we know that there exists two probability measures $P, Q \in \cuP(\R)$ such that $\supp(P) \subset [\kappa, +\infty)$, and
	\begin{equation}\label{eq:neg_perc_cond}
		W_2^2 \left( P, Q \right) \le \frac{\veps^2}{\alpha} \log \left( \frac{C}{\veps} \right), \quad D_{\rm KL} \left( Q \Vert \normal(0, 1) \right) \le \frac{1}{\alpha} \log \left( \frac{C}{\veps} \right). 
	\end{equation}
	Since $D_{\rm KL} ( Q \Vert \normal(0, 1) ) < \infty$, $Q$ must have density $q$. It is easy to see that
	\begin{equation*}
		\inf_{\supp(P) \subset [\kappa, +\infty)} W_2^2 \left( P, Q \right) = \int_{-\infty}^{\kappa} (\kappa - x)^2 q(x) \d x.
	\end{equation*}
	Therefore, the conditions Eq.~\eqref{eq:neg_perc_cond} reduce to
	\begin{equation}
		\int_{-\infty}^{\kappa} (\kappa - x)^2 q(x) \d x = \int_{\R} (\kappa - x)_+^2 q(x) \d x \le \Delta_1, \quad \int_{-\infty}^{\infty} q(x) \log \frac{q(x)}{\phi(x)} \d x \le \Delta_2,
	\end{equation}
	where
	\begin{equation*}
		\Delta_1 = \Delta_1 (\alpha, \veps) = \frac{\veps^2}{\alpha} \log \left( \frac{C}{\veps} \right), \ \Delta_2 = \Delta_2 (\alpha, \veps) = \frac{1}{\alpha} \log \left( \frac{C}{\veps} \right).
	\end{equation*}
	According to Lemma~\ref{lem:min_entropy}, given that $q$ is a density and $\int_{\R} (\kappa - x)_+^2 q(x) \d x \le \Delta_1$, we have
	\begin{equation*}
		\int_{-\infty}^{\infty} q(x) \log \frac{q(x)}{\phi(x)} \ge - \min_{\mu \ge 0} \left\{ \mu \Delta_1 + \log \E \left[ \exp \left( -\mu (\kappa - G)_+^2 \right) \right] \right\}, \ G \sim \normal (0, 1),
	\end{equation*}
	thus leading to
	\begin{align*}
		& \Delta_2 \ge - \min_{\mu \ge 0} \left\{ \mu \Delta_1 + \log \E \left[ \exp \left( -\mu (\kappa - G)_+^2 \right) \right] \right\} \\
		\implies\, & \min_{\mu \ge 0} \left\{ \Delta_2 +  \mu \Delta_1 + \log \E \left[ \exp \left( -\mu (\kappa - G)_+^2 \right) \right] \right\} \ge 0.
	\end{align*}
	The above inequality must hold for any $\veps \in (0, \veps_0)$. Hence, we obtain that
	\begin{equation}\label{eq:cond_negperc}
	    \forall \veps \in (0, \veps_0) \ \mbox{and} \ \mu \ge 0, \ \frac{1 + \veps^2 \mu}{\alpha} \log \left( \frac{C}{\veps} \right) \ge - \log \E \left[ \exp \left( -\mu (\kappa - G)_+^2 \right) \right].
	\end{equation}
    Based on the proof of Theorem 3.2 in \cite{montanari2021tractability}, we know that taking $\mu = \vert \kappa \vert^{2(1 + \eta / 4)}$ yields the estimate:
    \begin{equation*}
    	- \log \E \left[ \exp \left( -\mu (\kappa - G)_+^2 \right) \right] = (1 + o_{\kappa} (1)) \left( 1 - \E \left[ \exp \left( -\mu (\kappa - G)_+^2 \right) \right] \right) = (1 + o_{\kappa} (1)) \Phi (\kappa).
    \end{equation*}
    We now choose $\veps = \vert \kappa \vert ^{-1 - \eta/2}$, then Eq.~\eqref{eq:cond_negperc} implies that
    \begin{equation*}
    	\alpha \le \frac{(1 + \veps^2 \mu) \log (C/\veps)}{- \log \E \left[ \exp \left( -\mu (\kappa - G)_+^2 \right) \right]} \le (1 + o_{\kappa} (1)) \frac{(1 + \vert \kappa \vert^{- \eta/2}) \log \left( C \vert \kappa \vert^{1 + \eta/2} \right)}{\Phi(\kappa)} = \left(1 + \frac{\eta}{2} + o_{\kappa} (1) \right) \frac{\log \vert \kappa \vert}{\Phi(\kappa)},
    \end{equation*}
    which contradicts the assumption $\alpha > (1 + \eta) \log \vert \kappa \vert / \Phi(\kappa)$ for sufficiently negative $\kappa$. This proves Eq.~\eqref{eq:negperc_upper_bd}.
\end{proof}

\begin{rem}
	The above theorem shows that our KL-Wasserstein outer bound is tight up to a constant factor. Indeed, if we take $\alpha = (1 - \eta) \log \vert \kappa \vert / \Phi (\kappa)$, then \cite[Theorem~3.1]{montanari2021tractability} implies that there exists an $(\alpha, 1)$-feasible distribution $P$ supported on $[\kappa, +\infty)$. Taking $\veps = \vert \kappa \vert^{-1-\eta}$ and setting
	\begin{equation*}
		\Delta_1 = \Delta_1 (\alpha, \veps) = \frac{\veps^2}{\alpha} \log \left( \frac{C}{\veps} \right) = (1 + o_{\kappa} (1)) \Phi(\kappa) \cdot \frac{1+\eta}{1 - \eta} \vert \kappa \vert^{- 2 - 2 \eta} ,
	\end{equation*}
	then from the proof of Theorem~\ref{thm:negperc_upper_bd}, we know that for any $Q \in \cuP (\R)$ satisfying $W_2^2 (P, Q) \le \Delta_1$, one has
	\begin{align*}
		\Delta_2 =\, & \Delta_2 (\alpha, \veps) = D_{\sKL} (P \Vert Q) \ge - \min_{\mu \ge 0} \left\{ \mu \Delta_1 + \log \E \left[ \exp \left( -\mu (\kappa - G)_+^2 \right) \right] \right\} \\
		\stackrel{(i)}{\ge}\, & (1 + o_{\kappa} (1)) \Phi(\kappa) - o_{\kappa} (1) \Phi(\kappa) = (1 + o_{\kappa}(1)) \Phi(\kappa) \\
		\ge\, & (1+o_{\kappa} (1) ) \frac{1 - \eta}{1+\eta} \cdot \frac{1}{\alpha} \log \left( \frac{C}{\veps} \right),
	\end{align*}
	where $(i)$ follows from the choice of $\Delta_1$ and $\mu = \vert \kappa \vert^{2+\eta}$. Note that the last line matches our KL-Wasserstein outer bound up to a constant factor for small $\veps > 0$.
\end{rem}

\begin{lem}\label{lem:min_entropy}
	Given $w(x) \ge 0$ and $p(x)$ a density function on $\R$, the minimum of $D_{\sKL} (q \Vert p)$ subject to $\int_{\R} w(x) q(x) \d x \le \Delta$ for a density $q$ is given as
	\begin{equation*}
		- \min_{\mu \ge 0} \left\{ \mu \Delta + \log \int_{\R} p(x) \exp \left( - \mu w(x) \right) \d x \right\} = - \min_{\mu \ge 0} \left\{ \mu \Delta + \log \E_{X \sim p} \left[ \exp \left( - \mu w(X) \right) \right] \right\}.
	\end{equation*}
\end{lem}

\begin{proof}
	We study the following optimization problem:
	\begin{equation}\label{eq:min_entropy}
		\mbox{minimize} \ \int_{\R} q(x) \log \frac{q(x)}{p(x)} \d x, \quad \mbox{subject to} \ \int_{\R} q(x) \d x = 1, \ \int_{\R} w(x) q(x) \d x \le \Delta.
	\end{equation}
	According to Lagrange duality, the optimum of Eq.~\eqref{eq:min_entropy} can be expressed as
	\begin{align*}
		& \min_{q(x) \ge 0} \max_{\lambda \in \R, \mu \ge 0} \left\{ \int_{\R} q(x) \log \frac{q(x)}{p(x)} \d x + \lambda \left( \int_{\R} q(x) \d x - 1 \right) + \mu \left( \int_{\R} w(x) q(x) \d x - \Delta \right) \right\} \\
		=\, & \max_{\lambda \in \R, \mu \ge 0} \left\{ - \lambda - \mu \Delta + \min_{q(x) \ge 0} \int_{\R} q(x) \left( \log \frac{q(x)}{p(x)} + \lambda + \mu w(x) \right) \d x \right\} \\
		=\, & \max_{\lambda \in \R, \mu \ge 0} \left\{ - \lambda - \mu \Delta + \int_{\R} \min_{q \ge 0} \left\{ q \left( \log \frac{q}{p(x)} + \lambda + \mu w(x) \right) \right\} \d x \right\} \\
		\stackrel{(i)}{=}\, & \max_{\lambda \in \R, \mu \ge 0} \left\{ - \lambda - \mu \Delta -e^{-\lambda - 1} \int_{\R} p(x) \exp(- \mu w(x)) \d x \right\} \\
		\stackrel{(ii)}{=}\, & \max_{\mu \ge 0} \left\{ - \mu \Delta - \log \int_{\R} p(x) \exp(- \mu w(x)) \d x \right\} \\
		=\, & - \min_{\mu \ge 0} \left\{ \mu \Delta + \log \int_{\R} p(x) \exp(- \mu w(x)) \d x \right\},
	\end{align*}
	where $(i)$ and $(ii)$ both follow from direct calculations. This completes the proof.
\end{proof}


\section{Inner bounds: Unsupervised learning}\label{sec:2nd_moment}
In this section, we present our main results on inner bounds for the feasibility set 
$\cuF_{m, \alpha}$. Given a target distribution $P$, we will show that it is feasible 
(below a critical value of $\alpha$) in two steps: 
First, we consider a discretization of $P$ supported on a finite set of points
$A\subseteq\R^m$. We will prove that the corresponding discretization of the empirical
distribution $\hat{P}_{n,\WW}$ of the projected points converges to the discretization of $P$.
We then establish feasibility of $P$ by taking increasingly denser meshes $A$.

For $P \in \cuP (\R^m)$ and a finite set $A = \{ \aa_1, \dots, \aa_M \} \subset \R^m$, define the 
$A$-discretization of $P$ as follows:
\begin{defn}[$A$-discretization]
	For $\xx \in \R^m$, denote the projection of $\xx$ onto $A$ as $\langle \xx \rangle_A$, namely
	\begin{equation*}
		\langle \xx \rangle_A = \argmin_{\aa_i \in A} \norm{\xx - \aa_i}_2.
	\end{equation*}
    The $A$-discretization of $P$ is then defined as $\langle P \rangle_A = 
    \Law (\langle X \rangle_A)$ where $X \sim P$.
\end{defn}
 
 Let $P_A$ be a probability distribution with support on $A$ such that $P_{A}(\aa_i)>0$
 for every  $\aa_i\in A$.
  We next define the feasibility lower bound $\oalpha_{\rm lb} (P_A)$ for $P_A$. As stated formally below, for 
  $\alpha<\oalpha_{\rm lb} (P_A)$ we can find projections $\WW=\WW_n(\bX)$ such that
  the $A$-discretization of the empirical distribution of projected points is close to
  $P_{A}$ with probability bounded away from zero.
\begin{defn}[Feasibility threshold of $P_{A}$]\label{def:2nd_thres}
	For $\QQ \in \R^{m \times m}$ satisfying $\QQ^\top \QQ \preceq \id_m$,
	let $\Phi^{(2)}_{A,\QQ}$ be the following probability distribution supported on $A\times A$:
	\begin{equation*}
		\Phi^{(2)}_{A,\QQ} := \Law\left( \langle \GG_1 \rangle_A , \ \langle \GG_2 \rangle_A \right),
		 \;\;\; (\GG_1, \GG_2) \sim \sN \left( \bzero, \ \begin{bmatrix}
    	\id_m & \QQ \\
    	\QQ^\top & \id_m
    \end{bmatrix} \right)\, .
	\end{equation*}
	 We define the probability distribution $R^{(2)}_{A,\QQ}$ supported on $A\times A$
	 via:
    \begin{equation}\label{eq:def_info_proj_Q}
    	R^{(2)}_{A,\QQ} := 
    	\argmin_{R \in \cuP(A \times A)} \Big\{ D_{\sKL}(R \Vert 	\Phi^{(2)}_{A,\QQ}):\;\;
        \mbox{\rm subject to} \  \sum_{j=1}^{M} R(\aa_i,\aa_j)= \sum_{j=1}^{M} R(\aa_j,\aa_i)=
        P_A(\aa_i), \ \forall i \in [M]\Big\}\, ,
    \end{equation}
   where, specializing Eq.~\eqref{eq:KL_def}, we have  
   $D_{\sKL}(R \Vert 	\Phi^{(2)}) = \sum_{i, j = 1}^{M} R(\aa_i,\aa_j) \log (R(\aa_i,\aa_j)/ 
   \Phi^{(2)}(\aa_i,\aa_j))$. In words,
     $R^{(2)}_{A,\QQ}$ is the information projection of the distribution $\Phi^{(2)}_{A,\QQ}$ 
     onto the set of distributions on $A \times A$ whose both margins are 
     $P_{A}$. 
  
  Next define $\Psi(\cdot\,;P_{A}):\R^{m\times m}\to \R$ via 
  \begin{align}
  \Psi(\QQ;P_{A}) := D_{\sKL} \big( R^{(2)}_{A,\QQ} \Vert \Phi^{(2)}_{A,\QQ}\big) + \frac{1}{\alpha}I(\QQ) \, ,
 \;\;\;\; I(\QQ) := - \frac{1}{2}\log \det (\id_m - \QQ^\top \QQ) \, .
  \end{align}   
  (We set  $\Psi(\QQ;P_{A})=+\infty$ if $\|\QQ\|_{\op}>1$, which is equivalent to defining $\Psi(\QQ;P_{A})$ only when $\QQ^\top \QQ \preceq \id_m$.)
  
Finally, we define
\begin{align}
\oalpha_{\rm lb} (P_A) := \sup\Big\{\alpha>0\, :\;\; \Psi(\QQ;P_{A})>\Psi(\bzero;P_{A}), \,
\;\;\;\forall \QQ\neq \bzero\Big\}\, .
\end{align}
\end{defn}

We next establish that $\oalpha_{\rm lb} (P_A)$ is indeed a lower bound on the 
feasibility threshold for $P_A$. (Here, feasibility is understood for the discretized
projections.)
\begin{thm}\label{thm:discrete_lower_bd}
Let $P_A$ be a probability distribution on the finite set $A\subseteq\R^m$ giving 
strictly positive mass to every $\aa_i\in A$.
	If $\alpha < \oalpha_{\rm lb} (P_A)$, then there exists a constant $C$ and
a sequence of a random orthogonal matrices $\WW = \WW_n (\XX)$ such that 
	\begin{equation*}
		\liminf_{n \to \infty} \P \left( d_{\sTV} \left( \langle \hat{P}_{n, \WW} \rangle_A, P_{A} \right) \le \frac{C}{n} \right) > 0,
	\end{equation*}
	where $\hat{P}_{n, \WW} = (1/n) \sum_{i=1}^{n} \delta_{\WW^\top \xx_i}$ is the empirical
	 distribution of the projected data points, and $d_{\sTV}$ is 
	 the total variation distance.  
\end{thm}

In the next two subsections, we will apply this result
to prove lower bounds on the supremum of $\alpha$ such that $P\in \cuF_{m,\alpha}$
for $P$ a probability distribution with a density on $\R^m$. We treat the cases $m=1$
and $m>1$ separately and carry out a more accurate analysis in the first case.
As anticipated above, our proof is based on approximating $P$ by $P_A=\<P\>_A$ for a sequence of
 finite sets $A$ of increasing cardinality.
It will be crucial to control the lower bounds $\oalpha_{\rm lb} (P_A)$ along such a sequence.
The following variational representation of 
 the KL divergence turns out to be particularly useful.
\begin{lem}[Donsker-Varadhan representation of KL divergence]\label{lem:var_rep_KL}
	Let $Q$ and $P$ be two probability measures on the same space $\cX$, then we have
	\begin{equation*}
		D_{\sKL} \left( Q \Vert P \right) = \sup_{g: \ \cX \to \R} \left\{ \E_{Q} \left[ g(X) \right] - \log \E_{P} \left[ \exp(g(X)) \right] \right\}.
	\end{equation*}
    Moreover, recalling $R^{(2)}_{A,\QQ}$ and $\Phi^{(2)}_{A,\QQ}$ from 
    Definition~\ref{def:2nd_thres}, and denoting
    \begin{equation*}
    	\Phi_A := \Law\big(\langle \GG \rangle_A \big) \;\; \text{for} \;\; \GG \sim \sN (\bzero, \id_m),
    \end{equation*}
    we get that
    \begin{align*}
    	D_{\sKL} (R^{(2)}_{A,\QQ}\Vert \Phi^{(2)}_{A,\QQ}) - & D_{\sKL} \left( 
    	R^{(2)}_{A,\bzero}\Vert \Phi^{(2)}_{A,\bzero} \right) \\
    	&\ge \sup_{(\lambda_i, \mu_i)_{i \in [M]}} \left\{ \sum_{i = 1}^{M} \left( \lambda_i + \mu_i \right) P_A(\aa_i) - \log \left( \sum_{i, j=1}^{M} e^{\lambda_i + \mu_j} \frac{P_A(\aa_i)P_A(\aa_j)}{\Phi_A(\aa_i)\Phi_A(\aa_j)} \Phi^{(2)}_{A,\QQ}(\aa_i,\aa_j)\right) \right\}.
    \end{align*}
\end{lem}

\subsection{Inner bound for $m = 1$}\label{sec:1D_inner_bd}
Here we specialize our results to the case $m = 1$, and consequently obtain an inner bound for
 $\cuF_{1, \alpha}$. To simplify our treatment,  
  we will require that our target distribution $P$ has a density $p(x)$ on $\R$, 
  which satisfies the following assumption:
\begin{ass}\label{ass:1D_inner_bd}
	Denote by $\phi (x)$ the density of the standard normal distribution. Assume that
	$P(\d x) = p(x)\, \d x$,
	 $\E_P [X] = \int_{\R} x p (x) \, \de x= 0$, and that the chi-square distance between
	  $P$ and $\sN (0, 1)$ is finite, namely
	\begin{equation*}
		\chi^2 \left( P, \sN (0, 1) \right) = \int_{\R} \frac{(p(x) - \phi(x))^2}{\phi(x)}\, \d x < \infty.
	\end{equation*}
    Since the KL divergence is always dominated by the chi-square distance, we know that $D_{\sKL} (P \Vert \sN(0, 1)) < \infty$.
\end{ass}

\begin{thm}[Lower bound on feasibility threshold for $m=1$]\label{thm:inner_bd_1D}
	Let $P \in \cuP(\R)$ have density function $p(x)$ satisfying Assumption~\ref{ass:1D_inner_bd},
	 and denote
	\begin{equation*}
		c_2 = \frac{1}{\sqrt{2}} \E_P (X^2 - 1) = \frac{1}{\sqrt{2}} \int_{\R} (x^2 - 1) (p(x) - \phi(x)) \, \d x.
	\end{equation*}
	(Note that $c_2^2 \le \chi^2 (P, \sN(0, 1))$ by Cauchy-Schwarz inequality.) 
	
 Define the following lower bound on the feasibility threshold for $P$:
	\begin{equation}\label{eq:def_feas_thres_cont}
		\alpha_{\rm lb} (P) = \max_{q \in [0, 1]} \min \left\{ \frac{I(q)}{D_{\sKL} (P \Vert \sN(0, 1))}, \ \frac{1}{2 \left( c_2^2 + q \left( \chi^2 (P, \sN(0, 1)) - c_2^2 \right) \right)} \right\},
	\end{equation} 
	where $I(q) = - \log (1 - q^2) / 2$. Then, as long as $\alpha < \alpha_{\rm lb} (P)$, 
	$P$ is $(\alpha, 1)$-feasible.
\end{thm}

\begin{prop}[Characterization of $\alpha_{\rm lb} (P)$]\label{prop:char_lower_bd}
	Let $\alpha_{\rm lb} (P)$ be as defined in Eq.~\eqref{eq:def_feas_thres_cont}. Then, we have
	\begin{equation*}
		\alpha_{\rm lb} (P) \ge \frac{1}{4} \cdot \min \left\{ \frac{1}{c_2^2}, \ \frac{1}{D_{\sKL} (P \Vert \sN(0, 1))^{1/3} \left( \chi^2 (P, \sN(0, 1)) - c_2^2 \right)^{2/3} } \right\}\, .
	\end{equation*}
\end{prop}

\begin{rem}
	Recall that $D_{\sKL} (P \Vert \sN(0, 1))\le  \chi^2 (P, \sN(0, 1))$ always holds.
	If, in addition, $ \chi^2 (P, \sN(0, 1)) \le C_0 \cdot D_{\sKL} (P \Vert \sN(0, 1))$, 
	for a constant $C_0$, then there exists another constant $C=C(C_0)$ depending on $C_0$ such that the following lower bound on the feasibility threshold holds:
	\begin{equation*}
		\alpha_{\rm lb} (P) \ge \frac{C}{D_{\sKL} (P \Vert \sN(0, 1))}\, .
	\end{equation*}
	This matches the outer bound in Theorem~\ref{thm:div_outer_bd} up to a logarithmic factor.
\end{rem}

\subsection{Inner bound for general $m>1$}\label{sec:mD_inner_bd}

We now generalize the results in Section~\ref{sec:1D_inner_bd} to dimensions $m > 1$.
\begin{ass}\label{ass:mD_inner_bd}
	We assume that $P \in \cuP (\R^m)$ has zero mean, and 
	$\chi^2 (P, \sN (\bzero, \id_m)) < \infty$.
	In particular, $P$ has a density which we denote by $p(\xx)$.
	Let $\phi(\xx)$ be the density of $\sN (\bzero, \id_m)$, 
	and further denote $h(\xx) = (p(\xx) - \phi(\xx)) / \phi(\xx)$. For any 
	$\UU \in O(m, m)$, the function $h(\UU \xx)$ has the following multivariate Hermite expansion
	(see the proof of Theorem~\ref{thm:inner_bd_mD}):
	\begin{equation*}
    	h \left( \UU \xx \right) = \sum_{\vert \nn \vert \ge 2} c_{\nn} (\UU) {\rm He}_{n_1} (x_1) \cdots {\rm He}_{n_m} (x_m),
	\end{equation*}
	where $\xx = (x_1, \cdots, x_m)$, $\nn = (n_1, \cdots, n_m)$, $c_{\nn} (\UU) = c_{n_1, \cdots, n_m} (\UU)$, and 
	$ \vert \nn \vert = n_1 + \cdots + n_m$. Moreover, we know that for $\GG \sim \sN (\bzero, \id_m)$,
	\begin{equation*}
		\chi^2 (P, \sN(\bzero, \id_m)) = \E [h(\GG)^2] = \E [h(\UU \GG)^2] =
		 \sum_{\vert \nn \vert \ge 2} c_{\nn} (\UU)^2,
	\end{equation*}
	which is independent of $\UU$ due to rotational invariance of $\sN (\bzero, \id_m)$.
\end{ass}
\begin{thm}[Lower bound on feasibility threshold for $m\ge 2$]\label{thm:inner_bd_mD}
	Let $P \in \cuP (\R^m)$ satisfy Assumption~\ref{ass:mD_inner_bd}, and define 
	$\alpha_{\rm lb} (P)$ to be the supremum of all $\alpha > 0$ such that the following happens:
	 There exists a neighborhood $Q_0$ of $\bzero$, such that for any $\QQ \in Q_0$ 
	 having singular values $\{ q_1, \cdots, q_m \}$,
	\begin{equation*}
        \sup_{\UU \in O(m, m)} \left\{ \sum_{\vert \nn \vert \ge 2} c_{\nn} (\UU)^2 \prod_{i=1}^{m} q_i^{n_i} \right\} <
         \frac{1}{2 \alpha} \sum_{i=1}^{m} q_i^2 = \frac{1}{2 \alpha} \norm{\QQ}_{\rm F}^2,
    \end{equation*}
	and that
	\begin{equation*}
		\alpha < \frac{\inf_{\QQ \notin Q_0} I(\QQ)}{D_{\sKL} (P \Vert \sN (\bzero, \id_m))}.
	\end{equation*}
	Then, as long as $\alpha < \alpha_{\rm lb} (P)$, $P$ is $(\alpha, m)$-feasible.
\end{thm}

\section{Main results: Supervised learning}\label{sec:MainSupervised}

In order to motivate our generalization of previous results to supervised learning, 
consider the following empirical risk minimization (ERM) problem:
\begin{align}\label{eq:EmpiricalRisk}
	\hR_n^{\star}(\XX,\yy):= \inf_{h\in \cH_m} \inf_{\WW\in O(d,m)}\frac{1}{n}\sum_{i=1}^n
	L(y_i,h(\WW^\top\xx_i))\, .
\end{align}
Here the minimization is over $\WW$ in $O(d, m)$, the set of $d\times m$ orthogonal matrices,
and $h$ in a set $\cH_m$ of functions $h:\R^m\to\R$. For instance, we could consider
$\cH_m:=\{h:\R^m\to\R:\|h\|_{\Lip}\le C\}$ (the set of functions with Lipschitz modulus at most 
$C$), or $\cH_m:=\{h(\uu) = \sum_{i=1}^ma_i\sigma(u_i):\;\|\aa\|_1\le B\}$
so that $\xx\mapsto h(\WW^\top\xx)$ is a two-layer neural network with total second-layer
weights bounded by $B$ and orthonormal first-layer weights.
Finally  $\cH_m:=\{h(\uu) = \sum_{i=1}^ma_i\sigma(\<\bb_i,\uu\>):\;\|\aa\|_1\le B\,, \
\|\bb_i\|_2 \le B, \ \forall i \le m \}$ allows to treat all two-layer networks with $m$ neurons and bounded first and second-layer weights.

The next proposition establishes that the minimum of the empirical risk is
given asymptotically by a variational problem over the feasibility set $\cuF_{m,\alpha}^{\varphi}$.
\begin{prop}\label{prop:ERM_asymptotics}
	Assume $n,d\to\infty$ with $n/d\to\alpha\in (0,\infty)$. 
	Further assume $L:\{+1,-1\}\times\R\to\R$ to be bounded continuous, and 
	$\cH_m\subseteq\{h:\R^m\to\R:\|h\|_{\Lip}\le C\}$, for some constant $C$. Then, we have
	\begin{equation}
		\plimsup_{n,d\to\infty}\hR_n^{\star}(\XX,\yy) = \inf_{h\in \cH_m} 
		\inf_{P\in\cuF^{\varphi}_{m,\alpha}}
		\int_{\{\pm 1\} \times \R^m} L(y,h(\zz))\, P(\d y,\d \zz)\, .\label{eq:ERM-asymp}
	\end{equation}
    (For a sequence of random variables $\{Z_n\}_{n\ge 1}$ and a constant $c$,
    $\plimsup_{n \to\infty}Z_n=c$ if $c$ is the infimum of all $c'$ such that 
    $\lim_{n \to\infty}\P(Z_n\ge c')=0$.)
\end{prop}

\subsection{Wasserstein outer bound for $m = 1$}

Throughout this section, we assume $m=1$, so that $\WW=\ww\in\S^{d-1}$.
Recall the data model defined in Section \ref{sec:NullModel},
whereby $\xx_i\sim\normal(\bzero,\id_d)$ and $y_i\in\{+1,-1\}$ are such that
\begin{equation*}
	\P(y_i=+1|\xx_i)= \varphi(\VV^\top\xx_i) = 1 - \P(y_i=-1|\xx_i),
\end{equation*}
with $\VV\in O(d,k)$. Notice that we allow for $k>1$.

We are interested in the empirical distribution of $\{(y_i,\<\ww,\xx_i\>)\}_{i\le n}$, denoted by
\begin{equation*}
	\hat{P}_{n, \ww} := \frac{1}{n} \sum_{i = 1}^{n} \delta_{\left( y_i, \langle \xx_i, \ww \rangle \right)}.
\end{equation*}
We begin with an elementary calculation that characterizes the asymptotics of this 
distribution for a vector $\ww$ that does not depend on the data $\XX,\yy$.
\begin{lem}\label{lem:key_triple}
	Let the random variables $(Y, G, Z)$ be such that
	\begin{equation}\label{eq:key_triple}
		(Y, G) \perp Z, \ G \sim \sN \left( \bzero, \id_k \right), \ Z \sim \sN (0, 1), \ \text{and} \ \P (Y = +1 \vert G) = \varphi(G) = 1 - \P (Y = -1 \vert G)\, .
	\end{equation}
	For $\ww \in \S^{d - 1}$ independent of $\XX,\yy$,  define
	\begin{equation}\label{eq:p_w_dist}
		P_{\ww} = \Law \left( Y, \ww^\top \VV G + \sqrt{1 - \norm{\VV^\top \ww}_2^2} \cdot Z \right)\,.
	\end{equation}
	Then, in the limit $n,d\to\infty$ (not necessarily proportionally), we have, in probability,
	\begin{align}
		\lim_{n,d\to\infty}d_{\sKS}(\hat{P}_{n, \ww},P_{\ww}) = 0\,.
	\end{align}
\end{lem}
In other words, the empirical distribution we are interested in, $\hat{P}_{n, \ww}$,
is close to  $P_{\ww}$ for `most' directions $\ww$. This is analogous to the result of
 \cite{diaconis1984asymptotics} establishing, in the unsupervised setting, 
 that one-dimensional projections are Gaussian along most directions. (Note
 however that, unlike \cite{diaconis1984asymptotics}, we assume here a specific data distribution.)

We next quantify the $W_2$ deviation of $\hat{P}_{n, \ww}$ from $P_\ww$ along
`atypical' directions. It is useful to define a modified $W_2$ distance for future applications.
\begin{defn}[$\eta$-constrained $W_2$ distance]\label{def:cons_wp_metric}
	Let $P$ and $Q$ be two probability measures on $\R^d$. 
	For any $\eta \ge 0$, the $\eta$-constrained $W_2$ distance between $P$ and $Q$ 
	is defined by
	\begin{equation}\label{eq:cons_wp_metric}
		W_{2}^{(\eta)} ( P, Q ) = \left( 
		\inf_{\gamma \in \Gamma^{(\eta)}(P, Q)} \int_{\R^d \times \R^d} 
		\norm{\xx - \yy}_2^2 \gamma (\d \xx \times \d \yy) \right)^{1/2},
	\end{equation}
	where $\Gamma^{(\eta)}(P, Q)$ denotes the set of all couplings $\gamma$ of 
	$P$ and $Q$ which satisfy
	\begin{equation*}
		\int_{\R^d \times \R^d} \<\ee_1, \xx - \yy \>^2 \, \gamma (\d \xx \times \d \yy)  \le \eta^2\,,
	\end{equation*}
	where $\ee_1 = (1, 0, \cdots, 0)^\top$. By convention 
	$W_{2}^{(\eta)} ( P, Q ) = \infty$ whenever $\Gamma^{(\eta)}(P, Q) = \emptyset$. 
\end{defn}
We notice that this definition can be easily generalized to $W_p$ distance for every $p \ge 1$, and to projections along more than 
one direction.

\begin{thm}\label{thm:W2_outer_bd}
	Consider i.i.d. data $(\yy, \XX) = \{ (y_i, \xx_i) \}_{i \in [n]}$, 
	with $\xx_i\sim\sN(\bzero,\id_d)$ and $\P \left( y_i =+1  \vert \xx_i \right) = \varphi(\VV^\top \xx_i )$,
	cf. Section~\ref{sec:NullModel}. Assume that $n, d \to \infty$ with 
	$n / d \to \alpha \in (0, \infty)$ with $k$ fixed.
	Then for all $\eta > 0$, we have, almost surely,
	\begin{equation}\label{eq:W2_as_conv}
		\lim_{n,d \to \infty} \max_{\norm{\ww}_2 = 1} \left( W_2^{(\eta)} \left( \hat{P}_{n, \ww}, P_{\ww} \right) - \frac{1}{\sqrt{\alpha}} \sqrt{1 - \norm{\VV^\top \ww}_2^2} \right)_{+} =0\, .
	\end{equation}
\end{thm}
Informally, this theorem ensures that the $W_2$ distance between the empirical distribution
of $\{\<\ww,\xx_i\>\}_{i\le n}$ and the second marginal distribution of $P_{\ww}$ is upper bounded by
$\sqrt{(1 - \Vert \VV^\top \ww \Vert_2^2)/\alpha}$ if we match the two probability measures 
on their first marginal (by letting $\eta\to 0$).

\begin{rem}
	Since the constrained Wasserstein distance always dominates the original one, an immediate consequence of Theorem~\ref{thm:W2_outer_bd} is that, almost surely,
	\begin{equation*}
		\lim_{n,d \to \infty} \max_{\norm{\ww}_2 = 1} \left( W_2 \left( \hat{P}_{n, \ww}, P_{\ww} \right) - \frac{1}{\sqrt{\alpha}} \sqrt{1 - \norm{\VV^\top \ww}_2^2} \right)_+ =0,
	\end{equation*}
    which further implies that
    \begin{equation*}
    	\cuF^{\varphi}_{1,\alpha} \subseteq \bigcup_{\norm{\bbeta}_2 \le 1} \left\{ P \in \cuP ( \{ \pm 1 \} \times \R): W_2 \left(P, \ \Law \left( Y, \bbeta^\top G + \sqrt{1 - \norm{\bbeta}_2^2} Z \right) \right) \le \frac{\sqrt{1 - \norm{\bbeta}_2^2}}{\sqrt{\alpha}} \right\}.
    \end{equation*}
	While the above bounds are for $m=1$, any one-dimensional projection of a probability distribution
	in $\cuF_{m, \alpha}^{\varphi}$ must belong to  $\cuF_{1, \alpha}^{\varphi}$.
	We thus obtain the following outer bound on $\cuF_{m, \alpha}^{\varphi}$:
	\begin{equation*}
		\cuF^{\varphi}_{m,\alpha} \subseteq \bigcap_{\norm{\btheta}_2 = 1} \bigcup_{\norm{\bbeta}_2 \le 1} \left\{ P \in \cuP ( \{ \pm 1 \} \times \R^{m}): W_2 \left(\btheta^\top P, \ \Law \left( Y, \bbeta^\top G + \sqrt{1 - \norm{\bbeta}_2^2} Z \right) \right) \le \frac{\sqrt{1 - \norm{\bbeta}_2^2}}{\sqrt{\alpha}} \right\},
	\end{equation*}
	where we recall the triple $(Y, G, Z)$ from Lemma~\ref{lem:key_triple}, and define $\btheta^\top P = \Law (Y, \btheta^\top U)$ for $P = \Law (Y, U)$.
\end{rem}

\begin{rem}\label{rem:sub_gauss_resp}
	While for simplicity we state Theorem \ref{thm:W2_outer_bd} for the case of binary 
	responses $y_i\in \{+1,-1\}$, the proof actually applies to the case of continuous sub-Gaussian responses 
	$y_i\in \R$ as well.
\end{rem}

\subsection{KL-Wasserstein outer bound for general $m$}\label{sec:supervised_kl}
In this section we generalize our KL-Wasserstein outer bound for unsupervised projection pursuit (Theorem~\ref{thm:div_outer_bd}) to the supervised case, and consequently derive an upper bound on the Wasserstein radius of $\cuF_{m, \alpha}^{\varphi}$ for general $m > 1$. To begin with, for $\WW \in O(d, m)$ we denote by
\begin{equation}
	\hat{P}_{n, \WW} := \frac{1}{n} \sum_{i=1}^{n} \delta_{(y_i, \WW^\top \xx_i)}
\end{equation}
the empirical distribution of $\{ (y_i, \WW^\top \xx_i) \}_{i=1}^{n}$. Similarly as Eq.~\eqref{eq:key_triple}, we define the random variables $(Y, G, Z_m)$ to be such that
\begin{equation}\label{eq:key_triple_m}
	(Y, G) \perp Z_m, \ G \sim \sN \left( \bzero, \id_k \right), \ Z_m \sim \sN (\bzero, \id_m), \ \text{and} \ \P (Y = +1 \vert G) = \varphi(G) = 1 - \P (Y = -1 \vert G)\, .
\end{equation}

\begin{thm}\label{thm:KL_W2_sup}
	For $a, b > 0$, define $\cuS_m^{\varphi} (a, b)$ as the set of probability measures $P = \Law(Y, G, U) \in \cuP(\{ \pm 1 \} \times \R^{k+m})$ satisfying the following conditions:
	\begin{itemize}
		\item [(a)] The marginal distribution of $(Y, G)$ is the same as specified in Eq.~\eqref{eq:key_triple_m}, namely
		    \begin{equation*}
		    	G \sim \sN \left( \bzero, \id_k \right) \ \text{and} \ \P (Y = +1 \vert G) = \varphi(G) = 1 - \P (Y = -1 \vert G)\, .
		    \end{equation*}
		\item [(b)] $P$ lies in a neighborhood of $\Law(Y, G, Z_m)$ defined via the $W_2$ distance and KL divergence: There exists $Q = \Law(Y, G, H) \in \cuP(\{ \pm 1 \} \times \R^{k+m})$ satisfying the above Condition (a), and that
		    \begin{equation*}
		    	W_2 (P, Q) \le a, \quad D_{\sKL} (Q \Vert \Law(Y, G, Z_m)) \le b.
		    \end{equation*}	    
	\end{itemize}
	Then, there exist absolute constants $C > 0$ and $\veps_0 > 0$ such that
	\begin{align*}
		\cuF_{m,\alpha}^{\varphi} \subseteq \bigcup_{\QQ \in \R^{m \times k}: \ \QQ \QQ^\top \preceq \id_m} \Bigg\{\, & \Law \left(Y, \QQ G + (\id_m - \QQ \QQ^\top)^{1/2} U \right): \\
		 & P = \Law(Y, G, U) \in \bigcap_{\veps \in (0, \veps_0)} \cuS_m^{\varphi} \left( \sqrt{\frac{m}{\alpha}} \veps \sqrt{\log \left( \frac{C}{\veps} \right)}, \,\frac{m}{\alpha} \log \left( \frac{C}{\veps} \right) \right) \Bigg\}.
	\end{align*}
\end{thm}

\begin{proof}
	Similar to the proof of Theorem~\ref{thm:W2_outer_bd}, for each $i \in [n]$ we obtain the following decomposition
	\begin{equation*}
		\left( y_i, \WW^\top \xx_i \right) = \left( y_i, \WW^\top \VV \bgg_i + \WW^\top \VV_{\perp} \zz_i \right),
	\end{equation*}
	where $(y_i, \bgg_i) \perp \zz_i$, $\bgg_i \sim \normal (\bzero, \id_k)$, $\zz_i \sim \normal (\bzero, \id_{d-k})$, and
	\begin{equation*}
		\P \left( y_i = +1 \vert \bgg_i \right) = \varphi(\bgg_i) = 1 - \P \left( y_i = -1 \vert \bgg_i \right).
	\end{equation*}
	Denoting $
	\QQ = \WW^\top \VV \in \R^{m \times k}$, we know that
	\begin{equation*}
		\left( \WW^\top \VV_{\perp} \right) \left( \WW^\top \VV_{\perp} \right)^\top = \id_m - \QQ \QQ^\top.
	\end{equation*}
	Hence, there exists some $(d-k) \times m$ orthogonal matrix $\WW_{\perp} \in O(d-k, m)$ such that
	\begin{equation*}
		\WW^\top \VV_{\perp} = \left( \id_m - \QQ \QQ^\top \right)^{1/2} \WW_{\perp},
	\end{equation*}
	and consequently we can write
	\begin{equation*}
		\left( y_i, \WW^\top \xx_i \right) = \left( y_i, \QQ \bgg_i + \left( \id_m - \QQ \QQ^\top \right)^{1/2} \WW_{\perp} \zz_i \right), \quad \QQ \QQ^\top \preceq \id_m.
	\end{equation*}
	It thus follows that
	\begin{equation*}
		\hat{P}_{n, \WW} = \Law \left( Y_n, \, \QQ G_n + (\id_m - \QQ \QQ^\top)^{1/2} U_n \right),
	\end{equation*}
	where
	\begin{equation*}
		(Y_n, G_n, U_n) \sim \hat{Q}_{n, \WW_{\perp}} = \frac{1}{n} \sum_{i=1}^{n} \delta_{(y_i, \bgg_i, \WW_{\perp} \zz_i)}.
	\end{equation*}
	Using a similar argument as that in the proof of Theorem~\ref{thm:div_outer_bd}, we obtain the following:
	\begin{lem}\label{lem:KL_W2_joint}
		There exist absolute constants $C > 0$ and $\veps_0 > 0$ such that for any $\veps \in (0, \veps_0)$ and $p \in [1, 2)$, the following happens: With high probability, for any $\WW_{\perp} \in O(d-k, m)$, $\hat{Q}_{n, \WW_{\perp}}$ satisfies Condition (b) in the statement of Theorem~\ref{thm:KL_W2_sup} with $W_2$ replaced by $W_p$, and
		\begin{equation*}
			a = \sqrt{\frac{m}{\alpha}} \veps \sqrt{\log \left( \frac{C}{\veps} \right)}, \, b =\frac{m}{\alpha} \log \left( \frac{C}{\veps} \right).
		\end{equation*}
	\end{lem}
	The conclusion of Theorem~\ref{thm:KL_W2_sup} then follows naturally from combining Lemma~\ref{lem:KL_W2_joint} and the compactness argument in the second part of the proof of Theorem~\ref{thm:div_outer_bd}.
\end{proof}

\subsection{Interpolation threshold for two-layers neural network}\label{sec:small_nn_apply}

As mentioned at the beginning of this section, a bound on the feasibility set 
$\cuF^{\varphi}_{m,\alpha}$ can be used to bound the typical value of the training
error (minimum empirical risk) over certain function classes, in the proportional asymptotics.
Of particular interest is the case of zero training error, which corresponds
to interpolation\footnote{Notice that the vanishing of \eqref{eq:ERM-asymp}
does not imply exactly vanishing training error but $\hR_n^{\star}(\XX,\yy)$.
For instance, in the case of hinge loss this means that all but $o(n)$ points are 
classified correctly with margin $\kappa-o(1)$.}.
We provide an illustration of
this application by considering binary classification with a positive margin $\kappa>0$,
using two-layer networks with $m$ hidden neurons.

For $\kappa > 0$, consider the loss function $L(y, \hat{y}) = (\kappa - y \hat{y})_+$
and the following set of functions:
\begin{equation*}
	\cH_m = \left\{ h (\uu) = \frac{1}{\sqrt{m}} \sum_{i = 1}^{m} a_i \sigma \left( \langle \bb_i, \uu \rangle \right): \norm{\aa}_1 \le m B, \ \max_{i \in [m]} \norm{\bb_i}_2 \le B \right\},
\end{equation*}
where $B > 0$ is a constant. 

Recalling the definition of $\hat{R}_n^\star (\XX, \yy)$ in Eq.~\eqref{eq:EmpiricalRisk}, we have $\hat{R}_n^\star (\XX, \yy) = 0$
if and only if there exists a function $f (\xx; \hat{\aa}, \hat{\WW}) \in \cF_{\mathsf{NN}}^{m, B}$
such that $y_i f(\xx_i; \hat{\aa}, \hat{\WW}) \ge \kappa$ for all $i \in [n]$. Here, 
$\cF_{\mathsf{NN}}^{m, B}$  denotes the collection of two-layers neural networks with bounded first and second-layer coefficients:
\begin{equation}\label{eq:2_layer_NN}
	\cF_{\mathsf{NN}}^{m, B} = \Big\{ 
	f \left( \xx; \aa, \WW \right) = \frac{1}{\sqrt{m}} \sum_{j=1}^{m} a_j \sigma \left( \langle \ww_j, \xx \rangle \right), \ \norm{\aa}_1 \le m B, \ \max_{j \in [m]} \norm{\ww_j}_2 \le B \Big\}.
\end{equation}

Our next theorem applies the $W_2$ upper bound of the previous section to bound
the $\kappa$-margin interpolation threshold for this model.
\begin{thm}\label{thm:NN_upper_bd}
	Consider i.i.d. data $(\yy, \XX) = \{ (y_i, \xx_i) \}_{1 \le i \le n}$ where
	$y_i \sim \Unif (\{ + 1,-1 \})$ is independent of $\xx_i \sim \sN (\bzero, \id_d)$.
	Assume $n, d \to \infty$ with $m$ fixed, and further assume that
	$\sigma(x)$ is  $L$-Lipschitz. Then,
	\begin{equation}
		\frac{n}{md} > \frac{2 L^2 B^4}{\kappa^2} \;\;\implies\;\; \lim_{n,d \to \infty} \P \left( \hat{R}_n^\star (\XX, \yy) = 0 \right) = 0.
	\end{equation}
\end{thm}
\begin{rem}
	Notice that the bound of Theorem \ref{thm:W2_outer_bd} holds for $m=1$, while 
	the last theorem considers models with general $m\ge 1$. In the proof we show that
	the $m=1$ bound can be leveraged to control higher-dimensional cases as well.
\end{rem}

Notice thar the Lipschitz constant of $f\in	\cF_{\mathsf{NN}}^{m, B}$ is upper bounded
by $\overline{\rm Lip}=LB^2$. Hence the last theorem indicates that
a $\kappa$-margin interpolating network exists only if $md/n\ge\kappa^2/(2\overline{\rm Lip}^2)$.
The dependence on $m,d,n$ is the natural one: we expect vanishing training error to be possible only 
if the number of parameters $md$ is larger than the sample size $n$. Despite this 
is an intuitive necessary condition, we are not aware of a proof of this fact, even in the asymptotic setting
treated here. A recent paper \cite{montanari2020interpolation} proves 
that if a $\kappa$-margin solution exists, then $md/n\ge C_{L,B}/\log(d/\kappa)$
which is significantly weaker in large dimensions.

We also point out that the case $m=1$ corresponds to linear separability, and
has been studied in detail (see next section). In that case the condition 
above reduces to  $d/n\ge \kappa^2/(2\overline{\rm Lip}^2)$, which captures the known correct dependence on $\kappa$
for $\kappa$ bounded away from $0$.


\subsection{Distribution of margins for max-margin classification}\label{sec:max_margin_apply}

In this section we assume $m = k=1$ and write $\VV=\btheta_*\in \S^{d - 1}$. Hence, for $i \in [n]$,
we have $\xx_i\sim\normal(\bzero,\id_d)$ and $\P (y_i = + 1 \vert \xx_i) =  \varphi (\langle \xx_i, \btheta_* \rangle) = 1 - \P (y_i = - 1 \vert \xx_i)$,
cf. Eq.~\eqref{eq:SupervisedModel}.
The max-margin classifier is defined by
\begin{equation}
	\hat{\btheta}^{\mathrm{MM}} := \hat{\btheta}^{\mathrm{MM}} (\yy, \XX) = \argmax_{\norm{\btheta}_2 = 1} \min_{1 \le i \le n} \left\{ y_i \langle \btheta, \xx_i \rangle \right\}\, .
\end{equation}
Define  the random variables $(Y, G, Z)$ as in Eq.~\eqref{eq:key_triple}, for $k=1$.
We will assume that $\varphi$ is such that $\P(YG>x)\wedge \P(YG< -x)>0$ for all $x$ (Assumption 3 in \cite{montanari2019generalization}).

It was proven in \cite{candes2020phase,montanari2019generalization} that,
for $\kappa\ge 0$, a $\kappa$-margin solution exists
with high probability if $\alpha<\alpha_{\mathrm s} (\kappa)$,
and does not exist with high probability if $\alpha>\alpha_{\mathrm s} (\kappa)$.
The  critical threshold $\alpha_{\mathrm s} (\kappa)$ is defined as
\begin{equation}\label{eq:crt_thres}
	\alpha_{\mathrm s} (\kappa) := \max_{\rho \in [-1, 1]} F_{\kappa} (\rho) \, ,\;\;\;\;\;
	F_{\kappa} (\rho) = \frac{1 - \rho^2}{\E \left[ \left( \kappa - \rho YG - \sqrt{1 - \rho^2} Z \right)_+^2 \right]}\, .
\end{equation}
For $\alpha < \alpha_{\mathrm s} (0)$, we define $\kappa_{\rm s} (\alpha)$ 
to be the unique $\kappa > 0$ such that $\alpha = \alpha_{\mathrm s} (\kappa)$, whose existence is ensured by the fact that $\alpha_s (\kappa)$ is strictly decreasing.

Here we are interested in the distribution of margins for the max-margin solution.
To be accurate, we denote the  empirical distribution of  margins by
\begin{equation}\label{eq:max_margin_dist}
	\hat{P}_{n, \hat{\btheta}^{\mathrm{MM}}} = \frac{1}{n} \sum_{i = 1}^{n} \delta_{y_i \left\langle \hat{\btheta}^{\mathrm{MM}}, \xx_i \right\rangle}.
\end{equation}
The margin distribution 
provides useful information about the structure of the 
max-margin classifier.
In order to state our result, we need to establish the following analytical fact. 
\begin{lem}\label{lem:unique_F_kappa}
	The function $F_{\kappa} (\rho)$ defined in Eq.~\eqref{eq:crt_thres} has a unique maximizer $\rho_* (\kappa) \in (- 1, 1)$.
\end{lem}

Following the notation of this lemma, we define
\begin{equation}\label{eq:mm_limit_dist}
	P_{\kappa, \varphi} = \Law \left( \max \left( \kappa, \rho_* (\kappa) YG + \sqrt{1 - \rho_*(\kappa)^2} Z \right) \right).
\end{equation}
The next theorem is proved by an application of the $W_2$ bound
of Theorem \ref{thm:W2_outer_bd}.
\begin{thm}\label{thm:mm_margin_dist}
	Consider i.i.d. data $\{ (y_i, \xx_i) \}_{1 \le i \le n}$ where 
	$\xx_i \sim \sN (\bzero, \id_d)$ and $y_i\in\{+1,-1\}$ are such that
	$\P (y_i = + 1 \vert \xx_i)  =\varphi (\langle \xx_i, \btheta_* \rangle) = 1 - \P (y_i = - 1 \vert \xx_i)$.
	Assume $n, d \to \infty$ with $n / d \to \alpha \in ( 0, \alpha_{\mathrm s} (0) )$ 
	and denote $\kappa = \kappa_{\mathrm s} (\alpha)$. 
	Recall 
	$\hat{P}_{n, \hat{\btheta}^{\mathrm{MM}}}$ from Eq.~\eqref{eq:max_margin_dist} and
	$P_{\kappa, \varphi}$ from Eq.~\eqref{eq:mm_limit_dist}. Then we have
	\begin{equation*}
		\plim_{n\to\infty }W_2 \big( \hat{P}_{n, \hat{\btheta}^{\mathrm{MM}}}, P_{\kappa, \varphi} \big) =0\, .
	\end{equation*}
\end{thm} 

\begin{rem}
	This result is already non-trivial in the case of purely random labels, i.e.,  if
	$\varphi(x) = \varphi_0(x) = 1/2$ identically. In this case $\rho_* (\kappa) \equiv 0$ and the limiting distribution 
	is $P_{\kappa, \varphi_0} = \Law(\max(\kappa, Z))$ where $Z \sim \sN (0, 1)$. In other words, $P_{\kappa,\varphi_0}$
	is a Gaussian truncated at $\kappa$, with the missing density replaced by a point mass at $\kappa$.
	It is easy to check that
	the mass at $\kappa$ is equal to $1 - \Phi(\kappa)$:
	\begin{align*}
		P_{\kappa,\varphi_0}(\d x) = (1 - \Phi(\kappa)) \delta_{\kappa}+ \phi(x){\bf 1}_{x\ge \kappa}\d x\,,
	\end{align*}
	with $\phi(x)$ the standard Gaussian density. In words, roughly $(1 - \Phi(\kappa)) n$ of the training samples
	have margin $\kappa$, and the others have Gaussian margins, conditional on being at least $\kappa$.
\end{rem}

\section{Proof techniques}\label{sec:ProofIdeas}
In the previous sections we presented two types of general outer bounds for the set of 
$(\alpha, m)$-feasible probability measures: $(i)$~Tight uniform upper bound on
$W_2 (P, \sN (0, 1))$, where $P \in \cuF_{1, \alpha}$; $(ii)$~Characterization of 
$\cuF_{m,\alpha}$ via the KL-Wasserstein neighborhood for general $m \ge 1$.
We hereby briefly describe the methods we use to prove these two types of results, 
with most of the technical work deferred to the appendices. We will mainly focus on the 
unsupervised case: the supervised case is technically more cumbersome, but can be 
addressed using the same ideas. We also illustrate the lower bound on the $W_2$ 
radius by constructing specific feasible distributions on $\R^m$, and sketch the proof of the 
$\chi^2$-KL inner bound via  the second moment method.

\noindent \textit{Wasserstein outer bound for $m = 1$.} We begin with defining the following random variable:
\begin{align*}
	\xi_{n} = & \frac{1}{\sqrt{\alpha}} - \max_{\norm{\ww}_2 = 1} W_2 \left( \frac{1}{n} \sum_{i = 1}^{n} \delta_{\langle \xx_i, \ww \rangle}, \sN (0, 1) \right) \\
	= & \min_{\norm{\ww}_2 = 1} \left\{ \frac{1}{\sqrt{\alpha}} -  W_2 \left( \frac{1}{n} \sum_{i = 1}^{n} \delta_{\langle \xx_i, \ww \rangle}, \sN (0, 1) \right) \right\} \\
	= & \min_{\norm{\ww}_2 = 1, \uu \in \R^n} \max_{\blambda \in \R^n} \left\{ \frac{1}{\sqrt{\alpha}} -  W_2 \left( \frac{1}{n} \sum_{i = 1}^{n} \delta_{u_i}, \sN (0, 1) \right) + \frac{1}{n} \blambda^\top \left( \uu - \XX \ww \right) \right\}.
\end{align*}
Note that a uniform $W_2$ upper bound is equivalent to a lower bound on $\xi_n$. 
Applying a variant of Gordon's Gaussian comparison 
inequality \cite{gordon1985some,thrampoulidis2015regularized} 
(with some additional technical work since the domains are unbounded) then yields
\begin{equation*}
	\P \left( \xi_{n} \le t \right) \le 2 \P \left( \xi_{n}^{(1)} \le t \right), \ \forall t \in \R,
\end{equation*}
where
\begin{equation*}
	\xi_n^{(1)} = \min_{\norm{\ww}_2 = 1, \uu \in \R^n} \max_{\blambda \in \R^n} \left\{ \frac{1}{\sqrt{\alpha}} -  W_2 \left( \frac{1}{n} \sum_{i = 1}^{n} \delta_{u_i}, \sN (0, 1) \right) + \frac{1}{n} \blambda^\top \left( \uu - \norm{\ww}_2 \hh \right) + \frac{1}{n} \norm{\blambda}_2 \ww^\top \bgg \right\}.
\end{equation*}
Here, $\bgg \sim \sN (\bzero, \id_d)$ and $\hh \sim \sN (\bzero, \id_n)$ are mutually independent, and further independent of $\XX$. 

It then suffices to obtain a high-probability lower bound for $\xi_n^{(1)}$. By direct calculation,
\begin{align*}
	\xi_n^{(1)} = & \min_{\uu \in \R^n} \max_{\blambda \in \R^n} \left\{ \frac{1}{\sqrt{\alpha}} -  W_2 \left( \frac{1}{n} \sum_{i = 1}^{n} \delta_{u_i}, \sN (0, 1) \right) + \frac{1}{n} \blambda^\top \left( \uu - \hh \right) - \frac{1}{n} \norm{\blambda}_2 \norm{\bgg}_2 \right\} \\
	= & \min_{ \uu \in \R^n} \max_{\gamma \ge 0} \left\{ \frac{1}{\sqrt{\alpha}} -  W_2 \left( \frac{1}{n} \sum_{i = 1}^{n} \delta_{u_i}, \sN (0, 1) \right) + \frac{\gamma}{\sqrt{n}} \left( \norm{\uu - \hh}_2 - \norm{\bgg}_2 \right) \right\} \\
	\ge & \min_{\uu \in \R^n} \max_{0 \le \gamma \le 1} \left\{ \frac{1}{\sqrt{\alpha}} -  W_2 \left( \frac{1}{n} \sum_{i = 1}^{n} \delta_{u_i}, \sN (0, 1) \right) + \frac{\gamma}{\sqrt{n}} \left( \norm{\uu - \hh}_2 - \norm{\bgg}_2 \right) \right\}.
\end{align*}
By the law of large numbers, with high probability we have
\begin{equation*}
	\frac{1}{\sqrt{n}} \left( \norm{\uu - \hh}_2 - \norm{\bgg}_2 \right) \ge W_2 \left( \frac{1}{n} \sum_{i = 1}^{n} \delta_{u_i}, \sN (0, 1) \right) - \frac{1}{\sqrt{\alpha}} - \veps
\end{equation*} 
for an (arbitrarily) small $\veps > 0$, which in turn implies $\xi_n^{(1)} \ge - \veps$. Hence, $\P \left( \xi_{n} \le - \veps \right) \to 0$ as $n \to \infty$. Consequently, the following holds with high probability:
\begin{equation*}
	\max_{\norm{\ww}_2 = 1} W_2 \left( \frac{1}{n} \sum_{i = 1}^{n} \delta_{\langle \xx_i, \ww \rangle}, \sN (0, 1) \right) \le \frac{1}{\sqrt{\alpha}} + \veps.
\end{equation*}
Combining the above uniform (for $\ww \in \S^{d - 1}$) upper bound with a compactness argument implies that any feasible distribution $P$ must satisfy $W_2 (P, \sN(0, 1)) \le 1 / \sqrt{\alpha}$.
\\

\noindent \textit{KL-Wasserstein outer bound for general $m$.} The proof follows from a combination of the chaining method and covering argument on the Stiefel manifold $O(d, m)$. Assume $P \in \cuF_{m,\alpha}$, then we use a variant of Dudley's entropy integral to show that, with high probability there exists a $\WW \in N_{\veps} (d, m)$ such that $W_2 (\hat{P}_{n, \WW}, P) = O(\veps \sqrt{m / \alpha} )$, where $N_{\veps} (d, m)$ is an $\veps$-covering of $O(d, m)$ and the big-$O$ notation only hides universal constant. Here, $\hat{P}_{n, \WW}$ denotes the empirical distribution $(1 / n) \sum_{i = 1}^{n} \delta_{\xx_i^\top \WW}$. This calculation implies that $\cuF_{m,\alpha}$ is contained in a $W_2$-neighborhood with radius $O(\veps \sqrt{m / \alpha})$ of the set $\{ \hat{P}_{n, \WW}: \WW \in N_{\veps} (d, m) \}$.

To prove Theorem~\ref{thm:div_outer_bd}, we need to show that $\{ \hat{P}_{n, \WW}: \WW \in N_{\veps} (d, m) \}$ is contained in a KL-divergence ball centered at $\sN (\bzero, \id_m)$. For any fixed $\WW \in N_{\veps} (d, m)$ and $a > 0$, the KL divergence between $\hat{P}_{n, \WW}$ and $\sN (\bzero, \id_m)$ is upper bounded by $a$ with probability at least $1 - \exp(- n (a + o(1)))$ (Sanov's theorem, see, e.g., \cite[Thm. 6.2.10]{Dembo_2010}). This result is a typical application of large deviations theory for empirical distributions. Now we can apply the union bound to $\WW \in N_{\veps} (d, m)$. Properly choosing $a = a(\veps)$ yields a high-probability upper bound on $\max_{\WW \in N_{\veps} (d, m)} D_{\rm KL} (\hat{P}_{n, \WW} \Vert \sN (\bzero, \id_m))$, which, together with the $W_2$ approximation, gives the desired KL-Wasserstein outer bound.
\\

\noindent \textit{Lower bound on $W_2$ radius.} The lower bound on the $W_2$ radius $\sup\{W_2(P,\normal(\bzero,\id_m)): \; P\in \cuF_{m,\alpha}\}$
is obtained by taking $\WW = (\vv_1(\XX),\dots,\vv_m(\XX))$, where
$\vv_{i}(\XX)\in \R^d$ is the $i$-th right singular vector of $\XX$.
We therefore have $\XX\WW = (s_1\uu_1(\XX),\dots,s_m\uu_m(\XX))$,
where $s_i$ is the $i$-th singular value and $\uu_i(\XX)$ is the $i$-th
left singular vector of $\XX$.
By the Bai-Yin law (see, e.g., \cite{bai2010spectral}), $s_1,\dots,s_m \to (1+\alpha^{-1/2})$ almost surely, and
by rotational invariance $(\uu_1(\XX),\dots,\uu_m(\XX))\in O(n,m)$ is uniformly random on the 
Stiefel manifold, whence it is easy to show that, in probability
\begin{align*}
	\frac{1}{n}\sum_{i=1}^m\delta_{\WW^{\top}\xx_i} \stackrel{w}{\Rightarrow} P =\normal(\bzero,(1+\alpha^{-1/2})^2\id_m)\, .
\end{align*}
It is then sufficient to compute $W_2(P,\normal(\bzero,\id_m))= \sqrt{m / \alpha}$.
\\

\noindent \textit{$\chi^2$-KL divergence inner bound.} As described in Section 
\ref{sec:2nd_moment}, the proof follows the following scheme: $(1)$~Discretize the probability 
distribution $P$ using a finitely supported distribution $P_A$;
$(2)$~Prove a feasibility result for the discrete distribution $P_A$ 
(Theorem \ref{thm:discrete_lower_bd});
$(3)$~Control the limit of finer and finer discretizations
(Theorem \ref{thm:inner_bd_1D} and Theorem \ref{thm:inner_bd_mD}). 

The proof of Theorem \ref{thm:discrete_lower_bd} is based on the second moment 
method and is the most technical part of the paper.
 The second moment method has been successful in proving existence of solutions for
 random constraint satisfaction problems (CSPs) \cite{achlioptas2002asymptotic, achlioptas2005rigorous, achlioptas2006random}. 
 Usually, one constructs a non-negative random variable $Z = Z(\XX)$ such that $Z > 0$
  if and only if a solution exists. The Paley-Zygmund inequality can then be used to
   lower bound the probability that there exists a solution:
\begin{equation*}
	\P \left( \text{A solution exists} \right) = \P (Z > 0) \ge \frac{\E [Z]^2}{\E [Z^2]}.
\end{equation*}
In the context of projection pursuit, one seeks a projection matrix $\WW$ such that $\hat{P}_{n, \WW}$ is close to the target distribution $P$. Therefore, a nature choice for $Z$ would be
\begin{equation*}
	Z:= \int_{O(d, m)} \bone \left\{ \hat{P}_{n, \WW} \in \mathsf{B}_{\veps} (P) \right\} \mu_{d, m} (\d \WW),
\end{equation*}
where $\mu_{d, m}$ is the uniform measure on $O(d, m)$, and $\mathsf{B}_{\veps} (P)$ is a small 
neighborhood of $P$ in $\cuP (\R^m)$ with respect to some topology. We then calculate the first and second moments of $Z$, and show that 
$\E [Z^2] = O(1) \cdot \E [Z]^2$ for $\alpha$ below a certain feasibility threshold. When computing these moments, 
we use a refined version of the classical Sanov's theorem to obtain exact asymptotics
 of the large deviation probability $\P (\hat{P}_{n, \WW} \in \mathsf{B}_{\veps} (P))$, which 
 can be found in \cite{dinwoodie1992mesures, ney1983dominating}.

\section*{Acknowledgements}

This work was supported by the NSF through award DMS-2031883, the Simons Foundation through Award 
814639 for the Collaboration on the Theoretical Foundations of Deep Learning, the NSF grant CCF-2006489, the ONR grant N00014-18-1-2729.

\newpage

\bibliographystyle{alpha}
\bibliography{proj_purs}

\newcommand{\etalchar}[1]{$^{#1}$}
\begin{thebibliography}{VDVW96}

\bibitem[AM02]{achlioptas2002asymptotic}
Dimitris Achlioptas and Cristopher Moore.
\newblock The asymptotic order of the random k-sat threshold.
\newblock In {\em The 43rd Annual IEEE Symposium on Foundations of Computer
  Science, 2002. Proceedings.}, pages 779--788. IEEE, 2002.

\bibitem[AM06]{achlioptas2006random}
Dimitris Achlioptas and Cristopher Moore.
\newblock Random k-sat: Two moments suffice to cross a sharp threshold.
\newblock {\em SIAM Journal on Computing}, 36(3):740--762, 2006.

\bibitem[ANP05]{achlioptas2005rigorous}
Dimitris Achlioptas, Assaf Naor, and Yuval Peres.
\newblock Rigorous location of phase transitions in hard optimization problems.
\newblock {\em Nature}, 435(7043):759--764, 2005.

\bibitem[BKN18]{bickel2018projection}
Peter~J Bickel, Gil Kur, and Boaz Nadler.
\newblock Projection pursuit in high dimensions.
\newblock {\em Proceedings of the National Academy of Sciences},
  115(37):9151--9156, 2018.

\bibitem[BKS{\etalchar{+}}06]{blanchard2006search}
Gilles Blanchard, Motoaki Kawanabe, Masashi Sugiyama, Vladimir Spokoiny,
  Klaus-Robert M{\"u}ller, and Sam Roweis.
\newblock In search of non-gaussian components of a high-dimensional
  distribution.
\newblock {\em Journal of Machine Learning Research}, 7(2), 2006.

\bibitem[BLM13]{boucheron2013concentration}
St{\'e}phane Boucheron, G{\'a}bor Lugosi, and Pascal Massart.
\newblock {\em Concentration inequalities: A nonasymptotic theory of
  independence}.
\newblock Oxford university press, 2013.

\bibitem[BS10]{bai2010spectral}
Zhidong Bai and Jack~W Silverstein.
\newblock {\em Spectral analysis of large dimensional random matrices},
  volume~20.
\newblock Springer, 2010.

\bibitem[CS20]{candes2020phase}
Emmanuel~J Cand{\`e}s and Pragya Sur.
\newblock The phase transition for the existence of the maximum likelihood
  estimate in high-dimensional logistic regression.
\newblock {\em The Annals of Statistics}, 48(1):27--42, 2020.

\bibitem[DB81]{de1981asymptotic}
Nicolaas~Govert De~Bruijn.
\newblock {\em Asymptotic methods in analysis}, volume~4.
\newblock Courier Corporation, 1981.

\bibitem[DEL92]{diaconis1992finite}
Persi~W Diaconis, Morris~L Eaton, and Steffen~L Lauritzen.
\newblock {Finite de Finetti theorems in linear models and multivariate
  analysis}.
\newblock {\em Scandinavian Journal of Statistics}, pages 289--315, 1992.

\bibitem[DF84]{diaconis1984asymptotics}
Persi Diaconis and David Freedman.
\newblock Asymptotics of graphical projection pursuit.
\newblock {\em The annals of statistics}, pages 793--815, 1984.

\bibitem[DH18]{dudeja2018learning}
Rishabh Dudeja and Daniel Hsu.
\newblock Learning single-index models in gaussian space.
\newblock In {\em Conference On Learning Theory}, pages 1887--1930. PMLR, 2018.

\bibitem[Din92]{dinwoodie1992mesures}
Ian~H Dinwoodie.
\newblock Mesures dominantes et th{\'e}oreme de sanov.
\newblock In {\em Annales de l'IHP Probabilit{\'e}s et statistiques},
  volume~28, pages 365--373, 1992.

\bibitem[Dur19]{durrett2019probability}
Rick Durrett.
\newblock {\em Probability: theory and examples}, volume~49.
\newblock Cambridge university press, 2019.

\bibitem[DZ10]{Dembo_2010}
Amir Dembo and Ofer Zeitouni.
\newblock {\em Large Deviations Techniques and Applications}.
\newblock Springer Berlin Heidelberg, 2010.

\bibitem[Eat89]{eaton1989group}
Morris~L Eaton.
\newblock Group invariance applications in statistics.
\newblock IMS, 1989.

\bibitem[ES96]{eichelsbacher1996large}
Peter Eichelsbacher and Uwe Schmock.
\newblock {\em Large deviations of products of empirical measures and
  U-Empirical measures in strong topologies}.
\newblock Univ. Bielefeld, Sonderforschungsbereich 343, Diskrete Strukturen in
  der Math., 1996.

\bibitem[FG15]{fournier2015rate}
Nicolas Fournier and Arnaud Guillin.
\newblock On the rate of convergence in wasserstein distance of the empirical
  measure.
\newblock {\em Probability Theory and Related Fields}, 162(3):707--738, 2015.

\bibitem[FT74]{friedman1974projection}
Jerome~H Friedman and John~W Tukey.
\newblock A projection pursuit algorithm for exploratory data analysis.
\newblock {\em IEEE Transactions on computers}, 100(9):881--890, 1974.

\bibitem[Gor85]{gordon1985some}
Yehoram Gordon.
\newblock Some inequalities for gaussian processes and applications.
\newblock {\em Israel Journal of Mathematics}, 50(4):265--289, 1985.

\bibitem[HO00]{hyvarinen2000independent}
Aapo Hyv{\"a}rinen and Erkki Oja.
\newblock Independent component analysis: algorithms and applications.
\newblock {\em Neural networks}, 13(4-5):411--430, 2000.

\bibitem[HPV17]{hinrichs2017entropy}
Aicke Hinrichs, Joscha Prochno, and Jan Vybiral.
\newblock Entropy numbers of embeddings of schatten classes.
\newblock {\em Journal of Functional Analysis}, 273(10):3241--3261, 2017.

\bibitem[Lop18]{loperfido2018skewness}
Nicola Loperfido.
\newblock Skewness-based projection pursuit: A computational approach.
\newblock {\em Computational Statistics \& Data Analysis}, 120:42--57, 2018.

\bibitem[LWDF11]{levin2011understanding}
Anat Levin, Yair Weiss, Fredo Durand, and William~T Freeman.
\newblock Understanding blind deconvolution algorithms.
\newblock {\em IEEE transactions on pattern analysis and machine intelligence},
  33(12):2354--2367, 2011.

\bibitem[MM21]{miolane2021distribution}
L{\'e}o Miolane and Andrea Montanari.
\newblock The distribution of the lasso: Uniform control over sparse balls and
  adaptive parameter tuning.
\newblock {\em The Annals of Statistics}, 49(4):2313--2335, 2021.

\bibitem[MRSY19]{montanari2019generalization}
Andrea Montanari, Feng Ruan, Youngtak Sohn, and Jun Yan.
\newblock The generalization error of max-margin linear classifiers:
  High-dimensional asymptotics in the overparametrized regime.
\newblock {\em arXiv preprint arXiv:1911.01544}, 2019.

\bibitem[MZ20]{montanari2020interpolation}
Andrea Montanari and Yiqiao Zhong.
\newblock The interpolation phase transition in neural networks: Memorization
  and generalization under lazy training.
\newblock {\em arXiv preprint arXiv:2007.12826}, 2020.

\bibitem[MZZ21]{montanari2021tractability}
Andrea Montanari, Yiqiao Zhong, and Kangjie Zhou.
\newblock Tractability from overparametrization: The example of the negative
  perceptron.
\newblock {\em arXiv preprint arXiv:2110.15824}, 2021.

\bibitem[Ney83]{ney1983dominating}
Peter Ney.
\newblock Dominating points and the asymptotics of large deviations for random
  walk on rd.
\newblock {\em The Annals of Probability}, pages 158--167, 1983.

\bibitem[O'D14]{o2014analysis}
Ryan O'Donnell.
\newblock {\em Analysis of boolean functions}.
\newblock Cambridge University Press, 2014.

\bibitem[Oui21]{ouimet2021precise}
Fr{\'e}d{\'e}ric Ouimet.
\newblock A precise local limit theorem for the multinomial distribution and
  some applications.
\newblock {\em Journal of Statistical Planning and Inference}, 215:218--233,
  2021.

\bibitem[Rah17]{rahman2017wiener}
Sharif Rahman.
\newblock Wiener--hermite polynomial expansion for multivariate gaussian
  probability measures.
\newblock {\em Journal of Mathematical Analysis and Applications},
  454(1):303--334, 2017.

\bibitem[R{\'e}n59]{renyi1959dimension}
Alfr{\'e}d R{\'e}nyi.
\newblock On the dimension and entropy of probability distributions.
\newblock {\em Acta Mathematica Academiae Scientiarum Hungarica},
  10(1-2):193--215, 1959.

\bibitem[RV10]{rudelson2010non}
Mark Rudelson and Roman Vershynin.
\newblock Non-asymptotic theory of random matrices: extreme singular values.
\newblock In {\em Proceedings of the International Congress of Mathematicians
  2010 (ICM 2010) (In 4 Volumes) Vol. I: Plenary Lectures and Ceremonies Vols.
  II--IV: Invited Lectures}, pages 1576--1602. World Scientific, 2010.

\bibitem[SF84]{siotani1984asymptotic}
Minoru Siotani and Yasunori Fujikoshi.
\newblock Asymptotic approximations for the distributions of multinomial
  goodness-of-fit statistics.
\newblock {\em Hiroshima mathematical journal}, 14(1):115--124, 1984.

\bibitem[SNS16]{sasaki2016non}
Hiroaki Sasaki, Gang Niu, and Masashi Sugiyama.
\newblock Non-gaussian component analysis with log-density gradient estimation.
\newblock In {\em Artificial Intelligence and Statistics}, pages 1177--1185.
  PMLR, 2016.

\bibitem[Sto13]{stojnic2013negative}
Mihailo Stojnic.
\newblock Negative spherical perceptron.
\newblock {\em arXiv preprint arXiv:1306.3980}, 2013.

\bibitem[Tal96]{talagrand1996transportation}
Michel Talagrand.
\newblock Transportation cost for gaussian and other product measures.
\newblock {\em Geometric \& Functional Analysis GAFA}, 6(3):587--600, 1996.

\bibitem[Tal10]{talagrand2010mean}
Michel Talagrand.
\newblock {\em Mean field models for spin glasses: Volume I: Basic examples},
  volume~54.
\newblock Springer Science \& Business Media, 2010.

\bibitem[TOH15]{thrampoulidis2015regularized}
Christos Thrampoulidis, Samet Oymak, and Babak Hassibi.
\newblock Regularized linear regression: A precise analysis of the estimation
  error.
\newblock {\em Proceedings of Machine Learning Research}, 40:1683--1709, 2015.

\bibitem[VDVW96]{van1996weak}
Aad~W Van Der~Vaart and Jon Wellner.
\newblock {\em Weak convergence and empirical processes: with applications to
  statistics}.
\newblock Springer Science \& Business Media, 1996.

\bibitem[Ver18]{vershynin2018high}
Roman Vershynin.
\newblock {\em High-dimensional probability: An introduction with applications
  in data science}, volume~47.
\newblock Cambridge university press, 2018.

\bibitem[Vil09]{villani2009optimal}
C{\'e}dric Villani.
\newblock {\em Optimal transport: old and new}, volume 338.
\newblock Springer, 2009.

\bibitem[WWW10]{Wang2010}
Ran Wang, Xinyi Wang, and Liming Wu.
\newblock {Sanov's theorem in the {W}asserstein distance: a necessary and
  sufficient condition}.
\newblock {\em Statistics \& Probability Letters}, 80(5-6):505--512, 2010.

\end{thebibliography}

\newpage

\appendix

\section{Notations}\label{sec:proof_main}

Throughout the proof, we denote by $\norm{\XX}_{\op}$ and $\norm{\XX}_{\rm F}$ the operator norm and Frobenius 
norm of a matrix $\XX$, respectively. We use $C$ and $\{ C_i \}_{i \ge 0}$ to refer to constants independent of $n$ and $d$, whereas the values of $C$ and $\{ C_i \}_{i \ge 0}$ can change from line to line. We always assume the proportional limit $n, d \to \infty$ with $n / d \to \alpha \in (0, \infty)$, which is often abbreviated as $d \to \infty$ (or $n \to \infty$). For example, when we write $\lim_{d \to \infty}$, it is understood that the limit is taken under $n = n(d)$ and $n(d) / d \to \alpha$. We will use $O_n(\cdot)$ and $o_n(\cdot)$ for the standard big-$O$ and small-$o$ notation, where $n$ is the asymptotic variable. We occasionally write $a_n \gg b_n$ or $b_n \ll a_n$ if $b_n = o_n(a_n)$. We will write $A \sim B$ if $A = A(n)$ and $B = B(n)$ is such that $\lim_{n \to \infty} A(n) / B(n) = 1$.

\section{Proofs for Section~\ref{sec:MainUnsupervised}: Unsupervised learning outer bounds}

\subsection{Proof of Theorem~\ref{thm:outer_bound_1_dim}}
Specializing the conclusion of Theorem~\ref{thm:W2_outer_bd} to the case $\varphi = \varphi_0 \equiv 1/2$, we obtain that for any $\eta > 0$, with probability one,
\begin{equation*}
	\limsup_{d \to \infty} \max_{\norm{\ww}_2 = 1} W_2^{(\eta)} \left( \frac{1}{n} \sum_{i = 1}^{n} \delta_{ \left( y_i, \langle \xx_i, \ww \rangle \right)}, \Unif(\{ \pm 1 \}) \otimes \sN (0, 1) \right) \le \frac{1}{\sqrt{\alpha}},
\end{equation*}
where $y_i \sim \Unif(\{ \pm 1 \})$ is independent of $\xx_i \sim \sN (\bzero, \id_d)$ for $i \in [n]$. This further implies that
\begin{equation*}
	\limsup_{d \to \infty} \max_{\norm{\ww}_2 = 1} W_2 \left( \frac{1}{n} \sum_{i = 1}^{n} \delta_{ \langle \xx_i, \ww \rangle }, \sN (0, 1) \right) \le \frac{1}{\sqrt{\alpha}}.
\end{equation*}
Assume $P$ is $(\alpha, 1)$-feasible, then there exists a sequence of random vectors $\ww \in \S^{d - 1}$ such that
\begin{equation*}
	\hat{P}_{n, \ww} = \frac{1}{n} \sum_{i = 1}^{n} \delta_{\langle \xx_i, \ww \rangle} \stackrel{w}{\Rightarrow} P \ \text{in probability}.
\end{equation*}
By Lemma~\ref{lem:BL_diatance_conv}, $d_{\rm BL} (\hat{P}_{n, \ww}, P) \stackrel{p}{\to} 0$. Hence, we can choose a subsequence $\{ n_k \}$ such that $d_{\rm BL} (\hat{P}_{n_k, \ww}, P) \to 0$ almost surely. Along this subsequence, we have
\begin{equation*}
	\limsup_{k \to \infty} W_2 \left( \hat{P}_{n_k, \ww}, \sN (0, 1) \right) \le \frac{1}{\sqrt{\alpha}} \ \text{almost surely}.
\end{equation*}
According to Corollary 6.11 in \cite{villani2009optimal}, the $W_2$ distance is lower semicontinuous (with respect to the weak topology on $\cuP(\R)$), so we immediately get that
\begin{equation*}
	W_2 \left( P, \sN(0, 1) \right) \le \limsup_{k \to \infty} W_2 \left( \hat{P}_{n_k, \ww}, \sN (0, 1) \right) \le \frac{1}{\sqrt{\alpha}}.
\end{equation*}
This proves the upper bound, i.e.,
\begin{equation*}
	\sup \{ W_2 \left( P, \sN(0, 1) \right): P \in \cuF_{1, \alpha} \} \le \frac{1}{\sqrt{\alpha}}.
\end{equation*}
Now we show that $1 / \sqrt{\alpha}$ is achievable. Let $\XX = \UU \bD \VV^\top$ be the singular value decomposition of $\XX$, where $\UU = [\uu_1, \cdots, \uu_n] \in \R^{n \times n}$ are the left singular vectors of $\XX$, $\bD \in \R^{n \times d}$ is a diagonal matrix whose diagonal elements correspond to the singular values of $\XX$, which we denote as $\{ s_i \}_{1 \le i \le \min(n, d)}$, and $\VV = [\vv_1, \cdots, \vv_d] \in \R^{d \times d}$ are the right singular vectors of $\XX$. Taking $\ww = \vv_1$, then it follows that
\begin{equation*}
	\XX \ww = \UU \bD \VV^\top \vv_1 = \UU \bD \ee_1 = s_1 \UU \ee_1 = s_1 \uu_1.
\end{equation*}
According to \cite[Theorem 5.11]{bai2010spectral}, we know that $s_1 / \sqrt{n} \to 1 + 1 / \sqrt{\alpha}$ almost surely as $n \to \infty$. Moreover, due to rotational invariance of Gaussian random matrix we deduce that $\uu_1 \sim \Unif (\S^{n - 1})$. Then, Lemma~\ref{lem:unif_empr_dist} implies that as $n \to \infty$,
\begin{equation*}
	\frac{1}{n} \sum_{i = 1}^{n} \delta_{\sqrt{n} u_{1,i}} \stackrel{w}{\Rightarrow} \sN (0, 1) \ \text{in probability}.
\end{equation*}
Applying Slutsky's theorem yields that
\begin{equation*}
	\hat{P}_{n, \ww} = \frac{1}{n} \sum_{i = 1}^{n} \delta_{\langle \xx_i, \ww \rangle} = \frac{1}{n} \sum_{i = 1}^{n} \delta_{s_1 u_{1,i}} \stackrel{w}{\Rightarrow} \left( 1 + \frac{1}{\sqrt{\alpha}} \right) \sN (0, 1) = \sN \left( 0, \left( 1 + \frac{1}{\sqrt{\alpha}} \right)^2 \right) := P \ \text{in probability}.
\end{equation*}
Therefore, $P$ is $(\alpha, 1)$-feasible. Now, let $G_1$ and $G_2$ be two $\sN (0, 1)$ random variables, then we have
\begin{align*}
	W_2 \left( P, \sN(0, 1) \right)^2 = & \inf_{\Gamma(G_1, G_2)} \E \left[ \left( \left( 1 + \frac{1}{\sqrt{\alpha}} \right) G_1 - G_2 \right)^2 \right] \\
	= & \left( 1 + \frac{1}{\sqrt{\alpha}} \right)^2 + 1 - 2 \left( 1 + \frac{1}{\sqrt{\alpha}} \right) \cdot \sup_{\Gamma(G_1, G_2)} \E \left[ G_1 G_2 \right] \\
	\stackrel{(i)}{=} & \left( 1 + \frac{1}{\sqrt{\alpha}} \right)^2 + 1 - 2 \left( 1 + \frac{1}{\sqrt{\alpha}} \right) = \frac{1}{\alpha},
\end{align*}
where the equality $(i)$ is achieved if and only if $G_1 = G_2$. We thus obtain that $W_2 (P, \sN (0, 1)) = 1 / \sqrt{\alpha}$ and the desired result follows naturally.

\subsection{Proof of Theorem~\ref{thm:div_outer_bd}}
This proof consists of two parts: (i) We show a KL-Wasserstein outer bound property as described in the Theorem statement with $W_2$ distance replaced by $W_p$ for any $p \in [1, 2)$. This is necessary since we need a variant of Sanov's theorem which is valid only for $p \in [1, 2)$; (ii) Then, we send $p \to 2$ to obtain the KL-Wasserstein outer bound for $W_2$ distance, making use of the crucial Lemma~\ref{lem:Wp_continuity}.

\noindent {\bf Part (i).} Choose any $\veps_0 < 1$ and fix $p \in [1, 2)$, we will show that for any $\veps \in (0, \veps_0)$ and $\eta > 0$, there exist constants $C_i = C_i (\eta, \veps) > 0$, $i = 0, 1$, such that with probability at least $1 - C_0 \exp(-C_1 n)$, the following happens: For all $\WW \in O (d, m)$, we can find a $Q \in \cuP (\R^m)$ satisfying
\begin{equation}\label{eq:W_p_KL_bound}
	W_p \left( \hat{P}_{n, \WW}, Q \right) \le (1 + \eta) \sqrt{\frac{m}{\alpha}} \veps \sqrt{\log \left( \frac{C}{\veps} \right)}, \ D_{\rm KL} \left( Q \Vert \sN (\bzero, \id_m) \right) \le (1 + \eta) \frac{m}{\alpha} \log \left( \frac{C}{\veps} \right),
\end{equation}
where we define $\hat{P}_{n, \WW} = (1 / n) \sum_{i = 1}^{n} \delta_{\xx_i^\top \WW}$. Similar to the proof of Theorem~\ref{thm:W2_outer_bd}, we note that there exist positive constants $C_0 (\eta)$ and $C_1 (\eta)$ such that $\P (\norm{\XX}_{\rm op} > (1 + \eta / 2)(1 + 1 / \sqrt{\alpha}) \sqrt{n} ) \le C_0 \exp ( - C_1 n)$, see, e.g., Eq. (2.3) in \cite{rudelson2010non}. Hence, without loss of generality we may assume from now on that $\norm{\XX}_{\rm op} \le (1 + \eta / 2)(1 + 1 / \sqrt{\alpha}) \sqrt{n}$, all future statements will be proven conditioned on this event that happens with high probability.

Recall that $O (d, m)$ denotes the set of all $d \times m$ orthogonal matrices, namely
\begin{equation*}
	O (d, m) = \left\{ \WW \in \R^{d \times m}: \WW^\top \WW = \id_m \right\}.
\end{equation*}
Let $N_{\veps} (d, m)$ be an $\veps$-covering of $O (d, m)$ with respect to the operator norm, namely for any $\WW \in O(d, m)$, one can find some $\WW' \in N_{\veps} (d, m)$ such that $\norm{\WW - \WW'}_{\op} \le \veps$. According to Lemma 4.1 in \cite{hinrichs2017entropy}, we can choose $N_{\veps} (d, m)$ such that
\begin{equation}\label{eq:cover_bd_op}
	\vert N_{\veps} (d, m) \vert \le \left( \frac{C}{\veps} \right)^{d m},
\end{equation}
where $C > 0$ is an absolute constant. Let us define
\begin{equation}
	r = r_{m, \alpha}(\veps) = (1 + \eta) \frac{m}{\alpha} \log \left( \frac{C}{\veps} \right), \ \mathsf{B}_{\rm KL} (r) = \left\{ Q \in \cuP (\R^m): D_{\rm KL} (Q \Vert \normal(\bzero, \id_m)) \le r \right\},
\end{equation}
and
\begin{equation}
	H_{r} (\XX) = \sup_{\WW \in O(d, m)} W_p \left( \hat{P}_{n, \WW}, \mathsf{B}_{\rm KL} (r) \right) = \sup_{\WW \in O(d, m)} \inf_{Q \in \mathsf{B}_{\rm KL} (r)} W_p \left( \hat{P}_{n, \WW}, Q \right),
\end{equation}
where $W_p (P, X) = \inf_{Q \in X} W_p (P, Q)$ for $X \subset \cuP (\R^m)$. Then, we have the following decomposition:
\begin{equation}\label{eq:H_r_upper_bd}
	H_r (\XX) \le \sup_{\WW \in O(d, m)} \left\{ W_p \left( \hat{P}_{n, \WW}, \mathsf{B}_{\rm KL} (r) \right) - W_p \left( \hat{P}_{n, \Pi_0 \WW}, \mathsf{B}_{\rm KL} (r) \right) \right\} + \sup_{\WW \in N_{\veps} (d, m)} W_p \left( \hat{P}_{n, \WW}, \mathsf{B}_{\rm KL} (r) \right),
\end{equation}
where $\Pi_0 \WW$ is the projection of $\WW$ onto $N_{\veps} (d, m)$ with respect to $\norm{\cdot}_{\op}$.
Note that proving Eq.~\eqref{eq:W_p_KL_bound} is equivalent to showing that with high probability,
\begin{equation*}
	H_r (\XX) \le (1 + \eta) \sqrt{\frac{m}{\alpha}} \veps \sqrt{\log \left( \frac{C}{\veps} \right)}.
\end{equation*} 
It then suffices to control the two terms on the right hand side of Eq.~\eqref{eq:H_r_upper_bd} respectively.

We first control the discretization error. For future convenience, denote
\begin{equation*}
	F_r (\XX) = \sup_{\WW \in O(d, m)} \left\{ W_p \left( \hat{P}_{n, \WW}, \mathsf{B}_{\rm KL} (r) \right) - W_p \left( \hat{P}_{n, \Pi_0 \WW}, \mathsf{B}_{\rm KL} (r) \right) \right\}.
\end{equation*}
Then, we know that
\begin{align*}
	\left\vert F_r (\XX) - F_r (\XX') \right\vert \le\, & 2 \sup_{\WW \in O(d, m)} \left\vert W_p \left( \frac{1}{n} \sum_{i=1}^{n} \delta_{\WW^\top \xx_i}, \mathsf{B}_{\rm KL} (r) \right) - W_p \left( \frac{1}{n} \sum_{i=1}^{n} \delta_{\WW^\top \xx_i'}, \mathsf{B}_{\rm KL} (r) \right) \right\vert \\
	\stackrel{(i)}{\le}\, & 2 \sup_{\WW \in O(d, m)} W_p \left( \frac{1}{n} \sum_{i=1}^{n} \delta_{\WW^\top \xx_i}, \frac{1}{n} \sum_{i=1}^{n} \delta_{\WW^\top \xx_i'} \right) \\
	\le\, & 2 \sup_{\WW \in O(d, m)} W_2 \left( \frac{1}{n} \sum_{i=1}^{n} \delta_{\WW^\top \xx_i}, \frac{1}{n} \sum_{i=1}^{n} \delta_{\WW^\top \xx_i'} \right) \\
	\le\, & 2 \sup_{\WW \in O(d, m)} \sqrt{\frac{1}{n} \sum_{i=1}^{n} \norm{\WW^\top (\xx_i - \xx_i')}_2^2} = 2 \sup_{\WW \in O(d, m)} \frac{1}{\sqrt{n}} \norm{(\XX - \XX') \WW}_{\rm F} \\
	\le\, & 2 \sup_{\WW \in O(d, m)} \frac{1}{\sqrt{n}} \norm{\XX - \XX'}_{\rm F} \norm{\WW}_{\op} \le \frac{2}{\sqrt{n}} \norm{\XX - \XX'}_{\rm F},
\end{align*}
where $(i)$ is due to triangle inequality: $\vert W_p (P, X) - W_p (P', X) \vert \le W_p (P, P')$ for $p \in [1, 2]$. The above calculation shows that $F_r (\XX)$ is a $(2 / \sqrt{n})$-Lipschitz function of $\XX$. According to Gaussian concentration inequality, for any $t > 0$ we have
\begin{equation*}
	\P \left( \left\vert F_r (\XX) - \E [F_r (\XX)] \right\vert \ge t \right) \le 2 \exp \left( - \frac{n t^2}{8} \right).
\end{equation*}
Next, we use chaining to show that
\begin{equation}\label{eq:chaining_bound}
	\E [F_r (\XX)] = \E \left[ \sup_{\WW \in O(d, m)} \left\{ W_p \left( \hat{P}_{n, \WW}, \mathsf{B}_{\rm KL} (r) \right) - W_p \left( \hat{P}_{n, \Pi_0 \WW}, \mathsf{B}_{\rm KL} (r) \right) \right\} \right] \le \left(1 + \frac{\eta}{4} \right) \sqrt{\frac{m}{\alpha}} \veps \sqrt{\log \left( \frac{C}{\veps} \right)}.
\end{equation}
For $\GG = (G_1, \cdots, G_n) \in \R^{m \times n}$ and $\GG' = (G_1', \cdots, G_n') \in \R^{m \times n}$, we have
\begin{align*}
	& \left\vert W_p \left( \frac{1}{n} \sum_{i=1}^{n} \delta_{G_i}, \mathsf{B}_{\rm KL} (r) \right) - W_p \left( \frac{1}{n} \sum_{i=1}^{n} \delta_{G_i'}, \mathsf{B}_{\rm KL} (r) \right) \right\vert \le W_p \left( \frac{1}{n} \sum_{i=1}^{n} \delta_{G_i}, \frac{1}{n} \sum_{i=1}^{n} \delta_{G_i'} \right) \\
	\le\, & W_2 \left( \frac{1}{n} \sum_{i=1}^{n} \delta_{G_i}, \frac{1}{n} \sum_{i=1}^{n} \delta_{G_i'} \right) \le \frac{1}{\sqrt{n}} \sqrt{\sum_{i=1}^{n} \norm{G_i - G_i'}_2^2} = \frac{1}{\sqrt{n}} \norm{\GG - \GG'}_{\rm F},
\end{align*}
which implies that $W_p ( (1/n) \sum_{i=1}^{n} \delta_{G_i}, \mathsf{B}_{\rm KL} (r) )$ is an $(1 / \sqrt{n})$-Lipschitz function of $\GG$. According to Lemma~\ref{lem:MGF_diff_Lip}, we obtain that for $\WW, \WW' \in O(d, m)$ and $\forall t > 0$,
\begin{equation*}
	\E \left[ \exp \left( t \left( W_p \left( \hat{P}_{n, \WW}, \mathsf{B}_{\rm KL} (r) \right) - W_p \left( \hat{P}_{n, \WW'}, \mathsf{B}_{\rm KL} (r) \right) \right) \right) \right] \le \exp \left( \frac{t^2}{2 n} \norm{\WW - \WW'}_{\op}^2 \right),
\end{equation*}
thus leading to
\begin{equation}
	\norm{W_p \left( \hat{P}_{n, \WW}, \mathsf{B}_{\rm KL} (r) \right) - W_p \left( \hat{P}_{n, \WW'}, \mathsf{B}_{\rm KL} (r) \right)}_{\Psi_2} \le \frac{1}{\sqrt{n}} \norm{\WW - \WW'}_{\op},
\end{equation}
where $\norm{\cdot}_{\Psi_2}$ denotes the sub-Gaussian norm. Applying Lemma 13.1 of \cite{boucheron2013concentration}, we deduce that
\begin{align*}
	\E [F_r (\XX)] =\, & \E \left[ \sup_{\WW \in O(d, m)} \left\{ W_p \left( \hat{P}_{n, \WW}, \mathsf{B}_{\rm KL} (r) \right) - W_p \left( \hat{P}_{n, \Pi_0 \WW}, \mathsf{B}_{\rm KL} (r) \right) \right\} \right] \\
	\le\, & \frac{C}{\sqrt{n}} \int_{0}^{\veps} \sqrt{\log N(O(d, m), \norm{\cdot}_{\op}, u)} \d u \\
	\stackrel{(i)}{\le}\, &  C \sqrt{\frac{md}{n}} \int_{0}^{\veps} \sqrt{\log \left( \frac{C}{u} \right)} \d u \le C \sqrt{\frac{md}{n}} \veps \sqrt{\log \left( \frac{C}{\veps} \right)} \\
	\stackrel{(ii)}{\le}\, & \left( 1 + \frac{\eta}{4} \right) \sqrt{\frac{m}{\alpha}} \veps \sqrt{\log \left( \frac{C}{\veps} \right)},
\end{align*}
where $(i)$ follows from Eq.~\eqref{eq:cover_bd_op}, and $(ii)$ follows from recasting $C \veps$ as $\veps$ since $C$ is an absolute constant. This proves Eq.~\eqref{eq:chaining_bound}, combining which with the Gaussian concentration argument gives that
\begin{equation*}
	\sup_{\WW \in O(d, m)} \left\{ W_p \left( \hat{P}_{n, \WW}, \mathsf{B}_{\rm KL} (r) \right) - W_p \left( \hat{P}_{n, \Pi_0 \WW}, \mathsf{B}_{\rm KL} (r) \right) \right\} \le \left( 1 + \frac{\eta}{2} \right) \sqrt{\frac{m}{\alpha}} \veps \sqrt{\log \left( \frac{C}{\veps} \right)}
\end{equation*}
with probability at least $1 - C_0 \exp (- C_1 n)$.

In order to complete the proof of Eq.~\eqref{eq:W_p_KL_bound}, it suffices to show that
\begin{equation*}
	\P \left( \sup_{\WW \in N_{\veps} (d, m)} W_p \left( \hat{P}_{n, \WW}, \mathsf{B}_{\rm KL} (r) \right) > \frac{\eta}{2} \sqrt{\frac{m}{\alpha}} \veps \sqrt{\log \left( \frac{C}{\veps} \right)} \right) \le C_0 \exp(- C_1 n),
\end{equation*}
Equivalently, we prove that with probability no less than $1 - C_0 \exp(- C_1 n)$, for any $\WW \in N_{\veps} (d, m)$, there exists $Q \in \cuP (\R^m)$ such that 
\begin{equation}\label{eq:covering_div_bd}
	W_p \left( \hat{P}_{n, \WW}, Q \right) \le \frac{\eta}{2} \sqrt{\frac{m}{\alpha}} \veps \sqrt{\log \left( \frac{C}{\veps} \right)}, \ D_{\rm KL} \left( Q \Vert \sN (\bzero, \id_m) \right) \le r = (1 + \eta) \frac{m}{\alpha} \log \left( \frac{C}{\veps} \right).
\end{equation}
Let $\mu_2 (P)$ denote the second moment of a probability measure $P$. Denoting the set of distributions $P$ that satisfy Eq.~\eqref{eq:covering_div_bd} (with $\hat{P}_{n, \WW}$ replaced by $P$) as $S_{m, p}$ and defining $T_{m, p} = \{ P \notin S_{m, p}: \mu_2 (P) < \infty \}$, we get that for any fixed $\WW \in N_{\veps} (d, m)$, 
\begin{equation*}
	\P \left( \hat{P}_{n, \WW} \notin S_{m, p} \right) = \P \left( \hat{P}_{n, \WW} \in T_{m, p} \right) = \P \left( \frac{1}{n} \sum_{i = 1}^{n} \delta_{G_i} \in T_{m, p} \right),
\end{equation*}
where $G_i \sim_{\iid} \sN (\bzero, \id_m)$, and the last equality is due to the fact that $\xx_i^\top \WW \sim_{\iid} \sN (\bzero, \id_m)$. 

According to a variant of Sanov's theorem \cite[Thm. 1.1]{Wang2010}, we deduce that
\begin{equation*}
	\limsup_{n \to \infty} \frac{1}{n} \log \P \left( \frac{1}{n} \sum_{i = 1}^{n} \delta_{G_i} \in T_{m, p} \right) \le - \inf_{Q \in \mathrm{cl}_{\tau_p} (T_{m, p})} D_{\rm KL} \left( Q \Vert \sN (\bzero, \id_m) \right),
\end{equation*}
where $\mathrm{cl}_{\tau_p} (T_{m, p})$ denotes the closure of $T_{m, p}$ with respect to the $\tau_p$-topology that is induced by the $W_p$ distance in $\cuP (\R^m)$. It is straightforward to verify the conditions of Theorem 1.1 from \cite{Wang2010} for standard Gaussian distribution:
\begin{equation*}
	\Lambda (\lambda) = \log \int_{\R^m} \frac{1}{(2 \pi)^{m / 2}} \exp \left( \lambda \norm{\xx}_2^p - \frac{\norm{\xx}_2^2}{2} \right) \d \xx < \infty
\end{equation*}
for any $\lambda > 0$ and $p < 2$. We next prove by contradiction that
\begin{equation*}
	\inf_{Q \in \mathrm{cl}_{\tau_p} (T_{m, p})} D_{\rm KL} \left( Q \Vert \sN (\bzero, \id_m) \right) \ge (1 + \eta) \frac{m}{\alpha} \log \left( \frac{C}{\veps} \right).
\end{equation*}
Assume this is not true, then we can find a sequence of probability measures $\{ Q_k \}$ such that
\begin{equation*}
	Q_k \in T_{m, p}, \ W_p \left( Q_k, Q \right) \to 0, \ \text{and} \ D_{\rm KL} \left( Q \Vert \sN (\bzero, \id_m) \right) < (1 + \eta) \frac{m}{\alpha} \log \left( \frac{C}{\veps} \right).
\end{equation*}
Hence, for sufficiently large $k$ we have
\begin{equation*}
	W_p \left( Q_k, Q \right) < \frac{\eta}{2} \sqrt{\frac{m}{\alpha}} \veps \sqrt{\log \left( \frac{C}{\veps} \right)} \implies Q_k \in S_{m, p},
\end{equation*}
which contradicts the assumption that $Q_k \in T_{m, p}$. As a consequence, we obtain that
\begin{equation*}
	\limsup_{n \to \infty} \frac{1}{n} \log \P \left( \frac{1}{n} \sum_{i = 1}^{n} \delta_{G_i} \in T_{m, p} \right) \le - (1 + \eta) \frac{m}{\alpha} \log \left( \frac{C}{\veps} \right).
\end{equation*}
Hence, for large enough $n$, applying a union bound gives that
\begin{align*}
	& \P \left( \exists \WW \in N_{\veps} (d, m), \ \text{s.t.} \ \hat{P}_{n, \WW} \notin S_{m, p} \right) \le \vert N_{\veps} (d, m) \vert \sup_{\WW \in N_{\veps} (d, m)} \P \left( \hat{P}_{n, \WW} \notin S_{m, p} \right) \\
	\le & \left( \frac{C}{\veps} \right)^{d m} \exp \left( - n \cdot \left( 1 + \frac{\eta}{2} \right) \frac{m}{\alpha} \log \left( \frac{C}{\veps} \right) \right) = \exp \left( - m \log \left( \frac{C}{\veps} \right) \left( \left( 1 + \frac{\eta}{2} \right) \frac{n}{\alpha} - d \right) \right),
\end{align*}
which converges to $0$ exponentially fast as $n \to \infty$, since $n / d \to \alpha$. This proves Eq.~\eqref{eq:covering_div_bd}, and as a consequence Eq.~\eqref{eq:W_p_KL_bound} follows naturally. Using a similar argument as in the proof of Theorem~\ref{thm:W2_outer_bd}, we deduce (by Borel-Cantelli Lemma) that with probability one, for all sufficiently large $n$ and $\WW \in O (d, m)$, there exists $Q \in \cuP (\R^m)$ satisfying
\begin{equation*}
	W_p \left( \hat{P}_{n, \WW}, Q \right) \le (1 + \eta) \sqrt{\frac{m}{\alpha}} \veps \sqrt{\log \left( \frac{C}{\veps} \right)}, \ D_{\rm KL} \left( Q \Vert \sN (\bzero, \id_m) \right) \le (1 + \eta) \frac{m}{\alpha} \log \left( \frac{C}{\veps} \right).
\end{equation*}

\noindent {\bf Part (ii).} Now we continue the proof of Theorem~\ref{thm:div_outer_bd}. Let $P$ be an $(\alpha, m)$-feasible distribution, then there exists a sequence of random probability measures $\{ \hat{P}_{n, \WW} \}_{n \ge 1}$ such that $\hat{P}_{n, \WW} \stackrel{w}{\Rightarrow} P$ in probability. Similarly as in the proof of Theorem~\ref{thm:outer_bound_1_dim}, we can assume without loss of generality that $\hat{P}_{n, \WW} \stackrel{w}{\Rightarrow} P$ almost surely. For notational convenience we rewrite $\hat{P}_{n, \WW}$ as $P_n$. According to the previous claim from part (i), we know that for any $p < 2$ and sufficiently large $n$, there exists $Q_n \in \cuP (\R^m)$ such that
\begin{equation*}
	W_p \left( P_n, Q_n \right) \le \left( 1 + \frac{\eta}{2} \right) C(\alpha, m) \veps, \ D_{\rm KL} \left( Q_n \Vert \sN (\bzero, \id_m) \right) \le (1 + \eta) \frac{m}{\alpha} \log \left( \frac{C}{\veps} \right),
\end{equation*}
where we denote $C(\alpha, m) = \sqrt{(m / \alpha) \log (C / \veps)}$. Based on the transport-entropy inequality \cite[Thm. 1.1]{talagrand1996transportation}, we know that the sequence $\{ Q_n \}_{n \ge 1}$ has uniformly bounded second moments. Hence, $\{ Q_n \}$ is sequentially compact with respect to the $\tau_p$-topology (as $p < 2$), thus leading to the existence of a converging subsequence (still denoted as $\{ Q_n \}$) in $W_p$ metric. Consequently, there exists a $k \in \mathbb{N}$ satisfying that
\begin{equation*}
	W_p \left( Q_k, Q_n \right) \le \frac{\eta}{2} C(\alpha, m) \veps, \ \forall n \ge k.
\end{equation*}
Similar to the proof of Theorem~\ref{thm:outer_bound_1_dim}, using the lower semicontinuity of $W_p$ distance yields that
\begin{align*}
	W_p \left( P, Q_k \right) \le & \liminf_{n \to \infty} W_p \left( P_n, Q_k \right) \le \liminf_{n \to \infty} W_p \left( P_n, Q_n \right) + \limsup_{n \to \infty} W_p \left( Q_n, Q_k \right) \\
	\le & \left( 1 + \frac{\eta}{2} \right) C(\alpha, m) \veps + \frac{\eta}{2} C(\alpha, m) \veps = \left( 1 + \eta \right) C(\alpha, m) \veps.
\end{align*}
Since we already know that $D_{\rm KL} ( Q_n \Vert \sN (\bzero, \id_m) ) \le (1 + \eta) ( m / \alpha ) \log ( C / \veps )$, recasting $Q_k$ as $Q_p$ yields the following result: For any $p \in [1, 2)$, there exists a probability measure $Q_p \in \cuP (\R^m)$ such that
\begin{equation*}
	W_p \left( P, Q_p \right) \le (1 + \eta) \sqrt{\frac{m}{\alpha}} \veps \sqrt{\log \left( \frac{C}{\veps} \right)}, \ D_{\rm KL} \left( Q_p \Vert \sN (\bzero, \id_m) \right) \le (1 + \eta) \frac{m}{\alpha} \log \left( \frac{C}{\veps} \right).
\end{equation*}

Now we are in position to complete the proof of Theorem~\ref{thm:div_outer_bd}. Take a sequence $\{ p_n \} \subset [1, 2)$ such that $p_n \to 2$ and let $\{ Q_{p_n} \}$ be defined as above. Using again the transport-entropy inequality, we know that for all $k \in \mathbb{N}$, any subsequence of $\{ Q_{p_n} \}$ has a further subsequence which converges in $W_{p_k}$ metric (since $\{ Q_{p_n} \}$ is tight in $\tau_{p_k}$-topology). Proceeding with the diagonal argument, we can find a subsequence (still denoted as $\{ Q_{p_n} \}$) such that
\begin{equation*}
	Q_{p_n} \stackrel{W_{p_k}}{\longrightarrow} Q, \ \forall k \in \mathbb{N} \ \text{for some} \ Q \in \cuP (\R^m).
\end{equation*}
Note that $Q$ satisfies $D_{\rm KL} \left( Q \Vert \sN (\bzero, \id_m) \right) \le (1 + \eta) (m / \alpha) \log \left( C / \veps \right)$ as well, since any closed KL-divergence ball is compact (thus closed) in $\tau_p$-topology if $p < 2$, see, e.g., Theorem 1.8 (c) from \cite{eichelsbacher1996large}. Moreover, we know that
\begin{equation*}
	W_{p_k} \left( P, Q \right) \le (1 + \eta) \sqrt{\frac{m}{\alpha}} \veps \sqrt{\log \left( \frac{C}{\veps} \right)}, \ \forall k \in \mathbb{N},
\end{equation*}
which further implies that
\begin{equation*}
	W_{p} \left( P, Q \right) \le (1 + \eta) \sqrt{\frac{m}{\alpha}} \veps \sqrt{\log \left( \frac{C}{\veps} \right)}, \ \forall p \in [1, 2).
\end{equation*}
Applying Lemma~\ref{lem:Wp_continuity} immediately yields $W_2 \left( P, Q \right) \le (1 + \eta) \veps \sqrt{m / \alpha} \sqrt{\log \left( C / \veps \right)}$. This proves that for any $\eta > 0$ and $\veps \in (0, \veps_0)$,
\begin{equation*}
	\cuF_{m,\alpha} \subseteq \cuS_m \left( (1 + \eta) \sqrt{\frac{m}{\alpha}} \veps \sqrt{\log \left( \frac{C}{\veps} \right)}, (1 + \eta) \frac{m}{\alpha} \log \left( \frac{C}{\veps} \right) \right).
\end{equation*}
Sending $\eta \to 0^{+}$ and using the compactness argument again gives the conclusion of Theorem~\ref{thm:div_outer_bd}.

\subsection{Proof of Theorem~\ref{thm:m_W2_outer_bd}}
Fix $P \in \cuF_{m,\alpha}$, and let $Q$ be as described in the statement of Theorem~\ref{thm:div_outer_bd}. According to the Transport-Entropy inequality \cite[Thm. 1.1]{talagrand1996transportation}, we obtain that
\begin{equation*}
	W_2 \left( Q, \sN (\bzero, \id_m) \right) \le \sqrt{2 D_{\rm KL} \left( Q \Vert \sN (\bzero, \id_m) \right)} \le \sqrt{\frac{2 m}{\alpha} \log \left( \frac{C}{\veps} \right)},
\end{equation*}
which leads to
\begin{equation*}
	W_2 \left( P, \sN (\bzero, \id_m) \right) \le W_2 \left( P, Q \right) + W_2 \left( Q, \sN (\bzero, \id_m) \right) \le \sqrt{\frac{m}{\alpha}} \veps \sqrt{\log \left( \frac{C}{\veps} \right)} + \sqrt{ \frac{2 m}{\alpha} \log \left( \frac{C}{\veps} \right)}.
\end{equation*}
Taking $\veps = \veps_0 / 2$ and noting that $\veps_0$ is an absolute constant 
gives
\begin{equation*}
	W_2 \left( P, \sN (\bzero, \id_m) \right) = O \left( \sqrt{\frac{m}{\alpha}} \right).
\end{equation*}
This completes the proof of the first part of Theorem~\ref{thm:m_W2_outer_bd}. As for the lower bound, it suffices to take $\WW = (\vv_i(\XX))_{i \le m}$, where $\vv_i (\XX)$ is the $i$-th right singular vector of $\XX$. Similar to the proof of Theorem~\ref{thm:W2_outer_bd}, we can show that
\begin{equation*}
	\frac{1}{n} \sum_{i = 1}^{n} \delta_{\WW^\top \xx_i} \stackrel{w}{\Rightarrow} P = \left( 1 + \frac{1}{\sqrt{\alpha}} \right) \sN ( \bzero, \id_m ) \ \text{in probability},
\end{equation*}
and it is straightforward to verify $W_2 (P, \sN(\bzero, \id_m)) = \sqrt{m / \alpha}$.

\subsection{Proof of Theorem~\ref{thm:infeasible_info_dim}}
Since $\underline{d} (P) < m (1 - 1 / \alpha)$, we know that $C_0 (\alpha, m) := m / \alpha < m - \underline{d} (P)$ (abbreviated as $C_0$) and $C_1 (\alpha, m) := \sqrt{m / \alpha} < \infty$ (abbreviated as $C_1$). Then Theorem~\ref{thm:div_outer_bd} tells us that for all $\veps \in (0, \veps_0)$, there exists $Q \in \cuP (\R^m)$ such that (we may recast $\veps \sqrt{\log (C / \veps)}$ as $\veps$)
\begin{equation*}
	W_2 \left( P, Q \right) \le C_1 \veps, \ D_{\rm KL} \left( Q \Vert \sN (\bzero, \id_m) \right) \le C_0 \log \left( \frac{C}{\veps} \right).
\end{equation*}
Let $P_{\veps}$ denote the $\veps$-discretization of $P$ as in Definition~\ref{def:info_dim}, namely $P_{\veps} = \Law(\langle X \rangle_{\veps})$ if $P = \Law(X)$, then we know that
\begin{equation*}
	W_2 \left( P_{\veps}, Q \right) \le \left( C_1 + \sqrt{m} \right) \veps, \ D_{\rm KL} \left( Q \Vert \sN (\bzero, \id_m) \right) \le C_0 \log \left( \frac{C}{\veps} \right).
\end{equation*}
The first inequality is due to $W_2 (P_{\veps}, P) \le \sqrt{m} \veps$ and the triangle inequality.
Moreover, denoting $P_{\veps} = \sum_{k = 1}^{\infty} p_k \delta_{\cc_k}$, we know that
\begin{equation}\label{eq:info_dim_ineq}
	\underline{d} (P) = \liminf_{\veps \to 0} \frac{H(P_{\veps})}{\log (1 / \veps)} = \liminf_{\veps \to 0} \frac{- \sum_{k = 1}^{\infty} p_k \log p_k }{\log (1 / \veps)} < m - C_0.
\end{equation}
Since $D_{\rm KL} \left( Q \Vert \sN (\bzero, \id_m) \right) < \infty$, we know that $Q$ is absolutely continuous (with respect to the Lebesgue measure) with density $q$. Furthermore, when $W_2 (P_\veps, Q)$ is minimized, the optimal coupling must be such that $U = T(V)$ for some measurable function $T$, where $U \sim P_\veps$, $V \sim Q$. Recalling $P_\veps = \sum_{k = 1}^{\infty} p_k \delta_{\cc_k}$ and defining $S_k = T^{- 1} (\cc_k)$, we deduce that
\begin{equation*}
	W_2 \left( P_\veps, Q \right)^2 = \int_{\R^m} \norm{\xx - T(\xx)}_2^2 q(\xx) \d \xx = \sum_{k = 1}^{\infty} \int_{S_k} \norm{\xx - \cc_k}_2^2 q(\xx) \d \xx,
\end{equation*}
where $\int_{S_k} q (\xx) \d \xx = \P (U = \cc_k) = p_k$. Now let us estimate $D_{\rm KL} \left( Q \Vert \sN (\bzero, \id_m) \right)$, we have
\begin{align*}
	& D_{\rm KL} \left( Q \Vert \sN (\bzero, \id_m) \right) = \int_{\R^m} q (\xx) \log \left( \frac{q(\xx)}{\phi (\xx)} \right) \d \xx \\
	= & \frac{m}{2} \log (2 \pi) + \frac{1}{2} \int_{\R^m} \norm{\xx}_2^2 q(\xx) \d \xx + \int_{\R^m} q(\xx) \log q(\xx) \d \xx \\
	= & \frac{m}{2} \log (2 \pi) + \frac{1}{2} \mu_2 \left( Q \right) + \sum_{k = 1}^{\infty} \int_{S_k} q(\xx) \log q(\xx) \d \xx,
\end{align*}
where we recall that $\mu_2 (Q)$ is the second moment of $Q$. Now we define for all $k \ge 1$, $q_k (\uu) = \veps^m q(\cc_k + \veps \uu)$ and $S_k (\veps) = (S_k - \cc_k) / \veps$, then it follows that
\begin{equation*}
	(C_1 + \sqrt{m}) ^2 \veps^2 \ge W_2 \left( P_\veps, Q \right)^2 = \sum_{k = 1}^{\infty} \int_{S_k (\veps)} \veps^2 \norm{\uu}_2^2 \cdot \veps^m q (\cc_k + \veps \uu) \d \uu = \veps^2 \sum_{k = 1}^{\infty} \int_{S_k (\veps)} \norm{\uu}_2^2 q_k (\uu) \d \uu,
\end{equation*}
which further implies that
\begin{equation*}
	\sum_{k = 1}^{\infty} \int_{S_k (\veps)} \norm{\uu}_2^2 q_k (\uu) \d \uu \le (C_1 + \sqrt{m})^2.
\end{equation*}
We can show in a similar way that $\int_{S_k (\veps)} q_k (\uu) \d \uu = p_k$. Therefore, $r_k (\uu) = q_k (\uu) / p_k$ is a density on $S_k (\veps)$. For every $k \in \mathbb{N}$, we then have
\begin{align*}
	\int_{S_k} q(\xx) \log q(\xx) \d \xx = & \int_{S_k (\veps)} q_k (\uu) \log \left( \frac{q_k (\uu)}{\veps^m} \right) \d \uu \\
	= & m p_k \log \left( \frac{1}{\veps} \right) + \int_{S_k (\veps)} q_k (\uu) \log q_k (\uu) \d \uu \\
	= & m p_k \log \left( \frac{1}{\veps} \right) + p_k \log p_k + p_k \int_{S_k (\veps)} r_k (\uu) \log r_k (\uu) \d \uu \\
	\stackrel{(i)}{\ge} & m p_k \log \left( \frac{1}{\veps} \right) + p_k \log p_k + p_k \int_{S_k (\veps)} r_k (\uu) \log \phi (\uu) \d \uu \\
	= & m p_k \log \left( \frac{1}{\veps} \right) + p_k \log p_k - \frac{m p_k}{2} \log(2 \pi) - \frac{1}{2} \int_{S_k (\veps)} \norm{\uu}_2^2 q_k (\uu) \d \uu,
\end{align*}
where $(i)$ follows from the non-negativity of the KL divergence. Summing up the above inequality for $k \ge 1$, we finally deduce that
\begin{align*}
	D_{\rm KL} \left( Q \Vert \sN (\bzero, \id_m) \right) = & \frac{m}{2} \log (2 \pi) + \frac{1}{2} \mu_2 \left( Q \right) + \sum_{k = 1}^{\infty} \int_{S_k} q(\xx) \log q(\xx) \d \xx \\
	\ge & \frac{1}{2} \mu_2 \left( Q \right) + m \log \left( \frac{1}{\veps} \right) + \sum_{k = 1}^{\infty} p_k \log p_k - \frac{1}{2} \sum_{k = 1}^{\infty} \int_{S_k (\veps)} \norm{\uu}_2^2 q_k (\uu) \d \uu \\
	\ge & m \log \left( \frac{1}{\veps} \right) + \sum_{k = 1}^{\infty} p_k \log p_k - \frac{(C_1 + \sqrt{m})^2}{2},
\end{align*}
which in turn implies that
\begin{align*}
	& C_0 \log \left( \frac{C}{\veps} \right) \ge m \log \left( \frac{1}{\veps} \right) + \sum_{k = 1}^{\infty} p_k \log p_k - \frac{(C_1 + \sqrt{m})^2}{2} \\
	\implies & H(P_{\veps}) \ge (m - C_0) \log \left( \frac{1}{\veps} \right) - C_0 \log C - \frac{(C_1 + \sqrt{m})^2}{2}.
\end{align*}
Since we know that $\underline{d} (P) < m - C_0$, sending $\veps \to 0$ in the above equation results a contradiction with Eq.~\eqref{eq:info_dim_ineq}. Hence, $P$ is not $(\alpha, m)$-feasible, as desired. To prove the ``as a consequence" part we only need to notice that if $P$ supported on an $s$-dimensional smooth manifold in $\R^m$, then $\underline{d} (P) \le s$, see, for example, the discussion following Theorem 4 in \cite{renyi1959dimension}.

\section{Proofs for Section~\ref{sec:2nd_moment}: Unsupervised learning inner bounds}
To avoid heavy notation, we denote $p_i = \Phi_A (\aa_i)$ and $r_i = P_A (\aa_i)$ for each $i \in [M]$, and use $\pp = (p_1, \cdots, p_M)$ and $\rr = (r_1, \cdots, r_M)$ to represent the discrete distribution $\Phi_A$ and $P_A$, respectively. Therefore, $\oalpha_{\rm lb} (P_A) = \oalpha_{\rm lb} (\rr)$. Moreover, we write $\pp (\QQ) = \Phi_{A, \QQ}^{(2)}$ and $\rr (\QQ) = R_{A, \QQ}^{(2)}$ for $\Phi_{A, \QQ}^{(2)}$ and $R_{A, \QQ}^{(2)}$ introduced in Definition~\ref{def:2nd_thres}. In terms of components, this means $p_{ij} (\QQ) = \Phi_{A, \QQ}^{(2)} (\aa_i, \aa_j)$ and $r_{ij} (\QQ) = R_{A, \QQ}^{(2)} (\aa_i, \aa_j)$ for all $i, j \in [M]$.

\subsection{Proofs of Theorem~\ref{thm:discrete_lower_bd} and Lemma~\ref{lem:var_rep_KL}}

\begin{proof}[\bf Proof of Theorem~\ref{thm:discrete_lower_bd}]
	Let $\mu_{d, m}$ denote the uniform measure on the Stiefel manifold $O(d, m)$, and $\kk = (k_1, \cdots, k_M)$ is such that $k_i \in \mathbb{N}$, $\forall i \in [M]$ and $\sum_{i = 1}^{M} k_i = n$. From now on we denote $q_i (n) = k_i / n$ and write $q_i = q_i (n)$ for simplicity. Note that we can choose the vector $\kk$ such that for all $i \in [M]$, $\lim_{n \to \infty} q_i (n) = r_i$, and that $\vert q_i(n) - r_i \vert = O(1/n)$. We further denote $\qq = \Law (X)$ where $\P (X = \aa_i) = k_i / n = q_i$ for $1 \le i \le M$, and define the following random variable (note that $\qq$ is supported on $A$):
\begin{equation}\label{eq:def_Z}
	Z = \int_{O(d, m)} \bone \left\{ \langle \hat{P}_{n, \WW} \rangle_A = \qq \right\} \mu_{d, m} (\d \WW) = \int_{O(d, m)} \bone \left\{ \frac{1}{n} \sum_{i = 1}^{n} \delta_{\langle \WW^\top \xx_i \rangle_A} = \qq \right\} \mu_{d, m} (\d \WW).
\end{equation}
Then, we know that $Z \ge 0$, and that $Z > 0 \Longleftrightarrow \exists \WW \in O(d, m)$ such that $\langle \hat{P}_{n, \WW} \rangle_A = \qq$. The lemma below establishes that this holds true with probability bounded away from zero under some conditions regarding $\alpha$ and the $q_i$'s, which will be verified later.
\begin{lem}\label{lem:2nd_moment}
Define the following probability distribution on $[M] \times [M]$:
\begin{equation}\label{eq:def_info_proj_q}
	\begin{split}
		\qq (\QQ) = & (q_{ij} (\QQ))_{i, j \in [M]} = \argmin_{\qq = (q_{ij})_{i, j \in [M]}} D_{\sKL} \left( \qq \Vert \pp(\QQ) \right) \\
		= & \argmin_{\qq = (q_{ij})_{i, j \in [M]}} \sum_{i, j = 1}^{M} q_{ij} \log \frac{q_{ij}}{p_{ij} (\QQ)}, \\
		{\rm subject \ to} \ & \sum_{j=1}^{M} q_{ij} = q_i, \ \sum_{i=1}^{M} q_{ij} = q_j, \ \forall i, j \in [M].
	\end{split}
\end{equation}
In fact, $\qq (\QQ)$ is the information projection of the distribution $\pp (\QQ)$ onto the set of distributions on $[M] \times [M]$ whose both margins are $\qq = (q_1, \cdots, q_M)$. Assume that $\inf_{n \in \mathbb{N}, i \in [M]} q_i (n) > 0$, and that there exist $\alpha' > \alpha$ and $\eta > 0$ such that for all $n$ large enough,
\begin{equation}\label{eq:2nd_cond_complex}
\begin{split}
	\min_{\lambda_{\max} (\QQ^\top \QQ) \le 1/2} \left\{ D_{\sKL} \left( \qq(\QQ) \Vert \pp(\QQ) \right) + \frac{1}{\alpha'} I(\QQ) \right\} \ge & D_{\sKL} \left( \qq(\bzero) \Vert \pp(\bzero) \right) + \frac{1}{\alpha'} I(\bzero) - O \left( \frac{1}{n^2} \right), \\
	\min_{\lambda_{\max} (\QQ^\top \QQ) \in (1/2, 1)} \left\{ D_{\sKL} \left( \qq(\QQ) \Vert \pp(\QQ) \right) + \frac{1}{\alpha} I(\QQ) \right\} \ge & D_{\sKL} \left( \qq(\bzero) \Vert \pp(\bzero) \right) + \frac{1}{\alpha} I(\bzero) + \eta.
\end{split}
\end{equation}
Then, we have $\liminf_{n \to \infty} \P (Z > 0) > 0$.
\end{lem}
The proof of Lemma~\ref{lem:2nd_moment} will be deferred to Section~\ref{sec:proof_lem_2nd}. Next it suffices to verify the conditions of Lemma~\ref{lem:2nd_moment}. Note that by our assumption, $r_i > 0$ and $\vert q_i (n) - r_i \vert = O(1/n)$, which implies that $\inf_{n \in \mathbb{N}, i \in [M]} q_i (n) > 0$. To show Eq.~\eqref{eq:2nd_cond_complex}, let us choose $\alpha' \in (\alpha, \oalpha_{\rm lb} (\rr))$, then by Definition~\ref{def:2nd_thres}, we know that $D_{\sKL} \left( \rr(\QQ) \Vert \pp(\QQ) \right) + I(\QQ) / \alpha'$ achieves its unique minimum at $\QQ = \bzero$, thus leading to
	\begin{equation}\label{eq:optimal_r}
		\min_{\lambda_{\max} (\QQ^\top \QQ) \le 1/2} \left\{ D_{\sKL} \left( \rr(\QQ) \Vert \pp(\QQ) \right) + \frac{1}{\alpha'} I(\QQ) \right\} = D_{\sKL} \left( \rr(\bzero) \Vert \pp(\bzero) \right) + \frac{1}{\alpha'} I(\bzero).
	\end{equation}
    By our assumption, $\rr - \qq = O(1/n)$ and Eq.~\eqref{eq:q_bdd_away} holds (see also Lemma~\ref{lem:info_proj_cond} (c)), whose correctness is established only on the definition of $\qq(\QQ)$. Using Lemma~\ref{lem:cont_info_proj}, we then conclude that
    \begin{equation}\label{eq:sup_diff_KL}
    	\sup_{\lambda_{\max} (\QQ^\top \QQ) \le 1/2} \left\vert D_{\sKL} \left( \rr(\QQ) \Vert \pp(\QQ) \right) - D_{\sKL} \left( \qq(\QQ) \Vert \pp(\QQ) \right) \right\vert = O \left( \frac{1}{n} \right).
    \end{equation}
    Moreover, based on the proof of Lemma~\ref{lem:cont_info_proj} and the envelope theorem, we further deduce that
    \begin{equation}
    \begin{split}
    	& \min_{\lambda_{\max} (\QQ^\top \QQ) \le 1/2} \left\{ D_{\sKL} \left( \qq(\QQ) \Vert \pp(\QQ) \right) + \frac{1}{\alpha'} I(\QQ) \right\} \\
    	\stackrel{(i)}{=} & \min_{\lambda_{\max} (\QQ^\top \QQ) \le 1/2} \left\{ D_{\sKL} \left( \rr(\QQ) \Vert \pp(\QQ) \right) + \frac{1}{\alpha'} I(\QQ) \right\} + \sum_{i=1}^{M} (q_i - r_i) \left( \log \frac{r_{ii} (\bzero)}{p_{ii} (\bzero)} + 1 \right) + O \left( \frac{1}{n^2} \right) \\
    	\stackrel{(ii)}{=} & \min_{\lambda_{\max} (\QQ^\top \QQ) \le 1/2} \left\{ D_{\sKL} \left( \rr(\QQ) \Vert \pp(\QQ) \right) + \frac{1}{\alpha'} I(\QQ) \right\} + 2 \sum_{i=1}^{M} (q_i - r_i) \log \frac{r_i}{p_i} + O \left( \frac{1}{n^2} \right),
    \end{split}
    \end{equation}
    where in $(i)$ we use the fact that $D_{\sKL} \left( \rr(\QQ) \Vert \pp(\QQ) \right) + I(\QQ) / \alpha'$ is uniquely minimized at $\QQ = \bzero$, and $(ii)$ follows from Lemma~\ref{lem:info_proj_cond} (a): $r_{ij} (\bzero) = r_i r_j$. On the other hand, using Taylor expansion, we get that
    \begin{align*}
    	D_{\sKL} \left( \qq(\bzero) \Vert \pp(\bzero) \right) - D_{\sKL} \left( \rr(\bzero) \Vert \pp(\bzero) \right) = 2 \sum_{i = 1}^{M} \left( q_i \log \frac{q_i}{p_i} - r_i \log \frac{r_i}{p_i} \right) = 2 \sum_{i=1}^{M} (q_i - r_i) \log \frac{r_i}{p_i} + O \left( \frac{1}{n^2} \right).
    \end{align*}
    Combining these estimates together with Eq.~\eqref{eq:optimal_r} immediately implies the first part of Eq.~\eqref{eq:2nd_cond_complex}.

     To prove the second part, note that $\alpha < \oalpha_{\rm lb} (\rr)$. Therefore, we can choose $\eta > 0$ such that
    \begin{equation*}
    	\min_{\lambda_{\max} (\QQ^\top \QQ) \in (1/2, 1)} \left\{ D_{\sKL} \left( \rr(\QQ) \Vert \pp(\QQ) \right) + \frac{1}{\alpha} I(\QQ) \right\} \ge D_{\sKL} \left( \rr(\bzero) \Vert \pp(\bzero) \right) + \frac{1}{\alpha} I(\bzero) + 2 \eta,
    \end{equation*}
    and $\theta \in (0, 1)$ satisfying 
    \begin{equation*}
    	- \frac{1}{2 \alpha} \log (1 - \theta) \ge D_{\sKL} \left( \rr(\bzero) \Vert \pp(\bzero) \right) + 2 \eta.
    \end{equation*}
    Then, we have
    \begin{align*}
    	& \min_{\lambda_{\max} (\QQ^\top \QQ) \in (\theta, 1)} \left\{ D_{\sKL} \left( \qq(\QQ) \Vert \pp(\QQ) \right) + \frac{1}{\alpha} I(\QQ) \right\} \ge - \frac{1}{2 \alpha} \log (1 - \theta) \\
    	\ge & D_{\sKL} \left( \rr(\bzero) \Vert \pp(\bzero) \right) + 2 \eta \ge D_{\sKL} \left( \qq(\bzero) \Vert \pp(\bzero) \right) + \frac{1}{\alpha} I(\bzero) + \eta
    \end{align*}
    for sufficiently large $n$. Note that Eq.~\eqref{eq:sup_diff_KL} will still be true if we replace $1/2$ by $\theta$, meaning that
    \begin{align*}
    	& \min_{\lambda_{\max} (\QQ^\top \QQ) \in (1/2, \theta]} \left\{ D_{\sKL} \left( \qq(\QQ) \Vert \pp(\QQ) \right) + \frac{1}{\alpha} I(\QQ) \right\} \\
    	\ge & \min_{\lambda_{\max} (\QQ^\top \QQ) \in (1/2, \theta]} \left\{ D_{\sKL} \left( \rr(\QQ) \Vert \pp(\QQ) \right) + \frac{1}{\alpha} I(\QQ) \right\} - O \left( \frac{1}{n} \right) \\
    	\ge & \min_{\lambda_{\max} (\QQ^\top \QQ) \in (1/2, 1)} \left\{ D_{\sKL} \left( \rr(\QQ) \Vert \pp(\QQ) \right) + \frac{1}{\alpha} I(\QQ) \right\} - O \left( \frac{1}{n} \right) \\
    	\ge & D_{\sKL} \left( \rr(\bzero) \Vert \pp(\bzero) \right) + \frac{1}{\alpha} I(\bzero) + 2 \eta - O \left( \frac{1}{n} \right) \ge D_{\sKL} \left( \qq(\bzero) \Vert \pp(\bzero) \right) + \frac{1}{\alpha} I(\bzero) + \eta
    \end{align*}
    for sufficiently large $n$. Combining the above estimates, we finally get that
    \begin{equation*}
    	\min_{\lambda_{\max} (\QQ^\top \QQ) \in (1/2, 1)} \left\{ D_{\sKL} \left( \qq(\QQ) \Vert \pp(\QQ) \right) + \frac{1}{\alpha} I(\QQ) \right\} \ge D_{\sKL} \left( \qq(\bzero) \Vert \pp(\bzero) \right) + \frac{1}{\alpha} I(\bzero) + \eta.
    \end{equation*}
    This concludes the proof of Eq.~\eqref{eq:2nd_cond_complex}. Applying Lemmas~\ref{lem:2nd_moment}, we know that with probability bounded away from zero, there exists a random orthogonal matrix $\WW = \WW_n (\XX)$ such that $\langle \hat{P}_{n, \WW} \rangle_A = \qq$. By our assumption, $d_{\sTV} (\qq, \rr) = O(1/n)$. This completes the proof of Theorem~\ref{thm:discrete_lower_bd}.
\end{proof}

\begin{proof}[\bf Proof of Lemma~\ref{lem:var_rep_KL}]
	The first equality can be obtained by directly differentiating the right hand side, and one can verify that the supremum is achieved at $g(x) = \log Q(x) - \log P(x) + C$, where $C$ is any constant. Using this relationship, we obtain that
	\begin{align*}
		D_{\sKL} (\rr(\QQ) \Vert \pp(\QQ)) = & \sup_{(g_{ij})_{i, j \in [M]}} \left\{ \sum_{i, j = 1}^{M} g_{ij} r_{ij} (\QQ) - \log \left( \sum_{i, j=1}^{M} \exp \left( g_{ij} \right) p_{ij} (\QQ) \right) \right\} \\
		\ge & \sup_{(\lambda_i, \mu_i)_{i \in [M]}} \left\{ \sum_{i, j = 1}^{M} \left( \lambda_i + \mu_j \right) r_{ij} (\QQ) - \log \left( \sum_{i, j=1}^{M} \exp \left( \lambda_i + \mu_j \right) p_{ij} (\QQ) \right) \right\} \\
		\stackrel{(i)}{=} & \sup_{(\lambda_i, \mu_i)_{i \in [M]}} \left\{ \sum_{i = 1}^{M} \left( \lambda_i + \mu_i \right) r_i - \log \left( \sum_{i, j=1}^{M} \exp \left( \lambda_i + \mu_j \right) p_{ij} (\QQ) \right) \right\},
	\end{align*}
	where $(i)$ follows from the constraints on $(r_{ij} (\QQ))_{i, j \in [M]}$. Now, making a change of variable
	\begin{equation*}
		\lambda_i \mapsto \lambda_i + \log(r_i / p_i), \ \mu_j \mapsto \mu_j + \log(r_j / p_j),
	\end{equation*}
    we obtain that
    \begin{align*}
    	D_{\sKL} (\rr(\QQ) \Vert \pp(\QQ)) \ge & 2 D_{\sKL} (\rr \Vert \pp) + \sup_{(\lambda_i, \mu_i)_{i \in [M]}} \left\{ \sum_{i = 1}^{M} \left( \lambda_i + \mu_i \right) r_i - \log \left( \sum_{i, j=1}^{M} \exp \left( \lambda_i + \mu_j \right) \frac{r_i r_j p_{ij} (\QQ)}{p_i p_j} \right) \right\} \\
    	= & D_{\sKL} \left( \rr(\bzero) \Vert \pp(\bzero) \right) + \sup_{(\lambda_i, \mu_i)_{i \in [M]}} \left\{ \sum_{i = 1}^{M} \left( \lambda_i + \mu_i \right) r_i - \log \left( \sum_{i, j=1}^{M} \exp \left( \lambda_i + \mu_j \right) \frac{r_i r_j p_{ij} (\QQ)}{p_i p_j} \right) \right\}.
    \end{align*}
	 This completes the proof.
\end{proof}

\subsection{Proofs of Theorem~\ref{thm:inner_bd_1D} and Proposition~\ref{prop:char_lower_bd}}
\begin{proof}[\bf Proof of Theorem~\ref{thm:inner_bd_1D}]
First, we show that the optimum of the maximin problem on the right hand side of Eq.~\eqref{eq:def_feas_thres_cont} is achieved at $q_0 \in (0, 1)$, where $q_0$ is the unique solution to the equation
    \begin{equation}\label{eq:solve_q_0}
    	\frac{I(q)}{D_{\sKL} (P \Vert \sN (0, 1))} = \frac{1}{2 \left( c_2^2 + q \left( \chi^2 (P, \sN (0, 1)) - c_2^2 \right) \right)}.
    \end{equation}
Note that, on the one hand, the left hand side of Eq.~\eqref{eq:solve_q_0} is an increasing function of $q$, which equals $0$ when $q = 0$, and diverges to $+\infty$ when $q \to 1^{-}$. On the other hand, the right hand side of Eq.~\eqref{eq:solve_q_0} is decreasing in $q$ and always positive. According to the intermediate zero theorem, we know that Eq.~\eqref{eq:solve_q_0} has a unique solution $q_0 \in (0, 1)$. It then follows that
\begin{equation*}
	\alpha_{\rm lb} (P) = \frac{I(q_0)}{D_{\sKL} (P \Vert \sN (0, 1))} = \frac{1}{2 \left( c_2^2 + q_0 \left( \chi^2 (P, \sN (0, 1)) - c_2^2 \right) \right)}.
\end{equation*}

Next, we show that there exists a finite set $A$, and the corresponding discrete distribution $\rr$ obtained by projecting the probability measure $P$ onto $A$ (i.e., $\rr = \langle P \rangle_A$), such that $\alpha_{\rm lb} (P) \le \oalpha_{\rm lb} (\rr)$, and consequently $\alpha < \oalpha_{\rm lb} (\rr)$. According to Definition~\ref{def:2nd_thres}, this is equivalent to the claim that the function $q \in [-1, 1] \mapsto F(q) + I(q) / \alpha$ achieves its unique minimum at $q = 0$, where $F(q) = D_{\sKL} \left( \rr(q) \Vert \pp(q) \right)$. To prove our claim, let us fix this $q_0 \in (0, 1)$, and consider the following two situations: 

\noindent {\bf Case $(i)$.} $\vert q \vert \le q_0$: According to Lemma~\ref{lem:var_rep_KL}, we have (choose $\lambda_i = \mu_i = 0$ for all $i \in [M]$)
\begin{equation}\label{eq:disc_F_lower_bd}
	F(q) - F(0) \ge - \log \left( \sum_{i, j=1}^{M} \frac{r_i r_j p_{ij} (q)}{p_i p_j} \right) = - \log \left( 1 + \sum_{i, j=1}^{M} \frac{(r_i - p_i) (r_j - p_j) p_{ij} (q)}{p_i p_j} \right).
\end{equation}
Denote by $\phi_q (x, y)$ the PDF of $\sN (0, \begin{bmatrix}
	1 & q \\
	q & 1
\end{bmatrix} )$. Then, we have
\begin{align*}
	\sum_{i, j=1}^{M} \frac{(r_i - p_i) (r_j - p_j) p_{ij} (q)}{p_i p_j} = & \sum_{i, j=1}^{M} \int_{I_i} (p(x) - \phi(x)) \d x \cdot \int_{I_j} (p(x) - \phi(x)) \d x \cdot \frac{\int_{I_i \times I_j} \phi_q (x, y) \d x \d y}{\int_{I_i} \phi(x) \d x \cdot \int_{I_j} \phi (x) \d x} \\
	= & \sum_{i, j=1}^{M} \int_{I_i \times I_j} (p(x) - \phi(x)) (p(y) - \phi(y)) \d x \d y \cdot \frac{\int_{I_i \times I_j} \phi_q (x, y) \d x \d y}{\int_{I_i \times I_j} \phi(x) \phi(y) \d x \d y},
\end{align*}
where $I_i$ is the interval $\{ x \in \R: \langle x \rangle_A = a_i \}$. By direct computation, we get that
\begin{equation}\label{eq:phi_q_and_phi}
	\phi_q (x, y) = \phi (x) \phi (y) \cdot \frac{1}{\sqrt{1 - q^2}} \exp \left( - \frac{q^2}{2 (1 - q^2)} (x^2 + y^2) + \frac{q}{1 - q^2} xy \right) = \phi(x) \phi(y) \left( 1 + q xy + q^2 f_q (x, y) \right),
\end{equation}
where $\{ f_q (x, y) \}_{\vert q \vert \le q_0}$ is a family of equicontinuous functions on any compact set in $\R^2$. Using this representation for $\phi_q (x, y)$, we then obtain that
\begin{align*}
	& \sum_{i, j=1}^{M} \int_{I_i \times I_j} (p(x) - \phi(x)) (p(y) - \phi(y)) \d x \d y \cdot \frac{\int_{I_i \times I_j} \phi_q (x, y) \d x \d y}{\int_{I_i \times I_j} \phi(x) \phi(y) \d x \d y} \\
	= & \sum_{i, j=1}^{M} \int_{I_i \times I_j} (p(x) - \phi(x)) (p(y) - \phi(y)) \d x \d y \\
	& \times \left( 1 + q \cdot \frac{\int_{I_i \times I_j} x y \phi(x) \phi(y) \d x \d y}{\int_{I_i \times I_j} \phi(x) \phi(y) \d x \d y} + q^2 \cdot \frac{\int_{I_i \times I_j} f_q (x, y) \phi(x) \phi(y) \d x \d y}{\int_{I_i \times I_j} \phi(x) \phi(y) \d x \d y} \right) \\
	= & q \left( \sum_{i=1}^{M} \int_{I_i} (p(x) - \phi(x)) \d x \frac{\int_{I_i} x \phi(x) \d x}{\int_{I_i} \phi(x) \d x} \right)^2 \\
	& + q^2 \sum_{i, j=1}^{M} \int_{I_i \times I_j} (p(x) - \phi(x)) (p(y) - \phi(y)) \d x \d y \frac{\int_{I_i \times I_j} f_q (x, y) \phi(x) \phi(y) \d x \d y}{\int_{I_i \times I_j} \phi(x) \phi(y) \d x \d y}.
\end{align*}
By Assumption~\ref{ass:1D_inner_bd}, $\int_{\R} x p(x) = \int_{\R} x \phi (x) = 0$. Therefore, we can find a discrete set $A$ such that
\begin{equation*}
	\sum_{i=1}^{M} \int_{I_i} (p(x) - \phi(x)) \d x \frac{\int_{I_i} x \phi(x) \d x}{\int_{I_i} \phi(x) \d x} = 0,
\end{equation*}
thus leading to
\begin{align*}
	& \sum_{i, j=1}^{M} \int_{I_i \times I_j} (p(x) - \phi(x)) (p(y) - \phi(y)) \d x \d y \cdot \frac{\int_{I_i \times I_j} \phi_q (x, y) \d x \d y}{\int_{I_i \times I_j} \phi(x) \phi(y) \d x \d y} \\
	= & q^2 \sum_{i, j=1}^{M} \int_{I_i \times I_j} (p(x) - \phi(x)) (p(y) - \phi(y)) \d x \d y \frac{\int_{I_i \times I_j} f_q (x, y) \phi(x) \phi(y) \d x \d y}{\int_{I_i \times I_j} \phi(x) \phi(y) \d x \d y}.
\end{align*}
Since $\{ f_q (x, y) \}_{\vert q \vert \le q_0}$ is equicontinuous on any compact set, we know that for any $\veps > 0$, there exists a discretization $A = A_{\veps}$, such that the following holds uniformly for $q \in [- q_0, q_0]$:
\begin{small}
\begin{equation}\label{eq:2nd_order_approx}
	\Bigg\vert \sum_{i, j=1}^{M} \int_{I_i \times I_j} (p(x) - \phi(x)) (p(y) - \phi(y)) \d x \d y \frac{\int_{I_i \times I_j} f_q (x, y) \phi(x) \phi(y) \d x \d y}{\int_{I_i \times I_j} \phi(x) \phi(y) \d x \d y} - \int_{\R \times \R} (p(x) - \phi(x)) (p(y) - \phi(y)) f_q (x, y) \d x \d y \Bigg\vert \le \veps.
\end{equation}
\end{small}
This is achieved by first choosing $\max_{i \in [M]} \vert a_i \vert$ to be large enough, so that the integral outside the compact set (convex hull of $A^2$) is uniformly small for $\vert q \vert \le q_0$. Then, we can control the approximation error of the Riemann sum on a sufficiently refined grid. Combining Eq.s~\eqref{eq:disc_F_lower_bd} and \eqref{eq:2nd_order_approx}, we obtain that
\begin{align*}
	F(q) - F(0) \ge & - q^2 \sum_{i, j=1}^{M} \int_{I_i \times I_j} (p(x) - \phi(x)) (p(y) - \phi(y)) \d x \d y \frac{\int_{I_i \times I_j} f_q (x, y) \phi(x) \phi(y) \d x \d y}{\int_{I_i \times I_j} \phi(x) \phi(y) \d x \d y} \\
	\ge & - q^2 \left( \int_{\R \times \R} (p(x) - \phi(x)) (p(y) - \phi(y)) f_q (x, y) \d x \d y + \veps \right) \\
	= & - \veps q^2 - \int_{\R \times \R} (p(x) - \phi(x)) (p(y) - \phi(y)) \frac{\phi_q (x, y)}{\phi(x) \phi(y)} \d x \d y,
\end{align*}
where the last line follows from the definition of $f_q$. Therefore, if we can show that under the assumptions of Theorem~\ref{thm:inner_bd_1D},
\begin{equation}\label{eq:sup_KL_q0}
	\sup_{\vert q \vert \le q_0} \left\{ \frac{1}{q^2} \int_{\R \times \R} (p(x) - \phi(x)) (p(y) - \phi(y)) \frac{\phi_q (x, y)}{\phi(x) \phi(y)} \d x \d y \right\} < \frac{1}{2 \alpha},
\end{equation}
then there exists an $\veps > 0$ such that for all $q \in [-q_0, q_0]$,
\begin{equation*}
    F(q) - F(0) \ge - \veps q^2 - \left( \frac{1}{2 \alpha} - \veps \right) q^2 = - \frac{q^2}{2 \alpha} \ge - \frac{I(q)}{\alpha},
\end{equation*}
where the last line follows from the well-known inequality $I(q) \ge q^2 / 2$, and equality holds only for $q=0$. Hence, we have verified that
\begin{equation*}
	F(q) + \frac{I(q)}{\alpha} > F(0) + \frac{I(0)}{\alpha}, \ \forall q \in [-q_0, q_0] \backslash \{0\}.
\end{equation*}
The remaining part of Case $(i)$ will be devoted to proving Eq.~\eqref{eq:sup_KL_q0}. Denote $h(x) = (p(x) - \phi(x)) / \phi(x)$. Then according to Assumption~\ref{ass:1D_inner_bd}, we have
\begin{equation}\label{eq:cond_on_h}
    \E [h(G)] = \E [G h(G)] = 0, \ \text{and} \ \E[h^2 (G)] = \chi^2 \left( P, \sN(0, 1) \right) < \infty, \ G \sim \sN (0, 1),
\end{equation}
and similarly
\begin{equation*}
	\int_{\R \times \R} (p(x) - \phi(x)) (p(y) - \phi(y)) \frac{\phi_q (x, y)}{\phi(x) \phi(y)} \d x \d y = \E_{q} \left[ h(G_1) h(G_2) \right],
\end{equation*}
where $\E_q$ represents the expectation taken under $( G_{1}^\top, G_{2}^\top )^\top \sim \sN ( 0, \ \begin{bmatrix}
	1 & q \\
	q & 1
\end{bmatrix} )$. According to Eq.~\eqref{eq:cond_on_h}, we have the following expansion in $L^2 (\R, \gamma)$ (where $\gamma$ is the standard Gaussian measure):
\begin{equation*}
	h (x) = \sum_{n = 2}^{\infty} c_n {\rm He}_n (x) \implies \E \left[ h^2 (G) \right] = \sum_{n=2}^{\infty} c_n^2 < \infty.
\end{equation*}
Here, the sequence $\{ {\rm He}_n (x) \}_{n \ge 0}$ are normalized Hermite polynomials, satisfying that
\begin{equation*}
	\E \left[ {\rm He}_m (G) {\rm He}_n (G) \right] = \delta_{n m}, \ \forall n, m \in \mathbb{N},
\end{equation*}
where $\delta_{nm} = \bone_{n = m}$. Moreover, $\{ {\rm He}_n (x) \}_{n \ge 0}$ is a complete orthonormal basis of $L^2 (\R, \gamma)$.
Using this representation, we thus obtain that
\begin{equation*}
	\E_q \left[ h(G_1) h(G_2) \right] = \sum_{n, m = 2}^{\infty} c_n c_m \E_q \left[ {\rm He}_n (G_1) {\rm He}_m (G_2) \right] = \sum_{n=2}^{\infty} c_n^2 q^n,
\end{equation*}
where the last step follows from \cite[Lem. 3]{dudeja2018learning} (see also \cite{o2014analysis}). Note that for $q \in [-q_0, q_0]$, we have
\begin{equation*}
	\sum_{n=2}^{\infty} c_n^2 q^n \le c_2^2 q^2 + \sum_{n=3}^{\infty} c_n^2 q_0^{n - 2} q^2 \le \left( c_2^2 + q_0 \sum_{n = 3}^{\infty} c_n^2 \right) q^2 = \left( c_2^2 + q_0 \left( \chi^2 (P, \sN(0, 1)) - c_2^2 \right) \right) q^2.
\end{equation*}
Hence, Eq.~\eqref{eq:sup_KL_q0} is equivalent to
\begin{equation*}
	c_2^2 + q_0 \left( \chi^2 (P, \sN(0, 1)) - c_2^2 \right) < \frac{1}{2 \alpha} \Longleftrightarrow \alpha < \frac{1}{2 \left( c_2^2 + q_0 \left( \chi^2 (P, \sN(0, 1)) - c_2^2 \right) \right)},
\end{equation*}
which follows from the definition of $q_0$ and our assumption.

\noindent {\bf Case $(ii)$.} $\vert q \vert > q_0$: Using the chain rule for KL divergence, and notice $I(q) \ge I(q_0)$, we get that
\begin{equation*}
	F(q) + \frac{I(q)}{\alpha} = D_{\sKL} (\rr(q) \Vert \pp(q)) + \frac{I(q)}{\alpha} \ge D_{\sKL} (\rr \Vert \pp) + \frac{I(q_0)}{\alpha}.
\end{equation*}
Note that $F(0) = 2 D_{\sKL} (\rr \Vert \pp)$. Hence, as long as $\alpha < I(q_0) / D_{\sKL} (\rr \Vert \pp)$, we can conclude that
\begin{equation*}
	F(q) + \frac{I(q)}{\alpha} > F(0) + \frac{I(0)}{\alpha}, \quad \forall \vert q \vert > q_0.
\end{equation*}
Since we can choose the discretization $A$ to be refined enough such that $D_{\sKL} (\rr \Vert \pp)$ is (arbitrarily) close to $D_{\sKL} (r \Vert \phi) = D_{\sKL} (P, \sN(0, 1))$, we only need to require that $\alpha < I(q_0) / D_{\sKL} (P, \sN(0, 1))$, which is exactly the assumption of Theorem~\ref{thm:inner_bd_1D}. This concludes the proof for Case $(ii)$.

Combining our discussion for both cases, it follows that $F(q) + I(q) / \alpha$ achieves its unique minimum at $q = 0$, which completes the proof of our claim.

We now return to the proof of Theorem~\ref{thm:inner_bd_1D}. First, we show that for any $p \in [1, 2)$, there exists a constant $c > 0$ such that
	\begin{equation*}
		\liminf_{n \to \infty} \P \left( \inf_{\ww \in \S^{d - 1}} W_p (\hat{P}_{n, \ww}, P) \le \veps \right) \ge c
	\end{equation*}
    for all small enough $\veps > 0$. To this end, note that similar to the proof of Theorem~\ref{thm:W2_outer_bd}, we can show that with high probability,
	\begin{equation*}
		\sup_{\ww \in \S^{d - 1}} \mu_2 ( \hat{P}_{n, \ww} ) \le C(\alpha)
	\end{equation*}
    for some constant $C(\alpha)$ only depending on $\alpha$.  Since $\mu_2 (P) < \infty$, one can choose a discrete set $A$ such that $W_p (P, \langle P \rangle_A) \le \veps / 3$, and
    \begin{equation*}
    	\sup_{\ww \in \S^{d - 1}} W_p \left( \hat{P}_{n, \ww}, \langle \hat{P}_{n, \ww} \rangle_A \right) \le \frac{\veps}{3}
    \end{equation*}
    with high probability. According to Theorem~\ref{thm:discrete_lower_bd}, with probability bounded away from zero, there exists a random unit vector $\ww = \ww_n (\XX) \in \S^{d - 1}$ such that
    \begin{equation*}
        W_p \left( \langle \hat{P}_{n, \ww} \rangle_A, \langle P \rangle_A \right) = O(n^{-1}) \le \frac{\veps}{3}
    \end{equation*}
    for $n$ large enough, which further implies $W_p (\hat{P}_{n, \ww}, P) \le \veps$ by triangle inequality. Therefore, we deduce that
    \begin{equation*}
    	\liminf_{n \to \infty} \P \left( W_p (\hat{P}_{n, \ww}, P) \le \veps \right) \ge c(\veps) > 0,
    \end{equation*}
    where $c(\veps)$ only depends on $\veps$. Note that in the proof of Theorem~\ref{thm:discrete_lower_bd}, we actually showed that there exists $c > 0$ such that $c(\veps) \ge c$ for all small enough $\veps > 0$. Hence, there exists a $\veps_0 > 0$ such that for all $\veps \in (0, \veps_0)$,
    \begin{equation}\label{eq:inf_W_p_prob}
    	\liminf_{n \to \infty} \P \left( \inf_{\ww \in \S^{d - 1}} W_p (\hat{P}_{n, \ww}, P) \le \veps \right) \ge c.
    \end{equation}
    This proves our claim. Now we are in position to show that $P$ is $(\alpha, 1)$-feasible. Let us define for $\XX \in \R^{n \times d}$,
    \begin{equation*}
    	F(\XX) = \inf_{\ww \in \S^{d - 1}} W_p (\hat{P}_{n, \ww}, P) = \inf_{\ww \in \S^{d - 1}} W_p \left( \frac{1}{n} \sum_{i=1}^{n} \delta_{\langle \xx_i, \ww \rangle}, P \right).
    \end{equation*}
    Then, for any $\XX, \XX' \in \R^{n \times d}$, it follows that
    \begin{align*}
    	\left\vert F(\XX) - F(\XX') \right\vert = & \left\vert \inf_{\ww \in \S^{d - 1}} W_p \left( \frac{1}{n} \sum_{i=1}^{n} \delta_{\langle \xx_i, \ww \rangle}, P \right) - \inf_{\ww \in \S^{d - 1}} W_p \left( \frac{1}{n} \sum_{i=1}^{n} \delta_{\langle \xx_i', \ww \rangle}, P \right) \right\vert \\
    	\le & \sup_{\ww \in \S^{d - 1}} \left\vert W_p \left( \frac{1}{n} \sum_{i=1}^{n} \delta_{\langle \xx_i, \ww \rangle}, P \right) - W_p \left( \frac{1}{n} \sum_{i=1}^{n} \delta_{\langle \xx_i', \ww \rangle}, P \right) \right\vert \\
    	\le & \sup_{\ww \in \S^{d - 1}} W_p \left( \frac{1}{n} \sum_{i=1}^{n} \delta_{\langle \xx_i, \ww \rangle}, \frac{1}{n} \sum_{i=1}^{n} \delta_{\langle \xx_i', \ww \rangle} \right) \\
    	\le & \sup_{\ww \in \S^{d - 1}} W_2 \left( \frac{1}{n} \sum_{i=1}^{n} \delta_{\langle \xx_i, \ww \rangle}, \frac{1}{n} \sum_{i=1}^{n} \delta_{\langle \xx_i', \ww \rangle} \right) \\
    	\le & \sqrt{\frac{1}{n} \sum_{i=1}^{n} \langle \xx_i - \xx_i', \ww \rangle^2} = \frac{1}{\sqrt{n}} \norm{(\XX - \XX') \ww}_2 \\
    	\le & \frac{1}{\sqrt{n}} \norm{\XX - \XX'}_{\op} \norm{\ww}_2 \le \frac{1}{\sqrt{n}} \norm{\XX - \XX'}_{\rm F},
    \end{align*}
    meaning that $F(\XX)$ is ($1/\sqrt{n}$)-Lipschitz. Using Gaussian concentration inequality, we obtain that
    \begin{equation*}
    	\P \left( \left\vert F(\XX) - \E [F(\XX)] \right\vert \ge t \right) \le 2 \exp \left( - \frac{n t^2}{2} \right), \ \forall t > 0,
    \end{equation*}
    which further implies that $\P (\vert F(\XX) - \E [F(\XX)] \vert \ge \veps_n) \to 0$ for some $\veps_n \to 0$, for example, $\veps_n = n^{-1/4}$. Note that Eq.~\eqref{eq:inf_W_p_prob} implies $\liminf_{n \to \infty} \P (F(\XX) \le \veps) \ge c$. Hence, for large $n$ we must have $\E [F(\XX)] \le \veps + \veps_n$, and consequently $F(\XX) \le \veps + 2 \veps_n$ with high probability. To conclude, we have proved that for all $\veps \in (0, \veps_0)$,
    \begin{equation*}
    	\lim_{n \to \infty} \P \left( \inf_{\ww \in \S^{d - 1}} W_p (\hat{P}_{n, \ww}, P) \le \veps + 2 \veps_n \right) = 1.
    \end{equation*}
    As a consequence, there exists $l_0 \in \mathbb{N}$ such that for all $l \ge l_0$, we can find a sequence of random unit vectors $\{ \ww = \ww_{n, l} (\XX) \}_{n \ge 1}$ such that
    \begin{equation*}
    	\lim_{n \to \infty} \P \left( W_p (\hat{P}_{n, \ww}, P) \le \frac{1}{l} \right) = 1.
    \end{equation*}
    Using a diagonal argument similar to the proof of Lemma~\ref{lem:closure_prop}, we can show that $P$ is $(\alpha, 1)$-feasible, as desired. This completes the proof of Theorem~\ref{thm:inner_bd_1D}.
\end{proof}

\begin{proof}[\bf Proof of Proposition~\ref{prop:char_lower_bd}]
	For simplicity, let us denote
	\begin{equation*}
		D_1 = D_{\sKL} (P \Vert \sN(0, 1)), \ D_2 = c_2^2, \ D_3 = \chi^2 (P, \sN(0, 1)) - c_2^2.
	\end{equation*}
	By definition of $\alpha_{\rm lb} (P)$, we have
	\begin{align*}
		\alpha_{\rm lb} (P) = & \max_{q \in [0, 1]} \min \left\{ \frac{I(q)}{D_1}, \ \frac{1}{2 \left( D_2 + q D_3 \right)} \right\} \\
		\stackrel{(i)}{\ge} & \max_{q \in [0, 1]} \min \left\{ \frac{I(q)}{D_1}, \ \frac{1}{4 D_2}, \ \frac{1}{4 q D_3} \right\} \\
		= & \max_{q \in [0, 1]} \min \left\{ \frac{1}{4 D_2}, \ \min \left\{ \frac{I(q)}{D_1}, \ \frac{1}{4 q D_3} \right\} \right\} \\
		\stackrel{(ii)}{=} & \min \left\{ \frac{1}{4 D_2}, \ \max_{q \in [0, 1]} \min \left\{ \frac{I(q)}{D_1}, \ \frac{1}{4 q D_3} \right\} \right\},
	\end{align*}
	where $(i)$ follows from the simple inequality $1/(a + b) \ge \min \{ 1/2a, 1/2b \}$, and $(ii)$ is due to the fact that $1/4D_2$ does not depend on $q$. Noticing that $I(q) \ge q^2/2$, we obtain
	\begin{equation*}
		\max_{q \in [0, 1]} \min \left\{ \frac{I(q)}{D_1}, \ \frac{1}{4 q D_3} \right\} \ge \max_{q \in [0, 1]} \min \left\{ \frac{q^2}{2 D_1}, \ \frac{1}{4 q D_3} \right\} = \frac{1}{2^{5/3} D_1^{1/3} D_3^{2/3}}
	\end{equation*}
	by direct calculation. Therefore, it finally follows that
	\begin{equation*}
		\alpha_{\rm lb} (P) \ge \min \left\{ \frac{1}{4 D_2}, \ \frac{1}{2^{5/3} D_1^{1/3} D_3^{2/3}} \right\} \ge \frac{1}{4} \min \left\{ \frac{1}{c_2^2}, \ \frac{1}{D_{\sKL} (P \Vert \sN(0, 1))^{1/3} \left( \chi^2 (P, \sN(0, 1)) - c_2^2 \right)^{2/3} } \right\}.
	\end{equation*}
	This completes the proof.
\end{proof}

\subsection{Proof of Theorem~\ref{thm:inner_bd_mD}}
	Assume $\alpha < \alpha_{\rm lb} (P)$. Let $Q_0$ be the neighborhood of $\bzero$ in $\R^{m \times m}$ that satisfies our assumptions, we will first show that there exists a finite set $A$, and the corresponding discrete distribution $\rr$ obtained by projecting the probability measure $P$ onto $A$, i.e., $\rr = \langle P \rangle_A$, such that the function
	\begin{equation*}
		\QQ \in \left\{ \QQ \in \R^{m \times m}: \QQ^\top \QQ \preceq \id_m \right\} \mapsto F(\QQ) + \frac{I(\QQ)}{\alpha}
	\end{equation*}
	achieves its unique minimum at $\QQ = \bzero$, where $F(\QQ) = D_{\sKL} \left( \rr(\QQ) \Vert \pp(\QQ) \right)$.
 As in the proof of Theorem~\ref{thm:inner_bd_1D}, consider the following two cases:
	
	\noindent {\bf Case $(i)$.} $\QQ \in Q_0$: Similar to the discretization arguments in Case $(i)$ in the proof of Theorem~\ref{thm:inner_bd_1D}, we obtain that there exists a small $\veps > 0$, such that
	\begin{equation*}
		F(\QQ) - F(\bzero) \ge - \veps \norm{\QQ}_{\rm F}^2 - \int_{\R^m \times \R^m} \left( p(\xx) - \phi(\xx) \right) \left( p(\yy) - \phi(\yy) \right) \frac{\phi_{\QQ} (\xx, \yy)}{\phi(\xx) \phi(\yy)} \d \xx \d \yy,
	\end{equation*}
    where $\phi_{\QQ} (\xx, \yy)$ is the PDF of $\sN \left( \bzero, \ \begin{bmatrix}
    	\id_m & \QQ \\
    	\QQ^\top & \id_m
    \end{bmatrix} \right)$. Denoting $h(\xx) = (p(\xx) - \phi(\xx)) / \phi(\xx)$, we get that
    \begin{equation*}
    	\int_{\R^m \times \R^m} \left( p(\xx) - \phi(\xx) \right) \left( p(\yy) - \phi(\yy) \right) \frac{\phi_{\QQ} (\xx, \yy)}{\phi(\xx) \phi(\yy)} \d \xx \d \yy = \E_{\QQ} \left[ h(\GG^{(1)}) h(\GG^{(2)}) \right],
    \end{equation*}
    where by Assumption~\ref{ass:mD_inner_bd}, we know that
    \begin{equation*}
    	\E [h(\GG)] = \E [\GG h(\GG)] = 0, \ \text{and} \ \E[h^2 (\GG)] = \chi^2 (P, \sN(\bzero, \id_m)) < \infty, \ \GG \sim \sN (\bzero, \id_m).
    \end{equation*}
    Note that the above identities hold as well if we replace $h(\GG)$ by $h(\UU \GG)$, where $\UU$ is any orthogonal matrix in $O(m, m)$. This is due to the rotational invariance of standard Gaussian distribution. According to Lemma~\ref{lem:mul_her_com}, the function $\xx \mapsto h (\UU \xx)$ admits the following expansion in $L^2 (\R^m, \gamma^m)$: (here $\gamma^m$ is the standard Gaussian measure on $\R^m$)
    \begin{equation*}
    	h \left( \UU \xx \right) = \sum_{\vert \nn \vert \ge 2} c_{\nn} (\UU) {\rm He}_{n_1} (x_1) \cdots {\rm He}_{n_m} (x_m),
    \end{equation*}
    where $\xx = (x_1, \cdots, x_m)$, $\nn = (n_1, \cdots, n_m)$, $c_{\nn} (\UU) = c_{n_1, \cdots, n_m} (\UU)$, and $ \vert \nn \vert = n_1 + \cdots + n_m$. Now, let $\QQ = \UU \bD \VV^\top$ be the singular value decomposition of $\QQ$, where $\UU, \VV \in O(m, m)$, $\bD = \diag (q_1, \cdots, q_m)$ and $0 \le q_i \le 1$ for all $1 \le i \le m$. Then, it follows that
    \begin{align*}
    	& \E_{\QQ} \left[ h(\GG^{(1)}) h(\GG^{(2)}) \right] = \E_{\bD} \left[ h(\UU \GG^{(1)}) h (\VV \GG^{(2)}) \right] \\
    	= & \E_{\bD} \left[ \left( \sum_{\vert \nn \vert \ge 2} c_{\nn} (\UU) {\rm He}_{n_1} (G_1^{(1)}) \cdots {\rm He}_{n_m} (G_m^{(1)}) \right) \left( \sum_{\vert \nn \vert \ge 2} c_{\nn} (\VV) {\rm He}_{n_1} (G_1^{(2)}) \cdots {\rm He}_{n_m} (G_m^{(2)}) \right) \right] \\
    	= & \sum_{\vert \nn \vert \ge 2} \sum_{\vert \bbeta \vert \ge 2} c_{\nn} (\UU) c_{\bbeta} (\VV) \E_{\bD} \left[ \prod_{i=1}^{m} {\rm He}_{n_i} (G_i^{(1)}) {\rm He}_{\beta_i} (G_i^{(2)}) \right] \\
    	\stackrel{(i)}{=} & \sum_{\vert \nn \vert \ge 2} \sum_{\vert \bbeta \vert \ge 2} c_{\nn} (\UU) c_{\bbeta} (\VV) \prod_{i=1}^{m} \E_{q_i} \left[ {\rm He}_{n_i} (G_i^{(1)}) {\rm He}_{\beta_i} (G_i^{(2)}) \right] \\
    	\stackrel{(ii)}{=} & \sum_{\vert \nn \vert \ge 2} \sum_{\vert \bbeta \vert \ge 2} c_{\nn} (\UU) c_{\bbeta} (\VV) \prod_{i=1}^{m} \bone_{n_i = \beta_i} \cdot q_i^{n_i} = \sum_{\vert \nn \vert \ge 2} c_{\nn} (\UU) c_{\nn} (\VV) \prod_{i=1}^{m} q_i^{n_i} \\
    	\le & \frac{1}{2} \left( \sum_{\vert \nn \vert \ge 2} c_{\nn} (\UU)^2 \prod_{i=1}^{m} q_i^{n_i} + \sum_{\vert \nn \vert \ge 2} c_{\nn} (\VV)^2 \prod_{i=1}^{m} q_i^{n_i} \right) \le \sup_{\UU \in O(m, m)} \left\{ \sum_{\vert \nn \vert \ge 2} c_{\nn} (\UU)^2 \prod_{i=1}^{m} q_i^{n_i} \right\},
    \end{align*}
    where $(i)$ is due to the fact that $\bD = \diag (q_1, \cdots, q_m)$, and $(ii)$ follows from \cite[Lem. 3]{dudeja2018learning}. Moreover, by definition of $I(\QQ)$, we have
    \begin{equation*}
    	I (\QQ) = - \frac{1}{2} \log \det (\id_m - \QQ^\top \QQ) = - \frac{1}{2} \log \prod_{i=1}^{m} (1 - q_i^2) = - \frac{1}{2} \sum_{i=1}^{m} \log (1 - q_i^2) \ge \frac{1}{2} \sum_{i=1}^{m} q_i^2 = \frac{1}{2} \norm{\QQ}_{\rm F}^2.
    \end{equation*}
    Note that by our assumption, for any $\QQ \in Q_0$, the following inequality holds:
    \begin{equation}\label{eq:cond_in_Q_0}
        \sup_{\UU \in O(m, m)} \left\{ \sum_{\vert \nn \vert \ge 2} c_{\nn} (\UU)^2 \prod_{i=1}^{m} q_i^{n_i} \right\} < \frac{1}{2 \alpha} \sum_{i=1}^{m} q_i^2 = \frac{1}{2 \alpha} \norm{\QQ}_{\rm F}^2,
    \end{equation}
    Hence, we conclude that $F(\QQ) + I(\QQ) / \alpha > F(\bzero) + I(\bzero) / \alpha$ for any $\QQ \in Q_0 \backslash \{ \bzero \}$. This completes the discussion for Case $(i)$.
	
	\noindent {\bf Case $(ii)$.} $\QQ \notin Q_0$: The argument here is completely the same as Case $(ii)$ in the proof of Theorem~\ref{thm:inner_bd_1D}. We only need to require that
	\begin{equation}\label{eq:cond_out_Q_0}
		\alpha < \frac{\inf_{\QQ \notin Q_0} I(\QQ)}{D_{\sKL} (P \Vert \sN(\bzero, \id_m))},
	\end{equation}
	which is already guaranteed by the theorem statement.
    For example, if we take $Q_0 = \{ \QQ \in \R^{m \times m}: \QQ^\top \QQ \preceq \id_m, \ \norm{\QQ}_{\rm F}^2 \le m q_0^2 \}$ for some $q_0 \in (0, 1)$, then we have
    \begin{equation*}
    	\inf_{\QQ \notin Q_0} I(\QQ) = \inf_{\sum_{i=1}^{m} q_i^2 \ge m q_0^2} \sum_{i = 1}^{m} I(q_i) \ge \inf_{\sum_{i=1}^{m} q_i^2 \ge m q_0^2} \sum_{i = 1}^{m} \frac{q_i^2}{2} \ge \frac{m q_0^2}{2},
    \end{equation*}
    and Eq.~\eqref{eq:cond_out_Q_0} is equivalent to $\alpha < m q_0^2 / (2 D_{\sKL} (P \Vert \sN(\bzero, \id_m)))$.
    
    Combining the arguments in Case $(i)$ and $(ii)$ proves our claim. The rest part of this proof is completely identical to the sharp concentration argument in the proof of Theorem~\ref{thm:inner_bd_1D}.

\subsection{Proof of Lemma~\ref{lem:2nd_moment}}\label{sec:proof_lem_2nd}
	According to the Paley-Zygmund inequality,
	\begin{equation*}
		\P(Z > 0) = \E [\bone \{ Z > 0 \}] \ge \frac{\E [Z \bone \{ Z > 0 \}]^2}{\E [Z^2]} = \frac{\E [Z]^2}{\E [Z^2]},
	\end{equation*}
	where the intermediate bound follows from the Cauchy-Schwarz inequality. Therefore, it suffices to show that $\E [Z^2] = O_n (\E [Z]^2)$ as $n, d \to \infty$ with $n/d \to \alpha$. To this end, let us calculate the first and second moments of $Z$. By definition, for $i = 1, \cdots, M$,
	\begin{equation*}
		p_i = \P \left( \langle G \rangle_A = \aa_i \right), \quad \text{where} \ G \sim \sN (\bzero, \id_m).
	\end{equation*}
	Note that for any $\WW \in O(d, m)$, $G_i := \WW^\top \xx_i \sim_{\iid} \sN (\bzero, \id_m)$. Hence,
	\begin{equation*}
		\E [Z] = \P \left( \frac{1}{n} \sum_{i = 1}^{n} \delta_{\langle G_i \rangle_A} = \qq \right) = \frac{n!}{\prod_{i=1}^{M} k_i ! } \prod_{i=1}^{M} p_i^{k_i}.
	\end{equation*}
	By our assumption, for all $i \in [M]$, $k_i \to \infty$ as $n \to \infty$. Using Stirling's formula gives
	\begin{align*}
		\E [Z] & = \frac{n!}{\prod_{i=1}^{M} k_i ! } \prod_{i=1}^{M} p_i^{k_i} \sim \frac{(2 \pi)^{1/2} n^{n + 1 / 2} e^{- n}}{\prod_{i=1}^{M} (2 \pi)^{1/2} k_i^{k_i + 1/2} e^{-k_i} } \prod_{i=1}^{M} p_i^{k_i} \\ 
		& = (2 \pi n)^{-(M - 1) / 2} \prod_{i = 1}^{M} q_i^{-1/2} \exp \left( - n \sum_{i = 1}^{M} q_i \log \frac{q_i}{p_i} \right) \\
		& = (2 \pi n)^{-(M - 1) / 2} \prod_{i = 1}^{M} q_i^{-1/2} \exp \left( - n D_{\sKL} (\qq \Vert \pp) \right).
	\end{align*}
	To calculate $\E [Z^2]$, we first write $Z^2$ as a double integral:
	\begin{equation*}
		Z^2 = \int_{O(d, m) \times O(d, m)} \bone \left\{ \frac{1}{n} \sum_{i = 1}^{n} \delta_{\langle \WW_1^\top \xx_i \rangle_A} = \qq, \ \frac{1}{n} \sum_{i = 1}^{n} \delta_{\langle \WW_2^\top \xx_i \rangle_A} = \qq \right\} \mu_{d, m} (\d \WW_1) \mu_{d, m} (\d \WW_2).
	\end{equation*}
	Denoting $\QQ = \WW_1^\top \WW_2 \in \R^{m \times m}$, we observe two useful facts:
	\begin{itemize}
		\item [(a)] For $\WW_1, \WW_2 \sim_\iid \mu_{d, m}$, the probability density function of $\QQ$ is given by (cf. \cite[Prop. 2.1]{diaconis1992finite} or \cite[Prop. 7.3]{eaton1989group})
		\begin{equation*}
			p_{d, m} (\QQ) = (2 \pi)^{- m^2/2} \frac{\omega(d - m, m)}{\omega (d, m)} \det ( \id_m - \QQ^\top \QQ )^{(d - 2m - 1)/2} \cdot I_0 (\QQ^\top \QQ),
		\end{equation*}
		where $I_0 (\QQ^\top \QQ)$ is the indicator function of the set that all $m$ eigenvalues of $\QQ^\top \QQ$ are in $(0, 1)$, and $\omega(\cdot, \cdot)$ is the Wishart constant defined by
		\begin{equation*}
			\omega(s, t)^{-1} = \pi^{t(t - 1)/4} 2^{st/2} \prod_{j=1}^{t} \Gamma \left( \frac{s-j+1}{2} \right).
		\end{equation*}
		Using Stirling's formula $\Gamma (t) \sim \sqrt{2 \pi t} (t / e)^t$ for $t \gg 1$, we obtain that
		\begin{equation*}
			p_{d, m} (\QQ) \sim \left( \frac{d}{2 \pi} \right)^{m^2/2} \cdot \exp \left( \frac{d - 2m - 1}{2} \log \det \left( \id_m - \QQ^\top \QQ \right) \right) \cdot I_0 (\QQ^\top \QQ).
		\end{equation*}
		
		\item [(b)] For $1 \le i \le n$, we have
		\begin{equation}\label{eq:dist_G1_G2}
			\left( G_{1i}^\top, G_{2i}^\top \right)^\top := \left( \xx_i^\top \WW_1, \xx_i^\top \WW_2 \right)^\top \sim_\iid \sN \left( \bzero, \ \begin{bmatrix}
				\id_m & \QQ \\
				\QQ^\top & \id_m
			\end{bmatrix} \right).
		\end{equation}
	\end{itemize}
	Therefore, it follows that
	\begin{align*}
		\E \left[ Z^2 \right] = & \int_{O(d, m) \times O(d, m)} \P \left( \frac{1}{n} \sum_{i = 1}^{n} \delta_{\langle \WW_1^\top \xx_i \rangle_A} = \qq, \ \frac{1}{n} \sum_{i = 1}^{n} \delta_{\langle \WW_2^\top \xx_i \rangle_A} = \qq \right) \mu_{d, m} (\d \WW_1) \mu_{d, m} (\d \WW_2) \\
		= & \int_{\R^{m \times m}} \P_{\QQ} \left( \frac{1}{n} \sum_{i = 1}^{n} \delta_{\langle G_{1i} \rangle_A} = \qq, \ \frac{1}{n} \sum_{i = 1}^{n} \delta_{\langle G_{2i} \rangle_A} = \qq \right) p_{d, m} (\QQ) \d \QQ \\
		\sim & \int_{\R^{m \times m}} \left( \frac{d}{2 \pi} \right)^{m^2/2} \P_{\QQ} \left( \frac{1}{n} \sum_{i = 1}^{n} \delta_{\langle G_{1i} \rangle_A} = \qq, \ \frac{1}{n} \sum_{i = 1}^{n} \delta_{\langle G_{2i} \rangle_A} = \qq \right) \\
		& \times \exp \left( - (d - 2m - 1) I(\QQ) \right) \cdot I_0 (\QQ^\top \QQ) \d \QQ,
	\end{align*}
	where $\P_{\QQ}$ denotes the probability taken under Eq.~\eqref{eq:dist_G1_G2}, $I(\QQ) = - \log \det (\id_m - \QQ^\top \QQ) / 2$, and $\d \QQ$ is understood as $\prod_{i, j = 1}^{m} \d Q_{ij}$. To compute the probability in the above integral, recall that
	\begin{equation}\label{eq:def_p_Q}
		p_{ij} (\QQ) = \P_{\QQ} \left( \langle G_1 \rangle_A = a_i, \ \langle G_2 \rangle_A = a_j \right), \quad \text{for} \ i, j \in [M],
	\end{equation}
	and define
	\begin{equation*}
		S_{\kk} = \left\{ \nn = (n_{ij})_{i, j \in [M]}: \forall i, j, \ \sum_{j=1}^{M} n_{ij} = k_i, \ \sum_{i=1}^{M} n_{ij} = k_j \right\}.
	\end{equation*}
	Moreover, recall the definition of $\qq(\QQ)$ from Eq.~\eqref{eq:def_info_proj_q}, we deduce that
	\begin{align*}
		& \P_{\QQ} \left( \frac{1}{n} \sum_{i = 1}^{n} \delta_{\langle G_{1i} \rangle_A} = \qq, \ \frac{1}{n} \sum_{i = 1}^{n} \delta_{\langle G_{2i} \rangle_A} = \qq \right) = \sum_{\nn \in S_{\kk}} \frac{n!}{\prod_{i, j=1}^{M} n_{ij}!} \prod_{i, j=1}^{M} p_{ij} (\QQ)^{n_{ij}} \\
		= & \sum_{\nn \in S_{\kk}} \frac{n!}{\prod_{i, j=1}^{M} n_{ij}!} \prod_{i, j=1}^{M} q_{ij} (\QQ)^{n_{ij}} \left( \frac{p_{ij} (\QQ)}{q_{ij} (\QQ)} \right)^{n_{ij}} = \sum_{\nn \in S_{\kk}} \frac{n!}{\prod_{i, j=1}^{M} n_{ij}!} \prod_{i, j=1}^{M} q_{ij} (\QQ)^{n_{ij}} \exp \left( - \sum_{i, j=1}^{M} n_{ij} \log \frac{q_{ij} (\QQ)}{p_{ij} (\QQ)} \right) \\
		= & \exp \left( - n D_{\sKL} \left( \qq(\QQ) \Vert \pp(\QQ) \right) \right) \sum_{\nn \in S_{\kk}} \frac{n!}{\prod_{i, j=1}^{M} n_{ij}!} \prod_{i, j=1}^{M} q_{ij} (\QQ)^{n_{ij}} \exp \left( - n \sum_{i, j=1}^{M} \left( \frac{n_{ij}}{n} - q_{ij} (\QQ) \right) \log \frac{q_{ij} (\QQ)}{p_{ij} (\QQ)} \right).
	\end{align*}
	According to Lemma~\ref{lem:info_proj_cond} (b) and the optimality of $\qq (\QQ)$, we know that $\forall \nn \in S_{\kk}$,
	\begin{equation*}
		\sum_{i, j=1}^{M} \left( \frac{n_{ij}}{n} - q_{ij} (\QQ) \right) \log \frac{q_{ij} (\QQ)}{p_{ij} (\QQ)} = 0,
	\end{equation*}
	which further implies that
	\begin{equation*}
		\P_{\QQ} \left( \frac{1}{n} \sum_{i = 1}^{n} \delta_{\langle G_{1i} \rangle_A} = \qq, \ \frac{1}{n} \sum_{i = 1}^{n} \delta_{\langle G_{2i} \rangle_A} = \qq \right) = \exp \left( - n D_{\sKL} \left( \qq(\QQ) \Vert \pp(\QQ) \right) \right) \sum_{\nn \in S_{\kk}} \frac{n!}{\prod_{i, j=1}^{M} n_{ij}!} \prod_{i, j=1}^{M} q_{ij} (\QQ)^{n_{ij}}.
	\end{equation*}
	As a consequence, we have
	\begin{equation}\label{eq:PQ_crude_bound}
		\P_{\QQ} \left( \frac{1}{n} \sum_{i = 1}^{n} \delta_{\langle G_{1i} \rangle_A} = \qq, \ \frac{1}{n} \sum_{i = 1}^{n} \delta_{\langle G_{2i} \rangle_A} = \qq \right) \le \exp \left( - n D_{\sKL} \left( \qq(\QQ) \Vert \pp(\QQ) \right) \right).
	\end{equation}
	Note that this inequality can be deduced from \cite[Lem. 2.1.9]{Dembo_2010} as well.
	
	Now, let us consider the following two cases:
	
	{\bf Case $(i)$.} $\lambda_{\max} (\QQ^\top \QQ) > 1/2$. Using Eq.~\eqref{eq:PQ_crude_bound}, it follows that
	\begin{align*}
		& \int_{\lambda_{\max} (\QQ^\top \QQ) > 1/2} \left( \frac{d}{2 \pi} \right)^{m^2/2} \P_{\QQ} \left( \frac{1}{n} \sum_{i = 1}^{n} \delta_{\langle G_{1i} \rangle_A} = \qq, \ \frac{1}{n} \sum_{i = 1}^{n} \delta_{\langle G_{2i} \rangle_A} = \qq \right) \\
		& \times \exp \left( - (d - 2m - 1) I(\QQ) \right) \cdot I_0 (\QQ^\top \QQ) \d \QQ \\
		\le & \int_{\lambda_{\max} (\QQ^\top \QQ) \in (1/2, 1)} \left( \frac{d}{2 \pi} \right)^{m^2/2} \exp \left( - n D_{\sKL} \left( \qq(\QQ) \Vert \pp(\QQ) \right) - (d - 2m - 1) I(\QQ) \right) \d \QQ \\
		= & \int_{\lambda_{\max} (\QQ^\top \QQ) \in (1/2, 1)} \left( \frac{d}{2 \pi} \right)^{m^2/2} \exp \left( - n \left( D_{\sKL} \left( \qq(\QQ) \Vert \pp(\QQ) \right) + \frac{d - 2m - 1}{n} I(\QQ) \right) \right) \d \QQ.
	\end{align*}
	Note that $(d - 2m - 1)/n \to \alpha$, and by our assumption,
	\begin{equation*}
		\inf_{\lambda_{\max} (\QQ^\top \QQ) \in (1/2, 1)} \left\{ D_{\sKL} \left( \qq(\QQ) \Vert \pp(\QQ) \right) + \frac{1}{\alpha} I(\QQ) \right\} \ge D_{\sKL} \left( \qq(\bzero) \Vert \pp(\bzero) \right) + \frac{1}{\alpha} I(\bzero) + \eta = 2 D_{\sKL} \left( \qq \Vert \pp \right) + \eta,
	\end{equation*}
	where the last step follows from Lemma~\ref{lem:info_proj_cond} (a). This in turn implies that, for all large enough $n$ and $d$,
	\begin{equation}\label{eq:est_case_1}
		\begin{split}
			& \int_{\lambda_{\max} (\QQ^\top \QQ) \in (1/2, 1)} \left( \frac{d}{2 \pi} \right)^{m^2/2} \exp \left( - n \left( D_{\sKL} \left( \qq(\QQ) \Vert \pp(\QQ) \right) + \frac{d - 2m - 1}{n} I(\QQ) \right) \right) \d \QQ \\
			\le & \int_{\lambda_{\max} (\QQ^\top \QQ) \in (1/2, 1)} \left( \frac{d}{2 \pi} \right)^{m^2/2} \exp \left( - 2 n \left( D_{\sKL} \left( \qq \Vert \pp \right) + \eta/4 \right) \right) \d \QQ \\
			\le & n^{C_m} \exp \left( - 2 n \left( D_{\sKL} \left( \qq \Vert \pp \right) + \eta/4 \right) \right) = o_n \left( \E[Z]^2 \right).
		\end{split}
	\end{equation}
	It then suffices to consider only the integral over $\{ \QQ \in \R^{m \times m}: \lambda_{\max} (\QQ^\top \QQ) \le 1/2 \}$.
	
	{\bf Case $(ii)$.} $\lambda_{\max} (\QQ^\top \QQ) \le 1/2$. We start with defining for fixed $\eta \in (0, 1)$:
	\begin{equation*}
		B_{n, \kk} (\eta) = \left\{ \nn \in S_{\kk}: \ \left\vert \frac{n_{ij}}{n q_{ij} (\QQ)} - 1 \right\vert \le \frac{\eta}{n^{1/3}}, \ \forall i, j \in [M] \right\}.
	\end{equation*}
	Then, according to the local limit theorem for multinomial distribution (cf. \cite[Thm. 2.1]{ouimet2021precise} and \cite[Lem. 2.1]{siotani1984asymptotic}), we obtain that
	\begin{footnotesize}
		\begin{equation*}
			\sum_{\nn \in B_{n, \kk} (\eta)} \frac{n!}{\prod_{i, j=1}^{M} n_{ij}!} \prod_{i, j=1}^{M} q_{ij} (\QQ)^{n_{ij}} = \frac{ 1 + O ( n^{-1/2} )}{(2 \pi n)^{(M^2 - 1)/2}} \left( \prod_{i, j=1}^{M} q_{ij} (\QQ) \right)^{-1/2} \sum_{\nn \in B_{n, \kk} (\eta)} \exp \left( - \frac{1}{2} \sum_{i, j=1}^{M} \frac{\left( n_{ij} / \sqrt{n} - \sqrt{n} q_{ij} (\QQ) \right)^2}{q_{ij}(\QQ)} \right),
		\end{equation*}
	\end{footnotesize}
	where the quantity $O(n^{-1/2})$ is uniformly small for all $\QQ$ such that $\lambda_{\max} (\QQ^\top \QQ) \le 1/2$, since
	\begin{equation}\label{eq:q_bdd_away}
		\inf_{\lambda_{\max} (\QQ^\top \QQ) \le 1/2} \inf_{i, j \in [M]} q_{ij} (\QQ) > 0,
	\end{equation}
	which is a result of Lemma~\ref{lem:info_proj_cond} (c). Now define for $i, j \in [M]$, $x_{ij} = n_{ij} / \sqrt{n} - \sqrt{n} q_{ij} (\QQ)$, we know that the $x_{ij}$'s should satisfy
	\begin{equation*}
		\forall i, j \in [M], \ \sum_{j=1}^{M} x_{ij} = \sum_{i=1}^{M} x_{ij} = 0, \quad \text{and} \ \vert x_{ij} \vert \le \eta n^{1/6} q_{ij} (\QQ).
	\end{equation*}
	These constraints define an open set in an $(M-1)^2$-dimensional subspace of $\R^{M^2}$, with diverging radius as $n \to \infty$. Therefore, the summation over $\nn \in B_{n, \kk} (\eta)$ is well-approximated by a Riemann integral, namely
	\begin{align*}
		& \sum_{\nn \in B_{n, \kk} (\eta)} \exp \left( - \frac{1}{2} \sum_{i, j=1}^{M} \frac{\left( n_{ij} / \sqrt{n} - \sqrt{n} q_{ij} (\QQ) \right)^2}{q_{ij}(\QQ)} \right) = \sum_{\nn \in B_{n, \kk} (\eta)} \exp \left( - \frac{1}{2} \sum_{i, j=1}^{M} \frac{x_{ij}^2}{q_{ij}(\QQ)} \right) \\
		\sim & n^{(M-1)^2/2} \int_{S_M} \exp \left( - \frac{1}{2} \sum_{i, j=1}^{M} \frac{x_{ij}^2}{q_{ij}(\QQ)} \right) \d \xx = C(M, \QQ) \cdot n^{(M-1)^2/2},
	\end{align*}
	where $S_M = \{ \xx = (x_{ij})_{i, j \in [M]}: \ \forall i, j \in [M], \ \sum_{j=1}^{M} x_{ij} = \sum_{i=1}^{M} x_{ij} = 0 \}$, and we define
	\begin{equation*}
		C(M, \QQ) := \int_{S_M} \exp \left( - \frac{1}{2} \sum_{i, j=1}^{M} \frac{x_{ij}^2}{q_{ij}(\QQ)} \right) \d \xx,
	\end{equation*}
	which is a constant only depending on $M$ and $\QQ$. The convergence here is uniform in $\{ \QQ: \lambda_{\max} (\QQ^\top \QQ) \le 1/2 \}$ because of Eq.~\eqref{eq:q_bdd_away}. Now let us consider the summation outside $B_{n, \kk} (\eta)$. For any $\nn \in S_{\kk} \backslash B_{n, \kk} (\eta)$, using again Stirling's formula yields
	\begin{align*}
		\frac{n!}{\prod_{i, j=1}^{M} n_{ij}!} \prod_{i, j=1}^{M} q_{ij} (\QQ)^{n_{ij}} \le & C(M) (2 \pi n)^{- (M^2 - 1)/2} \left( \prod_{i, j=1}^{M} \frac{n}{n_{ij}} \right)^{1/2} \exp \left( \sum_{i, j=1}^{M} n_{ij} \log \frac{n q_{ij}(\QQ)}{n_{ij}} \right) \\
		\le & C(M) n^{1/2} \exp \left( 2 \sum_{i, j=1}^{M} n_{ij} \log \sqrt{\frac{n q_{ij}(\QQ)}{n_{ij}}} \right) \\
		\le & C(M) n^{1/2} \exp \left( 2 \sum_{i, j=1}^{M} n_{ij} \left( \sqrt{\frac{n q_{ij}(\QQ)}{n_{ij}}} - 1 \right) \right) \\
		= & C(M) n^{1/2} \exp \left( - n \sum_{i, j=1}^{M} \left( \sqrt{\frac{n_{ij}}{n}} - \sqrt{q_{ij} (\QQ)} \right)^2 \right),
	\end{align*}
	where $C(M)$ is a constant only depending on $M$. By definition, there exists $(i, j) \in [M]^2$ such that $\vert n_{ij}/n - q_{ij} (\QQ) \vert \ge \eta q_{ij} (\QQ) n^{-1/3}$, thus leading to
	\begin{equation*}
		\left( \sqrt{\frac{n_{ij}}{n}} - \sqrt{q_{ij} (\QQ)} \right)^2 \ge \frac{\eta^2 q_{ij} (\QQ)^2 n^{-2/3}}{4} \ge \frac{\eta^2 \inf_{i, j \in [M]} q_{ij} (\QQ)^2 n^{-2/3}}{4},
	\end{equation*}
	which further implies that
	\begin{equation*}
		\sum_{\nn \in S_{\kk} \backslash B_{n, \kk} (\eta)} \frac{n!}{\prod_{i, j=1}^{M} n_{ij}!} \prod_{i, j=1}^{M} q_{ij} (\QQ)^{n_{ij}} \le C(M) n^{M^2 + 1/2} \exp \left( - \frac{\eta^2 \inf_{i, j \in [M]} q_{ij} (\QQ)^2 n^{1/3}}{4} \right).
	\end{equation*}
	The right hand side of the above equation is $o_n (1)$ uniformly for all $\QQ$ satisfying $\lambda_{\max} (\QQ^\top \QQ) \le 1/2$. Combining the above estimates yields
	\begin{equation*}
		\P_{\QQ} \left( \frac{1}{n} \sum_{i = 1}^{n} \delta_{\langle G_{1i} \rangle_A} = \qq, \ \frac{1}{n} \sum_{i = 1}^{n} \delta_{\langle G_{2i} \rangle_A} = \qq \right) \sim C(M, \QQ) \cdot n^{-(M - 1)} \exp \left( - n D_{\sKL} \left( \qq(\QQ) \Vert \pp(\QQ) \right) \right),
	\end{equation*}
	thus leading to
	\begin{equation}\label{eq:est_case_2}
		\begin{split}
			& \int_{\lambda_{\max} (\QQ^\top \QQ) \le 1/2} \left( \frac{d}{2 \pi} \right)^{m^2/2} \P_{\QQ} \left( \frac{1}{n} \sum_{i = 1}^{n} \delta_{\langle G_{1i} \rangle_A} = \qq, \ \frac{1}{n} \sum_{i = 1}^{n} \delta_{\langle G_{2i} \rangle_A} = \qq \right) \exp \left( - (d - 2m - 1) I(\QQ) \right) \cdot I_0 (\QQ^\top \QQ) \d \QQ \\
			\sim & n^{-(M - 1)} \int_{\lambda_{\max} (\QQ^\top \QQ) \le 1/2} \left( \frac{d}{2 \pi} \right)^{m^2/2} C(M, \QQ) \exp \left( - n D_{\sKL} \left( \qq(\QQ) \Vert \pp(\QQ) \right) - (d - 2m - 1) I(\QQ) \right) \d \QQ.
		\end{split}
	\end{equation}
	From the asymptotics for $\E [Z]$ derived earlier, we know that $C(M, \bzero) = (2 \pi)^{-(M - 1)} \prod_{i=1}^{M} q_i^{-1}$.
	
	Based on the estimates in Case $(i)$ and $(ii)$, i.e., Eq.~\eqref{eq:est_case_1} and Eq.~\eqref{eq:est_case_2}, we finally obtain that
	\begin{small}
		\begin{align*}
			\E [Z^2] \le \, & o_n(\E[Z]^2) \\
			& + (1 + o_n(1))n^{-(M - 1)} \int_{\lambda_{\max} (\QQ^\top \QQ) \le 1/2} \left( \frac{d}{2 \pi} \right)^{m^2/2} C(M, \QQ) \exp \left( - n D_{\sKL} \left( \qq(\QQ) \Vert \pp(\QQ) \right) - (d - 2m - 1) I(\QQ) \right) \d \QQ.
		\end{align*}
	\end{small}
	Hence, it suffices to show that
	\begin{equation}\label{eq:2nd_target}
	\begin{split}
		& \int_{\lambda_{\max} (\QQ^\top \QQ) \le 1/2} \left( \frac{d}{2 \pi} \right)^{m^2/2} \frac{C(M, \QQ)}{C(M, \bzero)} \exp \left( - n D_{\sKL} \left( \qq(\QQ) \Vert \pp(\QQ) \right) - (d - 2m - 1) I(\QQ) \right) \d \QQ \\
		= \, & O_n \left( \exp \left( -2n D_{\sKL} (\qq \Vert \pp) \right) \right).
	\end{split}
	\end{equation}
	By our assumption, one can choose $\alpha' > \alpha$ such that
	\begin{equation*}
		\min_{\lambda_{\max} (\QQ^\top \QQ) \le 1/2} \left\{ D_{\sKL} \left( \qq(\QQ) \Vert \pp(\QQ) \right) + \frac{1}{\alpha'} I(\QQ) \right\} \ge D_{\sKL} \left( \qq(\bzero) \Vert \pp(\bzero) \right) + \frac{1}{\alpha'} I(\bzero) - O \left( \frac{1}{n^2} \right).
	\end{equation*}
	Therefore, for large enough $n, d$, it follows that
	\begin{align*}
		& \int_{\lambda_{\max} (\QQ^\top \QQ) \le 1/2} \left( \frac{d}{2 \pi} \right)^{m^2/2} \frac{C(M, \QQ)}{C(M, \bzero)} \exp \left( - n D_{\sKL} \left( \qq(\QQ) \Vert \pp(\QQ) \right) - (d - 2m - 1) I(\QQ) \right) \d \QQ \\
		\stackrel{(i)}{\le} & \int_{\lambda_{\max} (\QQ^\top \QQ) \le 1/2} \left( \frac{d}{2 \pi} \right)^{m^2/2} \frac{C(M, \QQ)}{C(M, \bzero)} \exp \left( - n \left( D_{\sKL} \left( \qq(\QQ) \Vert \pp(\QQ) \right) + \frac{1}{\alpha'} I(\QQ) \right) \right) \d \QQ \\
		\stackrel{(ii)}{\le} & (1 + o_n(1)) C(m, M, \alpha')\left( \exp \left( - n \cdot \min_{\lambda_{\max} (\QQ^\top \QQ) \le 1/2} \left\{ D_{\sKL} \left( \qq(\QQ) \Vert \pp(\QQ) \right) + \frac{1}{\alpha'} I(\QQ) \right\} \right) \right) \\
		\stackrel{(iii)}{=} & (1 + o_n(1)) C(m, M, \alpha') \left( \exp \left( - n \cdot \left( D_{\sKL} \left( \qq(\bzero) \Vert \pp(\bzero) \right) + \frac{1}{\alpha'} I(\bzero) \right) \right) \right) \\
		= & O_n \left( \exp \left( - 2 n \cdot  D_{\sKL} \left( \qq \Vert \pp \right) \right) \right),
	\end{align*}
	as desired. Here, $(i)$ is because of $n/(d - 2m - 1) \to \alpha <\alpha'$, in $(ii)$ we use Laplace's method (see e.g., Chap. 4.2 of \cite{de1981asymptotic}), and $(iii)$ follows from the choice of $\alpha'$. The constant $C(m, M, \alpha')$ has an explicit form:
	\begin{equation*}
		C(m, M, \alpha') = \left( \det \left( \nabla_{\QQ = \bzero}^2 \left( F(\QQ) + \frac{I(\QQ)}{\alpha'} \right) \right) \right)^{-1/2},
	\end{equation*}
	where $F(\QQ) = D_{\sKL} \left( \qq(\QQ) \Vert \pp(\QQ) \right)$. This completes the proof of Lemma~\ref{lem:2nd_moment}. Note that here the constant $C(m, M, \alpha')$ is uniformly upper bounded for $\rr = \langle P \rangle_A$ with $P$ a fixed probability measure, and the discrete set $A$ sufficiently refined. As a consequence, there exists a constant $c > 0$ satisfying that $\liminf_{n \to \infty} \P (Z > 0) \ge c$ for all such $A$ and $\rr = \langle P \rangle_A$.

\section{Proofs for Section~\ref{sec:MainSupervised}: Applications to supervised learning}

\subsection{Proof of Proposition~\ref{prop:ERM_asymptotics}}
We first show that
\begin{equation*}
	\plimsup_{n,d\to\infty}\hR_n^{\star}(\XX,\yy) \le \inf_{h\in \cH_m} 
	\inf_{P\in\cuF^{\varphi}_{m,\alpha}}
	\int_{\{\pm 1\} \times \R^m} L(y,h(\zz))\, P(\d y,\d \zz)\, .
\end{equation*}
For any fixed $\veps > 0$, choose $h_{\veps} \in \cH_m$ and $P_{\veps} \in \cuF_{m,\alpha}^{\varphi}$ such that
\begin{equation*}
	\int_{\{ \pm 1 \} \times \R^m} L \left( y, h_{\veps} (\zz) \right) P_{\veps} (\d y, \d \zz) \le \inf_{h\in \cH_m} 
	\inf_{P\in\cuF^{\varphi}_{m,\alpha}}
	\int_{\{ \pm 1 \} \times \R^m} L(y,h(\zz))\, P(\d y,\d \zz) + \veps.
\end{equation*}
Since $P_{\veps} \in \cuF_{m,\alpha}^{\varphi}$, we know that there exists a sequence of random orthogonal matrices $\{ \WW = \WW_n (\XX, \yy) \}_{n \in \mathbb{N}}$ such that $\hat{P}_{n, \WW} \stackrel{w}{\Rightarrow} P_{\veps}$ in probability. As $h_{\veps}$ is Lipschitz-continuous and $L$ is bounded continuous, we deduce that as $n \to \infty$,
\begin{equation*}
	\int_{\{ \pm 1 \} \times \R^m} L \left( y, h_{\veps} (\zz) \right) \hat{P}_{n, \WW} (\d y, \d \zz) \to \int_{\{ \pm 1 \} \times \R^m} L \left( y, h_{\veps} (\zz) \right) P_{\veps} (\d y, \d \zz) \ \text{in probability}.
\end{equation*}
By definition, we have
\begin{equation*}
	\int_{\{ \pm 1 \} \times \R^m} L \left( y, h_{\veps} (\zz) \right) \hat{P}_{n, \WW} (\d y, \d \zz) \ge \hat{R}_n^{\star} (\XX, \yy),
\end{equation*}
thus leading to
\begin{align*}
	& \P \left( \hat{R}_n^{\star} (\XX, \yy) \ge \inf_{h\in \cH_m} 
	\inf_{P\in\cuF^{\varphi}_{m,\alpha}}
	\int_{\{ \pm 1 \} \times \R^m} L(y,h(\zz))\, P(\d y,\d \zz) + 2 \veps \right) \\
	\le & \P \left( \int_{\{ \pm 1 \} \times \R^m} L \left( y, h_{\veps} (\zz) \right) \hat{P}_{n, \WW} (\d y, \d \zz) \ge \int_{\{ \pm 1 \} \times \R^m} L \left( y, h_{\veps} (\zz) \right) P_{\veps} (\d y, \d \zz) + \veps \right) \to 0
\end{align*}
as $n \to \infty$. Since $\veps > 0$ is arbitrary, this proves
\begin{equation*}
	\plimsup_{n,d\to\infty}\hR_n^{\star}(\XX,\yy) \le \inf_{h\in \cH_m} 
	\inf_{P\in\cuF^{\varphi}_{m,\alpha}}
	\int_{\{ \pm 1 \} \times \R^m} L(y,h(\zz))\, P(\d y,\d \zz).
\end{equation*}
Next we prove the inverse bound. Defining the set of $(\alpha, m)$-feasible random probability measures on $\cuP (\{ \pm 1 \} \times \R^m)$ as:
\begin{align*}
	\widetilde{\cuF}^{\varphi}_{m,\alpha}:= \Big\{ \tilde{P} \ \text{is a random element in} \ \cuP ( \{ \pm 1 \} \times \R^{m}):\;  \exists
	\WW= \WW_n (\XX,\yy,\omega), \ \WW^\top \WW = \id_d \\
	\mbox{ such that }
	\frac{1}{n} \sum_{i=1}^{n} \delta_{\left(y_i,\xx_i^\top \WW \right)} \stackrel{w}{\Rightarrow} \tilde{P} \mbox{ in distribution }
	\Big\}\, ,
\end{align*}
we will first show that
\begin{equation}\label{eq:ERM_inter}
	\inf_{ \tilde{P} \in \widetilde{\cuF}^{\varphi}_{m,\alpha}}
	\E_{\tilde{P}} \left[ \inf_{h\in \cH_m} \int_{\{ \pm 1 \} \times \R^m} L(y,h(\zz))\, \tilde{P} (\d y,\d \zz) \right] \le \plimsup_{n,d\to\infty}\hR_n^{\star}(\XX,\yy)\, ,
\end{equation}
where the expectation $\E_{\tilde{P}}$ is taken over the randomness in $\tilde{P} \in \widetilde{\cuF}^{\varphi}_{m,\alpha}$. To show Eq.~\eqref{eq:ERM_inter}, let us denote by $\WW = \WW_n (\XX, \yy)$ the corresponding empirical minimizer of $\hat{R}_n (\XX, \yy)$ and $\hat{P}_n$ the empirical distribution of $\{ (y_i, \WW^\top \xx_i) \}_{i \le n}$, i.e., $\hat{P}_n = (1 / n) \sum_{i = 1}^{n} \delta_{(y_i, \WW^\top \xx_i)}$, we thus obtain that
\begin{equation*}
	\hR_n^{\star}(\XX,\yy) = \inf_{h \in \cH_m} \int_{\{ \pm 1 \} \times \R^m} L \left( y, h(\zz) \right) \hat{P}_n (\d y, \d \zz).
\end{equation*}
By our $W_2$ outer bound, i.e., Theorem~\ref{thm:W2_outer_bd}, $\{ \hat{P}_n \}$ is a tight sequence of random elements in $\cuP (\R^{m+1})$ equipped with the bounded Lipschitz distance that metrizes weak convergence, since the $W_2$ distance always dominates the bounded Lipschitz distance. In fact, we have $d_{\rm BL} \le W_1 \le W_2$ according to the dual representation of $W_1$. Therefore, $\{ \hat{P}_n \}$ has a subsequence $\{ \hat{P}_{n_k} \}$ which weakly converges to some random probability measure $ \tilde{P} \in \cuP (\{ \pm 1 \} \times \R^m)$ in distribution. By definition, $ \tilde{P} \in \widetilde{\cuF}_{m, \alpha}^{\varphi}$. We then deduce that
\begin{align*}
	\plimsup_{n,d\to\infty} \hR_n^{\star}(\XX,\yy) \ge & \plimsup_{k \to \infty} \inf_{h \in \cH_m} \int_{\{ \pm 1 \} \times \R^m} L \left( y, h(\zz) \right) \hat{P}_{n_k} (\d y, \d \zz) \\
	\stackrel{(i)}{\ge} & \limsup_{k \to \infty} \E_{\hat{P}_{n_k}} \left[ \inf_{h \in \cH_m} \int_{\{ \pm 1 \} \times \R^m} L \left( y, h(\zz) \right) \hat{P}_{n_k} (\d y, \d \zz) \right] \\
	\stackrel{(ii)}{\ge} & \E_{\tilde{P}} \left[ \inf_{h\in \cH_m} \int_{\{ \pm 1 \} \times \R^m} L(y,h(\zz))\, \tilde{P} (\d y,\d \zz) \right] \\
	\ge & \inf_{ \tilde{P} \in \widetilde{\cuF}^{\varphi}_{m,\alpha}}
	\E_{\tilde{P}} \left[ \inf_{h\in \cH_m} \int_{\{ \pm 1 \} \times \R^m} L(y,h(\zz))\, \tilde{P} (\d y,\d \zz) \right],
\end{align*}
where $(i)$ follows from the boundedness of $L$, and $(ii)$ follows from Fatou's lemma and Skorokhod's representation theorem for probability measures on Polish space (note that $(\cuP(\{ \pm 1 \} \times \R^m), \ d_{\rm BL})$ is a Polish space). This completes the proof of Eq.~\eqref{eq:ERM_inter}.

Now it suffices to show that
\begin{equation*}
	\inf_{ \tilde{P} \in \widetilde{\cuF}^{\varphi}_{m,\alpha}}
	\E_{\tilde{P}} \left[ \inf_{h\in \cH_m} \int_{\{ \pm 1 \} \times \R^m} L(y,h(\zz))\, \tilde{P} (\d y,\d \zz) \right] \ge \inf_{h\in \cH_m} 
	\inf_{P\in\cuF^{\varphi}_{m,\alpha}}
	\int_{\{\pm 1\} \times \R^m} L(y,h(\zz))\, P(\d y,\d \zz).
\end{equation*}
Fix $ \tilde{P} \in \widetilde{\cuF}^{\varphi}_{m,\alpha}$ and $P \in \supp(\tilde{P})$, we know that for any $\veps > 0$, $\P ( \tilde{P} \in \mathsf{B}_{\veps} (P) ) > 0$, where we denote
\begin{equation*}
	\mathsf{B}_{\veps} (P) = \left\{ Q \in \cuP(\R^m): d_{\rm BL} (Q, P) < \veps \right\}.
\end{equation*}
By definition, there is a sequence of random probability measures $\{ \hat{P}_{n, \WW} \}_{n \ge 1}$ such that $\hat{P}_{n, \WW} \stackrel{w}{\Rightarrow} \tilde{P}$ in distribution. According to the Portmanteau lemma, we deduce that
\begin{equation*}
	\liminf_{n \to \infty} \P \left( \inf_{\WW \in O(d, m)} d_{\rm BL} \left( \hat{P}_{n, \WW}, P \right) < \veps \right) \ge \liminf_{n \to \infty} \P \left( d_{\rm BL} \left( \hat{P}_{n, \WW}, P \right) < \veps \right) \ge \P \left( d_{\rm BL} (\tilde{P}, P) < \veps \right) > 0,
\end{equation*}
since $\mathsf{B}_{\veps} (P)$ is an open set. Using the same idea as in the proof of Theorem~\ref{thm:inner_bd_1D}, i.e., sharp concentration of $\inf_{\WW \in O(d, m)} d_{\rm BL} ( \hat{P}_{n, \WW}, P )$, one can conclude that for any $\veps > 0$,
\begin{equation*}
	\liminf_{n \to \infty} \P \left( \inf_{\WW \in O(d, m)} d_{\rm BL} \left( \hat{P}_{n, \WW}, P \right) < \veps \right) = 1,
\end{equation*}
which further implies that
\begin{equation*}
	\inf_{\WW \in O(d, m)} d_{\rm BL} \left( \hat{P}_{n, \WW}, P \right) \to 0 \ \text{in probability}.
\end{equation*}
Therefore, $P$ is $(\alpha, m)$-feasible. As a consequence, we have for all $\tilde{P} \in \widetilde{\cuF}^{\varphi}_{m,\alpha}$, $\supp(\tilde{P}) \subset \cuF^{\varphi}_{m,\alpha}$, thus leading to the following estimates:
\begin{align*}
	& \inf_{ \tilde{P} \in \widetilde{\cuF}^{\varphi}_{m,\alpha}}
	\E_{\tilde{P}} \left[ \inf_{h\in \cH_m} \int_{\{ \pm 1 \} \times \R^m} L(y,h(\zz))\, \tilde{P} (\d y,\d \zz) \right] \\
	\ge & \inf_{ \tilde{P} \in \widetilde{\cuF}^{\varphi}_{m,\alpha}}
	\E_{\tilde{P}} \left[ \bone_{ \tilde{P} \in \supp(\tilde{P})} \cdot \inf_{h\in \cH_m} \int_{\{ \pm 1 \} \times \R^m} L(y,h(\zz))\, \tilde{P} (\d y,\d \zz) \right] \\
	\ge & \inf_{ \tilde{P} \in \widetilde{\cuF}^{\varphi}_{m,\alpha}} \inf_{P \in \supp(\tilde{P})} \inf_{h\in \cH_m} \int_{\{ \pm 1 \} \times \R^m} L(y,h(\zz))\, P (\d y,\d \zz) \\
	\ge & \inf_{P \in \cuF^{\varphi}_{m,\alpha}} \inf_{h\in \cH_m} \int_{\{ \pm 1 \} \times \R^m} L(y,h(\zz))\, P (\d y,\d \zz) \\
	\ge & \inf_{h\in \cH_m} 
	\inf_{P\in\cuF^{\varphi}_{m,\alpha}}
	\int_{\{\pm 1\} \times \R^m} L(y,h(\zz))\, P(\d y,\d \zz),
\end{align*}
as desired. Finally, we obtain that
\begin{equation*}
	\plimsup_{n,d\to\infty}\hR_n^{\star}(\XX,\yy) \ge \inf_{h\in \cH_m} 
	\inf_{P\in\cuF^{\varphi}_{m,\alpha}}
	\int_{\{\pm 1\} \times \R^m} L(y,h(\zz))\, P(\d y,\d \zz).
\end{equation*}
Combining this with the upper bound completes the proof of Proposition~\ref{prop:ERM_asymptotics}.

\subsection{Proofs of Lemma~\ref{lem:key_triple} and Theorem~\ref{thm:W2_outer_bd}}
\begin{proof}[\bf Proof of Lemma~\ref{lem:key_triple}]
Similarly as in the proof of Theorem~\ref{thm:W2_outer_bd}, we can write
\begin{equation*}
	\hat{P}_{n, \ww} = \frac{1}{n} \sum_{i = 1}^{n} \delta_{\left( y_i, \left\langle \ww_{\parallel}, \bgg_i \right\rangle + \sqrt{1 - \norm{\ww_{\parallel}}_2^2} \left\langle \ww_{\perp}, \zz_i \right\rangle \right) }, \ P_{\ww} = \Law \left( Y, \ww_{\parallel}^\top G + \sqrt{1 - \norm{\ww_{\parallel}}_2^2} Z \right).
\end{equation*}
By assumption, $\ww \in \S^{d - 1}$ is independent of $(\XX, \yy)$. Therefore,
\begin{equation*}
	\hat{P}_{n, \ww} \stackrel{d}{=} \frac{1}{n} \sum_{i = 1}^{n} \delta_{\left( y_i, \ww_{\parallel}^\top \bgg_i + \sqrt{1 - \norm{\ww_{\parallel}}_2^2} z_i \right) }, \ (y_i, \bgg_i, z_i) \stackrel{\iid}{\sim} \Law(Y, G, Z),
\end{equation*}
where $\ww_{\parallel} \in \R^k$ is independent of $(\XX, \yy)$ as well. The desired result then follows by applying Glivenko-Cantelli Theorem.
\end{proof}

\begin{proof}[\bf Proof of Theorem~\ref{thm:W2_outer_bd}]
According to the orthogonal invariance property of isotropic normal distribution, we may assume without loss of generality that
\begin{equation*}
	\VV = [\id_k, \bzero_{k \times (d - k)}]^\top, \ \XX = (\bGG, \ZZ), \ (\yy, \bGG) \perp \ZZ,
\end{equation*}
where $\yy = (y_1, \cdots, y_n)^\top \in \R^n$, $\bGG = (\bgg_1, \cdots, \bgg_n)^\top \in \R^{n \times k}$, and $\ZZ = (\zz_1, \cdots, \zz_n)^\top \in \R^{n \times (d - k)}$. Moreover, for each $i \in [n]$ we have
\begin{equation*}
	\P \left( y_i = + 1 \vert \bgg_i \right) = \varphi (\bgg_i) = 1 - \P \left( y_i = - 1 \vert \bgg_i \right).
\end{equation*}
From now on we only assume that $y_i$ is sub-Gaussian and $(y_i, \bgg_i) \perp \zz_i$, which is all we need for this proof as pointed out in Remark~\ref{rem:sub_gauss_resp}. We further denote $\VV_{\perp} = [\ee_{k + 1}, \cdots, \ee_d]$, then we know that $[\VV, \VV_{\perp}] = \id_{d}$. Hence, for any $\ww \in \S^{d - 1}$, it follows that
\begin{equation*}
	\langle \xx_i, \ww \rangle = (\bgg_i^\top, \zz_i^\top) [\VV, \VV_{\perp}]^\top \ww = \ww^\top \VV \bgg_i + \ww^\top \VV_{\perp} \zz_i.
\end{equation*}
For simplicity, we denote $\ww_{\parallel} = \VV^\top \ww$ and $\ww_{\perp} = \VV_{\perp}^{\top} \ww / \norm{\VV_{\perp}^{\top} \ww}_2 = \VV_{\perp}^{\top} \ww / \sqrt{1 - \norm{\VV^\top \ww}_2^2}$, then we have the following decomposition:
\begin{equation*}
	\langle \xx_i, \ww \rangle = \left\langle \ww_{\parallel}, \bgg_i \right\rangle + \sqrt{1 - \norm{\ww_{\parallel}}_2^2} \cdot \left\langle \ww_{\perp}, \zz_i \right\rangle, \ \norm{\ww_{\parallel}}_2 \le 1, \ \ww_{\perp} \in \S^{d - k - 1}.
\end{equation*}
Now let us consider the random variable:
\begin{align*}
	\xi_{n, \veps, \eta} = & \min_{\norm{\ww}_2 = 1} \left\{ \frac{1}{\sqrt{\alpha}} \sqrt{1 - \norm{\VV^\top \ww}_2^2} + \veps - W_2^{(\eta)} \left( \hat{P}_{n, \ww}, P_{\ww} \right) \right\} \\
	= & \min_{\norm{\ww_{\parallel}}_2 \le 1, \norm{\ww_{\perp}}_2 = 1} \left\{ \frac{\sqrt{1 - \norm{\ww_{\parallel}}_2^2}}{\sqrt{\alpha}} + \veps - W_2^{(\eta)} \left( \frac{1}{n} \sum_{i = 1}^{n} \delta_{\left( y_i, \left\langle \ww_{\parallel}, \bgg_i \right\rangle + \sqrt{1 - \norm{\ww_{\parallel}}_2^2} \left\langle \ww_{\perp}, \zz_i \right\rangle \right) }, P_{\ww} \right) \right\} \\
	= & \min_{\norm{\ww_{\parallel}}_2 \le 1, \norm{\ww_{\perp}}_2 = 1, \uu \in \R^n} \max_{\blambda \in \R^n} \left\{ \frac{\sqrt{1 - \norm{\ww_{\parallel}}_2^2}}{\sqrt{\alpha}} + \veps - W_2^{(\eta)} \left( \frac{1}{n} \sum_{i = 1}^{n} \delta_{\left( y_i, \left\langle \ww_{\parallel}, \bgg_i \right\rangle + \sqrt{1 - \norm{\ww_{\parallel}}_2^2} u_i \right) }, P_{\ww} \right) \right. \\
	& \quad \quad \quad \quad \quad \quad \quad \quad \left. + \frac{1}{n} \left\langle \blambda, \uu - \ZZ \ww_{\perp} \right\rangle \right\}, 
\end{align*}	
where in the last step we use Lagrange multiplier method, and the vector $\uu = (u_1, \cdots, u_n)^\top$ represents the dual variable. Recall that $P_{\ww} = \Law ( Y, \langle \ww_{\parallel}, G \rangle + \sqrt{1 - \norm{\ww_{\parallel}}_2^2} Z )$, and $(Y, G, Z)$ is distributed as Eq.~\eqref{eq:key_triple}. We first show that for any $\veps, \eta > 0$, there exist positive constants $C_0 (\veps, \eta)$ and $C_1 (\veps, \eta)$ such that
\begin{equation}\label{eq:xi_lower_bd}
	\P \left( \xi_{n, \veps, \eta} \le 0 \right) \le C_0 (\veps, \eta) \exp \left( - C_1 (\veps, \eta) n^{1 - \veps} \right).
\end{equation}

Using standard norm bound on Gaussian random matrices (see e.g., Corollary 7.7.3 in \cite{vershynin2018high}), we know that there exists a constant $C_0 > 0$, such that $\P ( \norm{\ZZ}_{\mathrm{op}} \le C_0 \sqrt{n} ) \ge 1 - 2 \exp ( - C_0 n)$. On this event, for any $\ww_{\perp} \in \S^{d - k - 1}$ we have $\norm{\ZZ \ww_{\perp}}_2 \le C_0 \sqrt{n}$, which in turn implies that
\begin{align*}
	\xi_{n, \veps, \eta} = & \min_{\norm{\ww_{\parallel}}_2 \le 1, \norm{\ww_{\perp}}_2 = 1, \norm{\uu}_2 \le C_0 \sqrt{n}} \max_{\blambda \in \R^n} \left\{ \frac{\sqrt{1 - \norm{\ww_{\parallel}}_2^2}}{\sqrt{\alpha}} + \veps \right. \\
	& \left. - W_2^{(\eta)} \left( \frac{1}{n} \sum_{i = 1}^{n} \delta_{\left( y_i, \left\langle \ww_{\parallel}, \bgg_i \right\rangle + \sqrt{1 - \norm{\ww_{\parallel}}_2^2} u_i \right) }, P_{\ww} \right) + \frac{1}{n} \left\langle \blambda, \uu - \ZZ \ww_{\perp} \right\rangle \right\} \\
	\ge & \min_{\norm{\ww_{\parallel}}_2 \le 1, \norm{\ww_{\perp}}_2 = 1, \norm{\uu}_2 \le C_0 \sqrt{n}} \max_{\norm{\blambda}_2 \le C \sqrt{n}} \left\{ \frac{\sqrt{1 - \norm{\ww_{\parallel}}_2^2}}{\sqrt{\alpha}} + \veps \right. \\
	& \left. - W_2^{(\eta)} \left( \frac{1}{n} \sum_{i = 1}^{n} \delta_{\left( y_i, \left\langle \ww_{\parallel}, \bgg_i \right\rangle + \sqrt{1 - \norm{\ww_{\parallel}}_2^2} u_i \right) }, P_{\ww} \right) + \frac{1}{n} \left\langle \blambda, \uu - \ZZ \ww_{\perp} \right\rangle \right\} \\
	:= & \xi_{n, \veps, \eta, C},
\end{align*}
where $C > 0$ is to be determined. The first equality holds since for $\norm{\uu}_2 > C_0 \sqrt{n}$, the inner maximum becomes $+\infty$. Therefore, we obtain the following inequality:
\begin{equation*}
	\P \left( \xi_{n, \veps, \eta} \le 0 \right) \le \P \left( \norm{\ZZ}_{\mathrm{op}} > C_0 \sqrt{n} \right) + \P \left( \xi_{n, \veps, \eta, C} \le 0 \right).
\end{equation*}
For future convenience, we denote the domain of the outer minimum by $\cD_{\ww, \uu}$, i.e.,
\begin{equation*}
	\cD_{\ww, \uu} = \left\{ (\ww_{\parallel}, \ww_{\perp}, \uu) \in \R^k \times \R^{d - k} \times \R^{n}: \norm{\ww_{\parallel}}_2 \le 1, \norm{\ww_{\perp}}_2 = 1, \norm{\uu}_2 \le C_0 \sqrt{n} \right\}.
\end{equation*}

Next we upper bound $\xi_{n, \veps, \eta, C}$ with Lemma~\ref{lem:gordon}. Let $\vv \sim \sN (\bzero, \id_{d - k})$ and $\hh \sim \sN (\bzero, \id_n)$ be independent standard normal vectors, and further independent of $(\yy, \bGG, \ZZ)$. Then we know that conditioning on $(\yy, \bGG)$, $\ZZ = (Z_{ij})_{i \in [n], j \in [d - k]} \sim_{\iid} \sN (0, 1)$ is independent of $(\vv, \hh)$. Define the following linearized version of $\xi_{n, \veps, \eta, C}$:
\begin{align*}
	\xi_{n, \veps, \eta, C}^{(1)} = & \min_{\cD_{\ww, \uu}} \max_{\norm{\blambda}_2 \le C \sqrt{n}} \left\{ \frac{\sqrt{1 - \norm{\ww_{\parallel}}_2^2}}{\sqrt{\alpha}} + \veps - W_2^{(\eta)} \left( \frac{1}{n} \sum_{i = 1}^{n} \delta_{\left( y_i, \left\langle \ww_{\parallel}, \bgg_i \right\rangle + \sqrt{1 - \norm{\ww_{\parallel}}_2^2} u_i \right) }, P_{\ww} \right) \right. \\
	& \left. + \frac{1}{n} \left\langle \blambda, \uu - \norm{\ww_{\perp}}_2 \hh \right\rangle + \frac{1}{n} \norm{\blambda}_2 \langle \ww_{\perp}, \vv \rangle \right\}.
\end{align*}
Note that with high probability, the following holds:
\begin{equation*}
	W_2 \left( \frac{1}{n} \sum_{i = 1}^{n} \delta_{y_i}, \Law(Y) \right) < \eta \implies W_2^{(\eta)} \left( \frac{1}{n} \sum_{i = 1}^{n} \delta_{\left( y_i, \left\langle \ww_{\parallel}, \bgg_i \right\rangle + \sqrt{1 - \norm{\ww_{\parallel}}_2^2} u_i \right) }, P_{\ww} \right) < \infty,
\end{equation*}
which further implies that the objective function in the minimax problem that defines $\xi_{n, \veps, \eta, C}^{(1)}$ is continuous (Lemma~\ref{lem:prop_cons_W2} (ii)). Therefore, according to Lemma~\ref{lem:gordon}, we have for all $t \in \R$,
\begin{equation*}
	\P \left( \xi_{n, \veps, \eta, C} \le t \vert \yy, \bGG \right) \le 2 \P \left( \xi_{n, \veps, \eta, C}^{(1)} \le t \vert \yy, \bGG \right).
\end{equation*}
Choose $t = 0$ and integrate over $(\yy, \bGG)$ the above inequality, we obtain that $\P ( \xi_{n, \veps, \eta, C} \le 0 ) \le 2 \P ( \xi_{n, \veps, \eta, C}^{(1)} \le 0 )$. Moreover, straightforward calculation reveals that
\begin{align*}
	\xi_{n, \veps, \eta, C}^{(1)} = & \min_{\norm{\ww_{\parallel}}_2 \le 1, \norm{\uu}_2 \le C_0 \sqrt{n}} \max_{\norm{\blambda}_2 \le C \sqrt{n}} \left\{ \frac{\sqrt{1 - \norm{\ww_{\parallel}}_2^2}}{\sqrt{\alpha}} + \veps \right. \\
	& \left. - W_2^{(\eta)} \left( \frac{1}{n} \sum_{i = 1}^{n} \delta_{\left( y_i, \left\langle \ww_{\parallel}, \bgg_i \right\rangle + \sqrt{1 - \norm{\ww_{\parallel}}_2^2} u_i \right) }, P_{\ww} \right) + \frac{1}{n} \left\langle \blambda, \uu - \hh \right\rangle - \frac{1}{n} \norm{\blambda}_2 \norm{\vv}_2 \right\}  \\
	= & \min_{\norm{\ww_{\parallel}}_2 \le 1, \norm{\uu}_2 \le C_0 \sqrt{n}} \max_{\norm{\blambda}_2 \le C \sqrt{n}} \left\{ \frac{\sqrt{1 - \norm{\ww_{\parallel}}_2^2}}{\sqrt{\alpha}} + \veps \right. \\
	& \left. - W_2^{(\eta)} \left( \frac{1}{n} \sum_{i = 1}^{n} \delta_{\left( y_i, \left\langle \ww_{\parallel}, \bgg_i \right\rangle + \sqrt{1 - \norm{\ww_{\parallel}}_2^2} u_i \right) }, P_{\ww} \right) + \frac{\norm{\blambda}_2}{n} \left( \norm{\uu - \hh}_2 - \norm{\vv}_2 \right) \right\} \\
	= & \min_{\norm{\ww_{\parallel}}_2 \le 1, \norm{\uu}_2 \le C_0 \sqrt{n}} \max_{0 \le \gamma \le C} \left\{ \frac{\sqrt{1 - \norm{\ww_{\parallel}}_2^2}}{\sqrt{\alpha}} + \veps \right. \\
	& \left. - W_2^{(\eta)} \left( \frac{1}{n} \sum_{i = 1}^{n} \delta_{\left( y_i, \left\langle \ww_{\parallel}, \bgg_i \right\rangle + \sqrt{1 - \norm{\ww_{\parallel}}_2^2} u_i \right) }, P_{\ww} \right) + \frac{\gamma}{\sqrt{n}} \left( \norm{\uu - \hh}_2 - \norm{\vv}_2 \right) \right\} \\
	= & \min_{\norm{\ww_{\parallel}}_2 \le 1, \norm{\uu}_2 \le C_0 \sqrt{n}} \left\{ \frac{\sqrt{1 - \norm{\ww_{\parallel}}_2^2}}{\sqrt{\alpha}} + \veps - W_2^{(\eta)} \left( \frac{1}{n} \sum_{i = 1}^{n} \delta_{\left( y_i, \left\langle \ww_{\parallel}, \bgg_i \right\rangle + \sqrt{1 - \norm{\ww_{\parallel}}_2^2} u_i \right) }, P_{\ww} \right) \right. \\
	& \left. + \frac{C}{\sqrt{n}} \left( \norm{\uu - \hh}_2 - \norm{\vv}_2 \right)_{+} \right\}.
\end{align*}
Now we aim to estimate $\xi_{n, \veps, \eta, C}^{(1)}$. By Lemma~\ref{lem:prop_cons_W2} (i), it follows that for all $(\ww_{\parallel}, \uu)$ such that $\norm{\ww_{\parallel}}_2 \le 1$ and $\norm{\uu}_2 \le C_0 \sqrt{n}$,
\begin{align*}
	& W_2^{(\eta)} \left( \frac{1}{n} \sum_{i = 1}^{n} \delta_{\left( y_i, \left\langle \ww_{\parallel}, \bgg_i \right\rangle + \sqrt{1 - \norm{\ww_{\parallel}}_2^2} u_i \right) }, P_{\ww} \right) \le W_2^{(\eta)} \left( \frac{1}{n} \sum_{i = 1}^{n} \delta_{\left( y_i, \left\langle \ww_{\parallel}, \bgg_i \right\rangle + \sqrt{1 - \norm{\ww_{\parallel}}_2^2} h_i \right) }, P_{\ww} \right) \\
	& + W_2^{(0)} \left( \frac{1}{n} \sum_{i = 1}^{n} \delta_{\left( y_i, \left\langle \ww_{\parallel}, \bgg_i \right\rangle + \sqrt{1 - \norm{\ww_{\parallel}}_2^2} h_i \right) }, \frac{1}{n} \sum_{i = 1}^{n} \delta_{\left( y_i, \left\langle \ww_{\parallel}, \bgg_i \right\rangle + \sqrt{1 - \norm{\ww_{\parallel}}_2^2} u_i \right) } \right) \\
	\le & W_2^{(\eta)} \left( \frac{1}{n} \sum_{i = 1}^{n} \delta_{\left( y_i, \left\langle \ww_{\parallel}, \bgg_i \right\rangle + \sqrt{1 - \norm{\ww_{\parallel}}_2^2} h_i \right) }, P_{\ww} \right) + \frac{\norm{\uu - \hh}_2}{\sqrt{n}} \sqrt{1 - \norm{\ww_{\parallel}}_2^2}.
\end{align*}
Note that $\{ (y_i, \bgg_i, h_i) \}_{i \in [n]} \sim_{\iid} \Law (Y, G, Z)$. According to Lemma~\ref{lem:W2_concentration}, there exist positive constants $C_1 := C_1 (\veps, \eta)$ and $C_2 := C_2 (\veps, \eta)$, such that
\begin{equation*}
	\P \left( W_2 \left( \frac{1}{n} \sum_{i = 1}^{n} \delta_{(y_i, \bgg_i, h_i)}, \Law (Y, G, Z) \right) > \min \left( \frac{\veps}{2}, \eta \right) \right) < C_1 \exp \left( - C_2 n^{1 - \veps} \right).
\end{equation*}
Denote $(1 / n) \sum_{i = 1}^{n} \delta_{(y_i, \bgg_i, h_i)} = \Law(Y_n, G_n, Z_n)$, then with probability at least $1 - C_1 \exp (- C_2 n^{1 - \veps})$, we can find a coupling $(Y_n, G_n, Z_n, Y, G, Z)$ satisfying
\begin{equation*}
	\E \left[ (Y_n - Y)^2 + \norm{G_n - G}_2^2 + (Z_n - Z)^2 \right] < \min \left( \frac{\veps^2}{4}, \eta^2 \right),
\end{equation*}
which further implies that $\E [(Y_n - Y)^2]^{1 / 2} < \eta$, and that
\begin{align*}
	& \E \left[ (Y_n - Y)^2 + \left( \left\langle \ww_{\parallel}, G_n \right\rangle + \sqrt{1 - \norm{\ww_{\parallel}}_2^2} \cdot Z_n - \left\langle \ww_{\parallel}, G \right\rangle -  \sqrt{1 - \norm{\ww_{\parallel}}_2^2} \cdot Z \right)^2 \right] \\
	= & \E \left[ (Y_n - Y)^2 + \left( \left\langle \ww_{\parallel}, G_n - G \right\rangle + \sqrt{1 - \norm{\ww_{\parallel}}_2^2} \cdot ( Z_n - Z ) \right)^2 \right] \\
	\stackrel{(i)}{\le} & \E \left[ (Y_n - Y)^2 + \left( \norm{\ww_{\parallel}}_2^2 + 1 - \norm{\ww_{\parallel}}_2^2 \right) \left( \norm{G_n - G}_2^2 + (Z_n - Z)^2 \right) \right] < \frac{\veps^2}{4},
\end{align*}
where $(i)$ follows from Cauchy-Schwarz inequality. Therefore, we conclude that with probability at least $1 - C_1 \exp(- C_2 n^{1 - \veps})$,
\begin{equation*}
	W_2^{(\eta)} \left( \frac{1}{n} \sum_{i = 1}^{n} \delta_{\left( y_i, \left\langle \ww_{\parallel}, \bgg_i \right\rangle + \sqrt{1 - \norm{\ww_{\parallel}}_2^2} h_i \right) }, P_{\ww} \right) < \frac{\veps}{2},
\end{equation*}
thus leading to
\begin{equation*}
	W_2^{(\eta)} \left( \frac{1}{n} \sum_{i = 1}^{n} \delta_{\left( y_i, \left\langle \ww_{\parallel}, \bgg_i \right\rangle + \sqrt{1 - \norm{\ww_{\parallel}}_2^2} u_i \right) }, P_{\ww} \right) < \frac{\veps}{2} + \frac{\norm{\uu - \hh}_2}{\sqrt{n}} \sqrt{1 - \norm{\ww_{\parallel}}_2^2}.
\end{equation*}
Using standard large deviation bounds (see e.g. \cite{durrett2019probability}), we know that there exists $C_3 = C_3 (\veps) > 0$ satisfying
\begin{equation*}
	\P \left( \left\vert \frac{1}{\sqrt{n}} \norm{\vv}_2 - \frac{1}{\sqrt{\alpha}} \right\vert < \frac{\veps}{2 C} \right) \ge 1 - \exp (- C_3 n).
\end{equation*}
Combing all these estimates, we finally obtain that with probability no less than $1 - C_1 \exp (- C_2 n^{1 - \veps})$, the following happens: For any $C \ge 1$, $\norm{\ww_{\parallel}}_2 \le 1$, and $\norm{\uu}_2 \le C_0 \sqrt{n}$, we have
\begin{align*}
	& \frac{\sqrt{1 - \norm{\ww_{\parallel}}_2^2}}{\sqrt{\alpha}} + \veps - W_2^{(\eta)} \left( \frac{1}{n} \sum_{i = 1}^{n} \delta_{\left( y_i, \left\langle \ww_{\parallel}, \bgg_i \right\rangle + \sqrt{1 - \norm{\ww_{\parallel}}_2^2} u_i \right) }, P_{\ww} \right) + \frac{C}{\sqrt{n}} \left( \norm{\uu - \hh}_2 - \norm{\vv}_2 \right)_{+} \\
	> & \frac{\sqrt{1 - \norm{\ww_{\parallel}}_2^2}}{\sqrt{\alpha}} + \veps - \frac{\veps}{2} - \frac{\norm{\uu - \hh}_2}{\sqrt{n}} \sqrt{1 - \norm{\ww_{\parallel}}_2^2} + C \left( \frac{\norm{\uu - \hh}_2}{\sqrt{n}} - \frac{1}{\sqrt{\alpha}} \right)_{+} - C \cdot \frac{\veps}{2 C} \\
	= & C \left( \frac{\norm{\uu - \hh}_2}{\sqrt{n}} - \frac{1}{\sqrt{\alpha}} \right)_{+} - \sqrt{1 - \norm{\ww_{\parallel}}_2^2} \left( \frac{\norm{\uu - \hh}_2}{\sqrt{n}} - \frac{1}{\sqrt{\alpha}} \right) \\
	\ge & \left( C - \sqrt{1 - \norm{\ww_{\parallel}}_2^2} \right) \left( \frac{\norm{\uu - \hh}_2}{\sqrt{n}} - \frac{1}{\sqrt{\alpha}} \right)_{+} \ge 0,
\end{align*}
which leads to $\xi_{n, \veps, \eta, C}^{(1)} > 0$. Therefore,
\begin{equation*}
	\P \left( \xi_{n, \veps, \eta, C} \le 0 \right) \le 2 \P \left( \xi_{n, \veps, \eta, C}^{(1)} \le 0 \right) \le 2 C_1 \exp (- C_2 n^{1 - \veps}).
\end{equation*}
Finally, we get that
\begin{align*}
	\P \left( \xi_{n, \veps, \eta} \le 0 \right) \le & \P \left( \norm{\ZZ}_{\mathrm{op}} > C_0 \sqrt{n} \right) + 2 \P \left( \xi_{n, \veps, \eta, C}^{(1)} \le 0 \right) \\
	\le & 2 \exp (- C_0 n) + 2 C_1 \exp (- C_2 n^{1 - \veps}),
\end{align*}
where $C_0, C_1, C_2$ only depends on $\veps$ and $\eta$. This proves Eq.~\eqref{eq:xi_lower_bd}.

Now we are in position to show Eq.~\eqref{eq:W2_as_conv}. Fix $\eta > 0$, for any $\veps > 0$, we have
\begin{align*}
	& \sum_{n = 1}^{\infty} \P \left( \max_{\norm{\ww}_2 = 1} \left( W_2^{(\eta)} \left( \hat{P}_{n, \ww}, P_{\ww} \right) - \frac{1}{\sqrt{\alpha}} \sqrt{1 - \norm{\VV^\top \ww}_2^2} \right)_{+} \ge \veps \right) \\
	\le & \sum_{n=1}^{\infty} \P \left( \xi_{n, \veps, \eta} \le 0 \right) \le \sum_{n = 1}^{\infty} C_0 (\veps, \eta) \exp \left( - C_1 (\veps, \eta) n^{1 - \veps} \right) < \infty.
\end{align*}
Therefore, by Borel-Cantelli Lemma we deduce that almost surely,
\begin{equation*}
	\limsup_{n \to \infty} \max_{\norm{\ww}_2 = 1} \left( W_2^{(\eta)} \left( \hat{P}_{n, \ww}, P_{\ww} \right) - \frac{1}{\sqrt{\alpha}} \sqrt{1 - \norm{\VV^\top \ww}_2^2} \right)_{+} \le \veps.
\end{equation*}
Since this is true for all $\veps > 0$, Eq.~\eqref{eq:W2_as_conv} follows naturally. This concludes the proof.
\end{proof}

\subsection{Proof of Theorem~\ref{thm:NN_upper_bd}}
With a slight abuse of notation, let us recast $\ww_j$ as $B_j \ww_j$ where $B_j \in [0, B]$ and $\norm{\ww_j}_2 = 1$. Denote $\hat{\WW} = (\ww_1, \cdots, \ww_m)$ and define the following empirical distribution:
\begin{equation*}
	\hat{P}_n := \hat{P}_{n, \hat{\WW}} = \frac{1}{n} \sum_{i = 1}^{n} \delta_{\left( y_i, \xx_i^\top \hat{\WW} \right)}.
\end{equation*}
Then $\hat{P}_n$ is a probability measure on $\{ \pm 1 \} \times \R^m$. Define the random vector $(y, \uu) = (y, u_1, \cdots, u_m)$ on the same probability space such that
\begin{equation*}
	(y, \uu) \sim \hat{P}_n \implies \P \left( (y, \uu) = \left( y_i, \xx_i^\top \hat{\WW} \right) \right) = \frac{1}{n}, \ \forall i \in [n].
\end{equation*}
Hence, we may write $\hat{P}_n = \Law (y, \uu) = \Law (y, u_1, \cdots, u_m)$. Now assume that with probability bounded away from $0$, a $\kappa$-margin interpolator exists, i.e., $\hat{R}_n^{\star} (\XX, \yy) = 0$ and
\begin{equation*}
	\exists \hat{\WW} = (\ww_1, \cdots, \ww_m), \ \mbox{s.t.} \ \frac{y_i}{\sqrt{m}} \sum_{j=1}^{m} a_j \sigma \left( B_j \langle \ww_j, \xx_i \rangle \right) \ge \kappa, \ \forall i \in [n],
\end{equation*}
then almost surely under $\hat{P}_n$, we have
\begin{equation}\label{eq:kappa_margin_dist}
	\frac{y}{\sqrt{m}} \sum_{j = 1}^{m} a_j \sigma \left( B_j u_j \right) \ge \kappa.
\end{equation}
Now let us denote for $j = 1, \cdots, m$, $\hat{P}_{n, j} = \Law (y, u_j) = (1/n) \sum_{i=1}^{n} \delta_{(y_i, \langle \ww_j, \xx_i \rangle)}$, and $P = \Unif \{ \pm 1 \} \otimes \sN (0, 1)$. By our assumption, we can assume that $n / d \to \alpha \ge (2 L^2 B^4 / \kappa^2) \cdot m$. (Otherwise, one may choose such a subsequence.) Then for any fixed $\veps > 0$ and $\eta > 0$, according to Theorem~\ref{thm:W2_outer_bd}, with high probability we have
\begin{equation*}
	W_2^{(\eta)} \left( \Law (y, u_j), P \right) = W_2^{(\eta)} \left( \hat{P}_{n, j}, P \right) \le \frac{1 + \veps}{\sqrt{\alpha}}, \ \forall 1 \le j \le m.
\end{equation*}
As a consequence, there exists a coupling $((y, u_j), (y', G))$ such that $(y', G) \sim P$, and that
\begin{equation}\label{eq:W2_outer_bd_1}
	\E \left[ (y - y')^2 \right] \le \eta^2, \ \E \left[ (y - y')^2 \right] + \E \left[ (u_j - G)^2 \right] \le \frac{(1 + \veps)^2}{\alpha},
\end{equation}
thus leading to (notice that $\E [y' \sigma( B_j G)] = 0$ by independence)
\begin{align*}
	& \left\vert \E \left[ y \sigma \left( B_j u_j \right) \right] \right\vert = \left\vert \E \left[ y \sigma \left( B_j u_j \right) \right] - \E \left[ y' \sigma \left( B_j G \right) \right] \right\vert \\
	\le & \left\vert \E \left[ y \left( \sigma \left( B_j u_j \right) - \sigma \left( B_j G \right) \right) \right] \right\vert + \left\vert \E \left[ (y' - y) \sigma \left( B_j G \right) \right] \right\vert \\
	\stackrel{(i)}{\le} & L B_j \cdot \E \left[ \left\vert u_j - G \right\vert \right] + \E \left[ (y - y')^2 \right]^{1/2} \cdot \E \left[ \sigma\left( B_j G \right)^2 \right]^{1/2} \\
	\le & L B_j \cdot \E \left[ \left( u_j - G \right)^2 \right]^{1/2} + \E \left[ (y - y')^2 \right]^{1/2} \cdot \max_{0 \le b \le B} \E \left[ \sigma \left( b G \right)^2 \right]^{1/2} \\
	\le & \frac{L B (1 + \veps)}{\sqrt{\alpha}} + \eta \cdot \max_{0 \le b \le B} \E \left[ \sigma \left( b G \right)^2 \right]^{1/2},
\end{align*}
where $(i)$ is due to the fact that $\sigma$ is $L$-Lipschitz and $y = \pm 1$, and the last inequality follows from Eq.~\eqref{eq:W2_outer_bd_1}. Now for any constant $C_{\rm ub} = C_{\rm ub} (\kappa, B, L)$, as long as $\alpha \ge C_{\rm ub} \cdot m$, we can choose $\veps$ and $\eta$ to be small enough so that
\begin{equation*}
	\vert \E \left[ y \sigma \left( B_j u_j \right) \right] \vert \le \frac{L B (1 + \veps)}{\sqrt{\alpha}} + \eta \cdot \max_{0 \le b \le B} \E \left[ \sigma \left( b G \right)^2 \right]^{1/2} \le \frac{L B}{\sqrt{C_{\rm ub}} \sqrt{m}}, \ \forall j \in [m].
\end{equation*}
Taking expectation on both sides of Eq.~\eqref{eq:kappa_margin_dist}, we finally obtain that with probability bounded away from $0$,
\begin{equation}\label{eq:interpolation_expc}
	\kappa \le \frac{1}{\sqrt{m}} \sum_{j = 1}^{m} \vert a_j \vert \cdot \left\vert \E \left[ y \sigma \left( B_j u_j \right) \right] \right\vert \le \frac{LB}{\sqrt{C_{\rm ub}}} \cdot \frac{1}{m} \sum_{j = 1}^{m} \vert a_j \vert \le \frac{LB^2}{\sqrt{C_{\rm ub}}},
\end{equation}
where the last inequality follows from the definition of $\cF_{\mathsf{NN}}^{m, B}$: $(1 / m) \sum_{j = 1}^{m} \vert a_j \vert \le B$. Therefore, if we choose
\begin{equation*}
	C_{\rm ub} = C_{\rm ub} (\kappa, B, L) = \frac{2 L^2 B^4}{\kappa^2} > \frac{L^2 B^4}{\kappa^2},
\end{equation*}
then Eq.~\eqref{eq:interpolation_expc} will result a contradiction. To conclude, as long as
\begin{equation*}
	\alpha \ge C_{\rm ub} (\kappa, B, L) \cdot m = \frac{2 L^2 B^4}{\kappa^2} \cdot m,
\end{equation*}
a $\kappa$-margin interpolator does not exist with high probability, namely $\hat{R}_n^* (\XX, \yy) > 0$. This completes the proof of Theorem~\ref{thm:NN_upper_bd}.

\subsection{Proofs of Lemma~\ref{lem:unique_F_kappa} and Theorem~\ref{thm:mm_margin_dist}}
\begin{proof}[\bf Proof of Lemma~\ref{lem:unique_F_kappa}]
Since $F_{\kappa}$ is continuous, its maximum must be achieved at some $\rho_* = \rho_* (\kappa) \in [-1, 1]$. We show that such $\rho_*$ is unique. Otherwise, assume that there exists $\rho_1 \neq \rho_2$ such that
\begin{equation*}
	F_{\kappa} (\rho_1) = F_{\kappa} (\rho_2) = \max_{\rho \in [-1, 1]} F_{\kappa} (\rho).
\end{equation*}
Denote this maximum by $F$, then by definition of $F_{\kappa}$ we get that
\begin{equation*}
	\E \left[ \left( \kappa - \rho_1 YG - \sqrt{1 - \rho_1^2} Z \right)_+^2 \right] = \frac{1 - \rho_1^2}{F}, \ \E \left[ \left( \kappa - \rho_2 YG - \sqrt{1 - \rho_2^2} Z \right)_+^2 \right] = \frac{1 - \rho_2^2}{F}.
\end{equation*}
Let us define a new function $G_{\kappa}: \mathsf{B}_{2} (1) \to \R$ by $G_{\kappa} (\rho, r) = \E [ ( \kappa - \rho YG - r Z )_+^2 ]^{1/2}$, where $\mathsf{B}_2 (1)$ is the closed unit disk in $\R^2$. Then, from Lemma 6.3 (a) in \cite{montanari2019generalization} we know that $G_{\kappa}$ is convex, thus leading to
\begin{align*}
	G_{\kappa} \left( \frac{\rho_1 + \rho_2}{2}, \frac{\sqrt{1 - \rho_1^2} + \sqrt{1 - \rho_2^2}}{2} \right) \le & \frac{1}{2} \left( G_{\kappa} \left( \rho_1, \sqrt{1 - \rho_1^2} \right) + G_{\kappa} \left( \rho_2, \sqrt{1 - \rho_2^2} \right) \right) \\
	= & \frac{1}{2 \sqrt{F}} \left( \sqrt{1 - \rho_1^2} + \sqrt{1 - \rho_2^2} \right).
\end{align*}
Denote $\rho = (\rho_1 + \rho_2) / 2$, $r = (\sqrt{1 - \rho_1^2} + \sqrt{1 - \rho_2^2}) / 2$. Since $\rho_1 \neq \rho_2$ and $\mathsf{B}_{2} (1)$ is strictly convex, we know that $\sqrt{\rho^2 + r^2} < 1$. We further denote
\begin{equation*}
	\rho' = \frac{\rho}{\sqrt{\rho^2 + r^2}}, \ r' = \frac{r}{\sqrt{\rho^2 + r^2}} \implies \rho'^2 + r'^2 = 1.
\end{equation*}
Then, it follows that
\begin{align*}
	\frac{1}{\sqrt{F}} = & \frac{1}{r} G_{\kappa} (\rho, r) = \frac{\sqrt{\rho^2 + r^2}}{r} \E \left[ \left( \frac{\kappa}{\sqrt{\rho^2 + r^2}} - \frac{\rho Y G}{\sqrt{\rho^2 + r^2}} - \frac{r Z}{\sqrt{\rho^2 + r^2}} \right)_+^2 \right]^{1 / 2} \\
	= & \frac{1}{r'} \E \left[ \left( \frac{\kappa}{\sqrt{\rho^2 + r^2}} - \rho' Y G - r' Z \right)_+^2 \right]^{1 / 2} > \frac{1}{r'} G_{\kappa} \left( \rho', r' \right),
\end{align*}
where the last inequality is due to the fact that $\sqrt{\rho^2 + r^2} < 1$. We finally obtain that
\begin{equation*}
	F_{\kappa} \left( \rho' \right) = \frac{r'^2}{G_{\kappa} \left( \rho', r' \right)^2} > F = \max_{\rho \in [-1, 1]} F_{\kappa} (\rho),
\end{equation*}
a contradiction. Therefore, the maximizer $\rho_*$ must be unique. 
\end{proof}

\begin{proof}[\bf Proof of Theorem~\ref{thm:mm_margin_dist}]
Without loss of generality, let us assume that $\btheta_* = \ee_1$ and denote $\hat{\btheta}^{\rm MM} = (\hat{\rho}, \hat{\ww}^\top)^\top$. For $i \in [n]$ one can write $\xx_i = (g_i, \zz_i)$, where $(y_i, g_i) \perp \zz_i$, and that
\begin{equation*}
	\P (y_i = + 1 \vert g_i) = \varphi (g_i) = 1 - \P (y_i = - 1 \vert g_i).
\end{equation*}
Then, Theorem~\ref{thm:W2_outer_bd} implies that, for any $\veps, \eta > 0$, with high probability we have
\begin{equation*}
	W_2^{(\eta)} \left( \frac{1}{n} \sum_{i = 1}^{n} \delta_{\left( y_i, \ \left\langle \hat{\btheta}^{\rm MM}, \xx_i \right\rangle \right)}, \Law \left( Y, \hat{\rho} G + \sqrt{1 - \hat{\rho}^2} Z \right) \right) \le \frac{\sqrt{1 - \hat{\rho}^2}}{\sqrt{\alpha}} + \veps.
\end{equation*}
Denote $(Y, \hat{\rho} G + \sqrt{1 - \hat{\rho}^2} Z ) = (Y, U)$, and
\begin{equation*}
	\frac{1}{n} \sum_{i = 1}^{n} \delta_{\left( y_i, \ \left\langle \hat{\btheta}^{\rm MM}, \xx_i \right\rangle \right)} = \Law (Y', U'),
\end{equation*}
then there exists a coupling $(Y, U, Y', U')$ satisfying
\begin{equation*}
	\E \left[ (Y - Y')^2 \right] \le \eta^2, \ \E \left[ (Y - Y')^2 \right] + \E \left[ (U - U')_2^2 \right] \le \left( \frac{\sqrt{1 - \hat{\rho}^2}}{\sqrt{\alpha}} + \veps \right)^2,
\end{equation*}
which leads to
\begin{align*}
	\E \left[ (Y U - Y' U')^2 \right] = & \E \left[ \left( (Y - Y') U + Y' (U -U') \right)^2 \right] \\
	= & \E \left[ (Y - Y')^2 U^2 \right] + 2 \E \left[ (Y - Y') Y' U (U -U') \right] + \E \left[ (U - U')^2 \right] \\
	\stackrel{(i)}{\le} & \E \left[ (Y - Y')^2 U^2 \right] + 2 \E \left[ (Y - Y')^2 U^2 \right]^{1/2} \E \left[ (U - U')^2 \right]^{1/2} + \E \left[ (U - U')^2 \right] \\
	\stackrel{(ii)}{\le} & \left( \E \left[ (Y - Y')^4 \right]^{1 / 4} \E \left[ U^4 \right]^{1/4} + \E \left[ (U - U')^2 \right]^{1/2} \right)^2 \\
	\stackrel{(iii)}{\le} & \left( \sqrt{2} \E \left[ (Y -Y')^2 \right]^{1/4} \E \left[ G^4 \right]^{1 / 4} + \E \left[ (U - U')^2 \right]^{1/2} \right)^2 \\
	\le & \left( 2 \sqrt{\eta} + \frac{\sqrt{1 - \hat{\rho}^2}}{\sqrt{\alpha}} + \veps \right)^2,
\end{align*}
where $(i)$ and $(ii)$ follow from Cauchy-Schwarz inequality, $(iii)$ is because of $\vert Y -Y' \vert \le 2$ and $U \sim \sN (0, 1)$. Hence, we can choose a sufficiently small $\eta$ ($\eta < \veps^2 / 4$ is enough) such that
\begin{equation*}
	\E \left[ (YU - Y'U')^2 \right]^{1 / 2} \le \frac{\sqrt{1 - \hat{\rho}^2}}{\sqrt{\alpha}} + 2 \veps,
\end{equation*}
which further implies that
\begin{align*}
	W_2 \left( \hat{P}_{n, \hat{\btheta}^{\rm MM}}, \hat{\rho} YG + \sqrt{1 - \hat{\rho}^2} Z \right) = & W_2 \left( \frac{1}{n} \sum_{i = 1}^{n} \delta_{ y_i \left\langle \hat{\btheta}^{\rm MM}, \xx_i \right\rangle } , \Law \left( Y \left( \hat{\rho} G + \sqrt{1 - \hat{\rho}^2} Z \right) \right) \right) \\
	\le & \frac{\sqrt{1 - \hat{\rho}^2}}{\sqrt{\alpha}} + 2 \veps. 
\end{align*}
Let $\kappa = \kappa_{\rm s} (\alpha)$, then according to Theorem 1 (b) in \cite{montanari2019generalization} we know that
\begin{equation*}
	\lim_{n \to \infty} \P \left( y_i \left\langle \hat{\btheta}^{\rm MM}, \xx_i \right\rangle \ge \kappa - \veps, \ \forall i \in [n] \right) = 1,
\end{equation*}
meaning that $\supp(\hat{P}_{n, \hat{\btheta}^{\rm MM}}) \subset [\kappa - \veps, \infty)$ with high probability. In other words, if we denote $\hat{P}_{n, \hat{\btheta}^{\rm MM}} = \Law (V)$, then with high probability we have $V \ge \kappa - \veps$, $a.s.$, and there exists a coupling $(V, Y, G, Z)$ such that
\begin{equation}\label{eq:kappa_margin_ineq}
	\begin{split}
		\E \left[ \left( V - \hat{\rho} YG - \sqrt{1 - \hat{\rho}^2} Z \right)^2 \right]^{1/2} \le & \sqrt{1 - \hat{\rho}^2} \cdot \frac{\E \left[ \left( \kappa - \rho_* YG - \sqrt{1 - \rho_*^2} Z \right)_+^2 \right]^{1/2}}{\sqrt{1 - \rho_*^2}} + 2 \veps \\
		= & \frac{\sqrt{1 - \hat{\rho}^2}}{\sqrt{F_{\kappa} (\rho_*)}} + 2 \veps,
	\end{split}
\end{equation}
where $\rho_* = \rho_* (\kappa) = \argmax_{\rho \in [-1, 1]} F_{\kappa} (\rho)$. For $\rho \in [-1, 1]$, let us define
\begin{equation*}
	f_{\kappa}^{*} ( \rho ) = \frac{1}{\sqrt{F_{\kappa} (\rho)}} - \frac{1}{\sqrt{F_{\kappa} (\rho_*)}}.
\end{equation*}
Then Lemma~\ref{lem:unique_F_kappa} implies that $f_{\kappa}^{*} (\rho) \ge 0$ for all $\rho \in [-1, 1]$, and $\rho \to \rho_*$ if and only if $f_{\kappa}^{*} (\rho) \to 0$. Hence, Eq.~\eqref{eq:kappa_margin_ineq} can be rewritten as
\begin{equation*}
	\E \left[ \left( V - \hat{\rho} YG - \sqrt{1 - \hat{\rho}^2} Z \right)^2 \right]^{1/2} \le \sqrt{1 - \hat{\rho}^2} \left( \frac{1}{\sqrt{F_{\kappa} (\hat{\rho})}} - f_{\kappa}^{*} (\hat{\rho}) \right) + 2 \veps,
\end{equation*}
which further implies that
\begin{align*}
	& \E \left[ \left( V - \hat{\rho} YG - \sqrt{1 - \hat{\rho}^2} Z \right)^2 \right]^{1/2} + \sqrt{1 - \hat{\rho}^2} \cdot f_{\kappa}^{*} (\hat{\rho}) \\
	\le & \E \left[ \left( \kappa - \hat{\rho} YG - \sqrt{1 - \hat{\rho}^2} Z \right)_{+}^2 \right]^{1/2} + 2 \veps \\
	\le & \E \left[ \left( \kappa - \veps - \hat{\rho} YG - \sqrt{1 - \hat{\rho}^2} Z \right)_{+}^2 \right]^{1/2} + 3 \veps,
\end{align*}
where the last inequality follows from the fact that
\begin{equation*}
	\frac{\d}{\d \kappa} \E \left[ \left( \kappa - \hat{\rho} YG - \sqrt{1 - \hat{\rho}^2} Z \right)_{+}^2 \right]^{1/2} = \frac{\E \left[ \left( \kappa - \hat{\rho} YG - \sqrt{1 - \hat{\rho}^2} Z \right)_{+} \right]}{\E \left[ \left( \kappa - \hat{\rho} YG - \sqrt{1 - \hat{\rho}^2} Z \right)_{+}^2 \right]^{1/2}} \le 1.
\end{equation*}
For future convenience, we denote $\hat{V} = \hat{\rho} YG + \sqrt{1 - \hat{\rho}^2} Z$, then it follows that
\begin{equation}\label{eq:kappa_W2_ineq}
	\E \left[ \left( V - \hat{V} \right)^2 \right]^{1 / 2} + \sqrt{1 - \hat{\rho}^2} \cdot f_{\kappa}^{*} (\hat{\rho}) \le \E \left[ \left( \max(\kappa - \veps, \hat{V}) - \hat{V} \right)^2 \right]^{1 / 2} + 3 \veps,
\end{equation}
and $V \ge \kappa - \veps$, $a.s.$. Note that when $\hat{V} < \kappa - \veps$, we have $V \ge \kappa - \veps > \hat{V}$, which leads to
\begin{align*}
	& ( V - \hat{V} )^2 \bone_{ \hat{V} < \kappa - \veps} \ge (V - (\kappa - \veps))^2 \bone_{\hat{V} < \kappa - \veps} + (\kappa - \veps - \hat{V})^2 \bone_{\hat{V} < \kappa - \veps} \\
	\implies & \E \left[ (V - (\kappa - \veps))^2 \bone_{\hat{V} < \kappa - \veps} \right] + \E \left[ (\kappa - \veps - \hat{V})^2 \bone_{\hat{V} < \kappa - \veps} \right] \le \E \left[ ( V - \hat{V} )^2 \bone_{ \hat{V} < \kappa - \veps} \right] \\
	\implies & \E \left[ (V - (\kappa - \veps))^2 \bone_{\hat{V} < \kappa - \veps} \right] + \E \left[ ( V - \hat{V} )^2 \bone_{ \hat{V} \ge \kappa - \veps} \right] + \E \left[ (\kappa - \veps - \hat{V})^2 \bone_{\hat{V} < \kappa - \veps} \right] \le \E \left[ ( V - \hat{V} )^2 \right] \\
	\implies & \E \left[ \left( V - \max(\kappa - \veps, \hat{V}) \right)^2 \right] \le \E \left[ (V - \hat{V})^2 \right] - \E \left[ \left( \max(\kappa - \veps, \hat{V}) - \hat{V} \right)^2 \right].
\end{align*}
Comparing this inequality with Eq.~\eqref{eq:kappa_W2_ineq} yields that $\sqrt{1 - \hat{\rho}^2} \cdot f_{\kappa}^{*} (\hat{\rho}) \le 3 \veps$. Note that for $\rho$ sufficiently close to $1$ or $- 1$ ($1 - \rho^2$ is small), the quantity
\begin{equation*}
	\sqrt{1 - \rho^2} \cdot f_{\kappa}^{*} (\rho) = \E \left[ \left( \kappa - \rho YG - \sqrt{1 - \rho^2} Z \right)_+^2 \right]^{1/2} - \frac{\sqrt{1 - \rho^2}}{\sqrt{F_{\kappa} (\rho_*)}}
\end{equation*}
is bounded away from $0$. Hence, for $\veps > 0$ small enough, we must have $\vert \hat{\rho} - \rho_* \vert = o_{\veps} (1)$. Moreover, Eq.~\eqref{eq:kappa_W2_ineq} implies that there exists a constant $C := C(\kappa)$ which only depends on $\kappa$ so that
\begin{equation*}
	\E \left[ \left( V - \max(\kappa - \veps, \hat{V}) \right)^2 \right] \le C \veps.
\end{equation*}
Therefore, we finally obtain that
\begin{equation*}
	W_2 \left( \hat{P}_{n, \hat{\btheta}^{\rm MM}}, \Law \left( \max \left( \kappa - \veps, \hat{\rho} YG + \sqrt{1 - \hat{\rho}^2} Z \right) \right) \right) \le \sqrt{C \veps},
\end{equation*}
thus leading to the following estimate:
\begin{align*}
	W_2 \left( \hat{P}_{n, \hat{\btheta}^{\mathrm{MM}}}, P_{\kappa, \varphi} \right) = & W_2 \left( \hat{P}_{n, \hat{\btheta}^{\mathrm{MM}}}, \Law \left( \max \left( \kappa, \rho_* YG + \sqrt{1 - \rho_*^2} Z \right) \right) \right)	\\
	\le & W_2 \left( \hat{P}_{n, \hat{\btheta}^{\mathrm{MM}}}, \Law \left( \max \left( \kappa, \hat{\rho} YG + \sqrt{1 - \hat{\rho}^2} Z \right) \right) \right) \\
	& + W_2 \left( \Law \left( \max \left( \kappa, \hat{\rho} YG + \sqrt{1 - \hat{\rho}^2} Z \right) \right), \Law \left( \max \left( \kappa, \rho_* YG + \sqrt{1 - \rho_*^2} Z \right) \right) \right) \\
	\stackrel{(i)}{=} & W_2 \left( \hat{P}_{n, \hat{\btheta}^{\mathrm{MM}}}, \Law \left( \max \left( \kappa, \hat{\rho} YG + \sqrt{1 - \hat{\rho}^2} Z \right) \right) \right) + o_{\veps} (1) \\
	\le & W_2 \left( \hat{P}_{n, \hat{\btheta}^{\mathrm{MM}}}, \Law \left( \max \left( \kappa - \veps, \hat{\rho} YG + \sqrt{1 - \hat{\rho}^2} Z \right) \right) \right) + o_{\veps} (1) \\
	& + W_2 \left( \Law \left( \max \left( \kappa, \hat{\rho} YG + \sqrt{1 - \hat{\rho}^2} Z \right) \right), \Law \left( \max \left( \kappa - \veps, \hat{\rho} YG + \sqrt{1 - \hat{\rho}^2} Z \right) \right) \right) \\
	\stackrel{(ii)}{\le} & W_2 \left( \hat{P}_{n, \hat{\btheta}^{\mathrm{MM}}}, \Law \left( \max \left( \kappa - \veps, \hat{\rho} YG + \sqrt{1 - \hat{\rho}^2} Z \right) \right) \right) + \veps + o_{\veps} (1) \\
	\le & \sqrt{C \veps} + \veps + o_{\veps} (1) = o_{\veps} (1),
\end{align*}
where $(i)$ is due to the fact that $\vert \hat{\rho} - \rho_* \vert = o_{\veps} (1)$, $(ii)$ is due to the fact that the mapping $\kappa \mapsto \max(\kappa, \hat{\rho} YG + \sqrt{1 - \hat{\rho}^2} Z)$ is $1$-Lipschitz. We thus conclude that for any $\veps > 0$, $W_2 ( \hat{P}_{n, \hat{\btheta}^{\mathrm{MM}}}, P_{\kappa, \varphi} ) \le o_{\veps} (1)$ happens with probability converging to one as $n \to \infty$. Sending $\veps \to 0$ gives
\begin{equation*}
	W_2 \left( \hat{P}_{n, \hat{\btheta}^{\mathrm{MM}}}, P_{\kappa, \varphi} \right) \stackrel{p}{\to} 0 \ \text{as} \ n \to \infty.
\end{equation*}
This completes the proof of Theorem~\ref{thm:mm_margin_dist}.
\end{proof}

\section{Auxiliary Lemmas}\label{sec:aux_lemma}
We need the following variant of Gordon's comparison theorem for Gaussian processes \cite{gordon1985some, thrampoulidis2015regularized}:
\begin{lem}[Corollary G.1 from \cite{miolane2021distribution}]\label{lem:gordon}
	Let $\cD_{\uu} \subset \mathbb{R}^{n_{1}+n_{2}}$ and $\cD_{\vv} \subset \mathbb{R}^{m_{1}+m_{2}}$ be compact sets and let $Q: \cD_{\uu} \times \cD_{\vv} \rightarrow \mathbb{R}$ be a continuous function. Let $\bGG = (G_{i, j} ) \stackrel{\iid}{\sim} \mathcal{N}(0,1)$, $\bgg \sim \sN (\bzero, \mathbf{I}_{n_{1}} )$ and $\hh \sim \sN (\bzero, \mathbf{I}_{m_{1}} )$ be independent standard Gaussian vectors. For $\uu \in \mathbb{R}^{n_{1}+n_{2}}$ and $\vv \in \mathbb{R}^{m_{1}+m_{2}}$, we define $\tilde{\uu}=\left(u_{1}, \ldots, u_{n_{1}}\right)$ and $\tilde{\vv}=\left(v_{1}, \ldots, v_{m_{1}}\right)$. Define
	\begin{equation*}
		\left\{\begin{array}{l}
			C^{*}(\bGG)=\min_{\uu \in \cD_{\uu}} \max_{\vv \in \cD_{\vv}} \tilde{\vv}^{\top} G \tilde{\uu}+Q(\uu, \vv), \\
			L^{*}(\bgg, \hh)=\min_{\uu \in \cD_{\uu}} \max_{\vv \in \cD_{\vv}} \norm{\tilde{\vv}}_2 \bgg^{\top} \tilde{\uu} + \norm{\tilde{\uu}}_2 \hh^{\top} \tilde{\vv} + Q(\uu, \vv) .
		\end{array}\right.
	\end{equation*}
	Then we have:
	\begin{itemize}
		\item [(i)] For all $t \in \mathbb{R}$,
		\begin{equation*}
			\mathbb{P}\left(C^{*}(\bGG) \leq t \right) \leq 2 \mathbb{P}\left(L^{*}(\bgg, \hh) \leq t\right).
		\end{equation*}
		\item [(ii)] If $\cD_{\uu}$ and $\cD_{\vv}$ are convex and if $Q$ is convex-concave, then for all $t \in \mathbb{R}$,
		\begin{equation*}
			\mathbb{P}\left(C^{*}(\bGG) \geq t\right) \leq 2 \mathbb{P}\left(L^{*}(\bgg, \hh) \geq t\right) .
		\end{equation*}
	\end{itemize}
\end{lem}

The next lemma collects some useful properties regarding the constrained second Wasserstein distance, which is introduced in Definition~\ref{def:cons_wp_metric}:
\begin{lem}\label{lem:prop_cons_W2}
	The followings are true about $W_2^{(\eta)} (\cdot, \cdot)$:
	\begin{itemize}
		\item [(i)] $W_2^{(\eta)} (\cdot, \cdot)$ satisfies the constrained triangle inequality: Let $\mu_1, \mu_2, \mu_3$ be three probability measures defined on the same space, and $\eta_1, \eta_2 \ge 0$, then we have
		\begin{equation*}
			W_2^{(\eta_1 + \eta_2)} (\mu_1, \mu_3) \le W_2^{(\eta_1)} (\mu_1, \mu_2) + W_2^{(\eta_2)} (\mu_2, \mu_3).
		\end{equation*}
		\item [(ii)] Let $(\yy, \uu) = \{ (y_i, u_i) \}_{i \in [n]}$ be such that $y_i \in \R$, $u_i \in \R^{d - 1}$ for all $i$, and let $P$ be a fixed probability distribution on $\R^d$. Define the function
		\begin{equation*}
			f (\yy, \uu) = W_2^{(\eta)} \left( \frac{1}{n} \sum_{i=1}^{n} \delta_{(y_i, u_i)} , P  \right).
		\end{equation*}
		Then, for fixed $\yy$, $f$ is continuous in $\uu$ on the set $\{ (\yy, \uu): f(\yy, \uu) < \infty \}$.
	\end{itemize}
\end{lem}

\begin{proof}    
	$(i)$ Denote $\mu_i = \Law (y_i, u_i)$ for $i = 1, 2, 3$. For any $\veps > 0$, there exist two couplings $(y_1, u_1, y_2, u_2)$ and $(y_2, u_2, y_3, u_3)$ such that
	\begin{align*}
		& \E \left[ \norm{y_1 - y_2}_2^2 \right]^{1 / 2} \le \eta_1, \ \E \left[ \norm{(y_1, u_1) - (y_2, u_2)}_2^2 \right]^{1 / 2} \le W_2^{(\eta_1)} (\mu_1, \mu_2) + \veps, \\
		& \E \left[ \norm{y_2 - y_3}_2^2 \right]^{1 / 2} \le \eta_2, \ \E \left[ \norm{(y_2, u_2) - (y_3, u_3)}_2^2 \right]^{1 / 2} \le W_2^{(\eta_2)} (\mu_2, \mu_3) + \veps.
	\end{align*}
	According to the gluing lemma in Chapter 1 of \cite{villani2009optimal}, there exists a coupling $(y_1, u_1, y_2, u_2, y_3, u_3)$ such that $\E \left[ \norm{y_1 - y_3}_2^2 \right]^{1 / 2} \le \eta_1 + \eta_2$, and that
	\begin{align*}
		& W_2^{(\eta_1 + \eta_2)} (\mu_1, \mu_3) \le \E \left[ \norm{(y_1, u_1) - (y_3, u_3)}_2^2 \right]^{1 / 2} \\
		\le & \E \left[ \norm{(y_1, u_1) - (y_2, u_2)}_2^2 \right]^{1 / 2} + \E \left[ \norm{(y_2, u_2) - (y_3, u_3)}_2^2 \right]^{1 / 2} \\
		\le & W_2^{(\eta_1)} (\mu_1, \mu_2) + W_2^{(\eta_2)} (\mu_2, \mu_3) + 2 \veps.
	\end{align*}
	Letting $\veps \to 0$ gives the desired inequality.
	
	$(ii)$ We first note that whether $f (\yy, \uu) < \infty$ only depends on $\yy$. Assume $f (\yy, \uu) < \infty$, then for any $\uu' \in \R^{n \times (d - 1)}$, one has the following inequality:
	\begin{align*}
		f (\yy, \uu) = & W_2^{(\eta)} \left( \frac{1}{n} \sum_{i=1}^{n} \delta_{(y_i, u_i)} , P  \right) \le W_2^{(\eta)} \left( \frac{1}{n} \sum_{i=1}^{n} \delta_{(y_i, u'_i)} , P  \right) \\
		& + W_2^{(0)} \left( \frac{1}{n} \sum_{i=1}^{n} \delta_{(y_i, u'_i)} , \frac{1}{n} \sum_{i=1}^{n} \delta_{(y_i, u_i)} \right) \\
		\le & f(\yy, \uu') + \frac{1}{\sqrt{n}} \norm{\uu - \uu'}_{\rm F}.
	\end{align*}
	It follows similarly that $f(\yy, \uu') \le f (\yy, \uu) + \norm{\uu - \uu'}_{\rm F} / \sqrt{n}$. This proves the continuity of $f(\yy, \cdot)$.
\end{proof}

We also need the following lemma which characterizes the convergence rate of the empirical measure in $W_2$ distance. This lemma is a corollary of Theorem 2 in \cite{fournier2015rate}.

\begin{lem}\label{lem:W2_concentration}
	Let $Z$ and $\{ Z_i \}_{i \in [n]}$ be i.i.d. sub-Gaussian random vectors in $\R^{k}$, then for any $\veps > 0$, there exist positive constants $C_1$ and $C_2$ such that
	\begin{equation*}
		\P \left( W_2 \left( \frac{1}{n} \sum_{i = 1}^{n} \delta_{Z_i}, \Law (Z) \right) > \veps \right) \le C_1 \exp \left( - C_2 n^{1 - \veps} \right),
	\end{equation*}
	where $C_1$ and $C_2$ only depend on $\veps$ and the common distribution $\Law (Z)$.
\end{lem}

\begin{proof}
	It is straightforward to verify Condition (2) of Theorem 2 from \cite{fournier2015rate} for $p = 2$ and $\alpha < 2$, since $Z$ is sub-Gaussian. Indeed, we have for $\mu = \Law(Z)$,
	\begin{equation*}
		\cE_{\alpha, \gamma} (\mu) = \E \left[ \exp \left( \gamma \norm{Z}_2^{\alpha} \right) \right] \le e^{\gamma} \P \left( \norm{Z}_2 \le 1 \right) + \E \left[ \exp \left( \gamma \norm{Z}_2^{2} \right) \right] < \infty
	\end{equation*}
	for $\gamma$ small enough.
	The lemma then follows from the definition of $a(n, x)$ and $b(n, x)$ under Condition (2) in the conclusion of Theorem 2 of \cite{fournier2015rate}.
\end{proof}

The lemma below is a standard result which connects in probability weak convergence and convergence in bounded-Lipschitz distance between distributions. Its proof is based on standard approximation arguments, and can be found in some well-known probability textbooks (cf. \cite{van1996weak}).

\begin{lem}\label{lem:BL_diatance_conv}
	Let $\{ \hat{P}_n \}_{n \in \mathbb{N}}$ be a sequence of random probability measures on a Polish space $\mathbb{M}$, and $P$ be a fixed probability measure on the same space, then as $n \to \infty$, $\hat{P}_n$ weakly converges to $P$ in probability if and only if $d_{\rm BL} (\hat{P}_n, P) \to 0$ in probability, where
	\begin{equation*}
		d_{\rm BL} \left( \hat{P}_n, P \right) = \sup_{f \in \cF_{\rm BL} (\mathbb{M})} \left\vert \int_{\mathbb{M}} f \d \hat{P}_n - \int_{\mathbb{M}} f \d P \right\vert.
	\end{equation*}
	Here, $\cF_{\rm BL} (\mathbb{M})$ stands for the class of all bounded Lipschitz functions on $\mathbb{M}$, i.e., $f \in \cF_{\rm BL} (\mathbb{M})$ if and only if for all $x, y \in \mathbb{M}$, $\vert f(x) \vert \le 1$ and $\vert f(x) - f(y) \vert \le \rho (x, y)$, where $\rho$ is the metric of $\mathbb{M}$.
\end{lem}

The next lemma characterizes the limiting empirical distribution of a uniform random vector on a high-dimensional unit sphere:

\begin{lem}\label{lem:unif_empr_dist}
	Assume $\uu = (u_1, \cdots, u_n)^\top \sim \Unif(\S^{n - 1})$, then as $n \to \infty$,
	\begin{equation*}
		\frac{1}{n} \sum_{i = 1}^{n} \delta_{\sqrt{n} u_i} \stackrel{w}{\Rightarrow} \sN (0, 1) \ \text{in probability}.
	\end{equation*}
\end{lem}

\begin{proof}
	Let $\{ g_n \}_{n \ge 1} \sim_{\iid} \sN (0, 1)$, then we know that $(1 / n) \sum_{i = 1}^{n} \delta_{g_i} \stackrel{w}{\Rightarrow} \sN (0, 1)$ almost surely. By the strong law of large numbers, it follows that $s_n^2 / n = (1 / n) \sum_{i = 1}^{n} g_i^2 \stackrel{a.s.}{\to} 1$. Hence, Slutsky's theorem implies that
	\begin{equation*}
		\frac{1}{n} \sum_{i = 1}^{n} \delta_{\frac{\sqrt{n} g_i}{s_n}} \stackrel{w}{\Rightarrow} \sN (0, 1) \ \text{almost surely}.
	\end{equation*}
	Moreover, from rotational invariance of normal distribution we deduce that
	\begin{equation*}
		\left( \frac{\sqrt{n} g_1}{s_n}, \cdots, \frac{\sqrt{n} g_n}{s_n} \right) \stackrel{d}{=} (\sqrt{n} u_1, \cdots, \sqrt{n} u_n),
	\end{equation*}
	thus leading to
	\begin{equation*}
		\frac{1}{n} \sum_{i = 1}^{n} \delta_{\sqrt{n} u_i} \stackrel{d}{=} \frac{1}{n} \sum_{i = 1}^{n} \delta_{\frac{\sqrt{n} g_i}{s_n}} \implies d_{\rm BL} \left( \frac{1}{n} \sum_{i = 1}^{n} \delta_{\sqrt{n} u_i}, \sN (0, 1) \right) \stackrel{d}{=} d_{\rm BL} \left( \frac{1}{n} \sum_{i = 1}^{n} \delta_{\frac{\sqrt{n} g_i}{s_n}}, \sN (0, 1) \right).
	\end{equation*}
	By our assumption, $d_{\rm BL} \left( (1 / n) \sum_{i = 1}^{n} \delta_{\sqrt{n} g_i / s_n}, \sN (0, 1) \right) \stackrel{a.s.}{\to} 0$. Hence we have 
	\begin{equation*}
		d_{\rm BL} \left( (1 / n) \sum_{i = 1}^{n} \delta_{\sqrt{n} u_i}, \sN (0, 1) \right) \stackrel{p}{\to} 0.
	\end{equation*}
	This implies that $(1 / n) \sum_{i = 1}^{n} \delta_{\sqrt{n} u_i} \stackrel{w}{\Rightarrow} \sN(0, 1)$ in probability and completes the proof.
\end{proof}

We present a lemma regarding the sub-Gaussianity of the difference between Lipschitz function of Gaussian random vectors:
\begin{lem}\label{lem:MGF_diff_Lip}
	Let $F: \R^{m \times n} \to \R$ be an $L$-Lipschitz function, namely
	\begin{equation}
		\left\vert F(\GG) - F(\GG') \right\vert \le L \norm{\GG - \GG'}_{\rm F}
	\end{equation}
	for all $\GG = (G_1, \cdots, G_n),\GG' = (G_1', \cdots, G_n') \in \R^{m \times n}$. Then, for any $\WW, \WW' \in O(d, m)$ and $t > 0$, we have
	\begin{equation}
		\E \left[ \exp \left( t \left( F \left( \WW^\top \XX^\top \right) - F \left( \WW'^\top \XX^\top \right) \right) \right) \right] \le \exp \left( \frac{L^2 t^2}{2} \norm{\WW - \WW'}_{\op}^2 \right).
	\end{equation}
\end{lem}

\begin{proof}
	By definition, we know that $(\WW^\top \XX^\top, \WW'^\top \XX^\top) \stackrel{d}{=} (\GG, \GG')$ where
	\begin{equation*}
		(G_i, G_i') \sim_{\iid} \normal \left( \bzero, \begin{bmatrix}
            \id_m & \QQ^\top \\
            \QQ & \id_m
            \end{bmatrix} \right).
	\end{equation*}
	Hence, it suffices to upper bound $\E [\exp (t (F(\GG) - F(\GG')))]$. To this end, we use Gaussian interpolation, which is similar to the proof of Theorem 1.3.4 in \cite{talagrand2010mean}. Define two mutually independent Gaussian ensembles:
	\begin{align}
		(\HH, \HH') =\, & \left( H_1, \cdots, H_n, H_1', \cdots, H_n' \right), \quad (H_i, H_i') \sim_{\iid} \normal \left( \bzero, \begin{bmatrix}
            \id_m & \id_m \\
            \id_m & \id_m
            \end{bmatrix} \right), \\
        (\GG, \GG') =\, & \left( G_1, \cdots, G_n, G_1', \cdots, G_n' \right), \quad (G_i, G_i') \sim_{\iid} \normal \left( \bzero, \begin{bmatrix}
            \id_m & \QQ^\top \\
            \QQ & \id_m
            \end{bmatrix} \right),
	\end{align}
	and denote for $\YY = (Y_1, \cdots, Y_n)$ and $\YY' = (Y_1', \cdots, Y_n')$:
	\begin{equation}
		f(\YY, \YY') = \exp \left( t \left( F(\YY) - F(\YY') \right) \right).
	\end{equation}
	We further construct the interpolation path
	\begin{equation*}
		\ZZ (s) = \sqrt{s} \GG + \sqrt{1-s} \HH, \quad \ZZ' (s) = \sqrt{s} \GG' + \sqrt{1-s} \HH',
	\end{equation*}
	then it follows that
	\begin{align*}
		& \E \left[ \exp \left( t \left( F(\GG) - F(\GG') \right) \right) \right] = \E \left[ f(\GG, \GG') \right] = \E \left[ f(\ZZ(1), \ZZ' (1)) \right] \\
		=\, & \E \left[ f(\ZZ(0), \ZZ'(0)) \right] + \int_{0}^{1} \d \E \left[ f(\ZZ(s), \ZZ'(s)) \right] = 1 + \int_{0}^{1} \frac{\d \E \left[ f(\ZZ(s), \ZZ'(s)) \right]}{\d s} \d s.
	\end{align*}
	By Lemma 1.3.1 in \cite{talagrand2010mean}, we know that
	\begin{align*}
		& \frac{\d \E \left[ f(\ZZ(s), \ZZ'(s)) \right]}{\d s} = \sum_{l=1}^{n} \sum_{i=1}^{m} \sum_{j=1}^{m} (q_{ij} - \delta_{ij}) \cdot \E \left[ \frac{\partial^2 f (\ZZ(s), \ZZ'(s))}{\partial Y_{il} \partial Y_{jl}'} \right] \\
		=\, & \sum_{l=1}^{n} \sum_{i=1}^{m} \sum_{j=1}^{m} (q_{ij} - \delta_{ij}) \cdot \E \left[ - t^2 \frac{\partial F(\ZZ(s))}{\partial Y_{il}} \frac{\partial F (\ZZ'(s))}{\partial Y_{jl}'} f(\ZZ(s), \ZZ'(s)) \right] \\
		=\, & t^2 \cdot \E \left[ \sum_{l=1}^{n} \sum_{i=1}^{m} \sum_{j=1}^{m} (\delta_{ij} - q_{ij}) \frac{\partial F(\ZZ(s))}{\partial Y_{il}} \frac{\partial F (\ZZ'(s))}{\partial Y_{jl}'} f(\ZZ(s), \ZZ'(s)) \right].
	\end{align*}
	For any fixed $l \in [n]$, we have the following estimate:
	\begin{equation*}
		\sum_{i=1}^{m} \sum_{j=1}^{m} (\delta_{ij} - q_{ij}) \frac{\partial F(\ZZ(s))}{\partial Y_{il}} \frac{\partial F (\ZZ'(s))}{\partial Y_{jl}'} = \norm{\id - \frac{\QQ + \QQ^\top}{2}}_{\op} \cdot \sqrt{\sum_{i=1}^{m} \left( \frac{\partial F(\ZZ(s))}{\partial Y_{il}} \right)^2} \cdot \sqrt{\sum_{j=1}^{m} \left( \frac{\partial F (\ZZ'(s))}{\partial Y_{jl}'} \right)^2},
	\end{equation*}
	thus leading to
	\begin{align*}
		& \sum_{l=1}^{n} \sum_{i=1}^{m} \sum_{j=1}^{m} (\delta_{ij} - q_{ij}) \frac{\partial F(\ZZ(s))}{\partial Y_{il}} \frac{\partial F (\ZZ'(s))}{\partial Y_{jl}'} \\
		\le\, & \norm{\id - \frac{\QQ + \QQ^\top}{2}}_{\op} \cdot \sum_{l=1}^{n} \sqrt{\sum_{i=1}^{m} \left( \frac{\partial F(\ZZ(s))}{\partial Y_{il}} \right)^2} \cdot \sqrt{\sum_{j=1}^{m} \left( \frac{\partial F (\ZZ'(s))}{\partial Y_{jl}'} \right)^2} \\
		\stackrel{(i)}{\le}\, & \norm{\id - \frac{\QQ + \QQ^\top}{2}}_{\op} \cdot \sqrt{\sum_{l=1}^{n} \sum_{i=1}^{m} \left( \frac{\partial F(\ZZ(s))}{\partial Y_{il}} \right)^2} \cdot \sqrt{\sum_{l=1}^{n} \sum_{j=1}^{m} \left( \frac{\partial F (\ZZ'(s))}{\partial Y_{jl}'} \right)^2} \\
		\stackrel{(ii)}{\le}\, & L^2 \cdot \norm{\id - \frac{\QQ + \QQ^\top}{2}}_{\op} = \frac{L^2}{2} \norm{(\WW - \WW')^\top (\WW - \WW')}_{\op} = \frac{L^2}{2} \norm{\WW - \WW'}_{\op}^2,
	\end{align*}
	where $(i)$ follows from Cauchy-Schwarz inequality, and $(ii)$ follows from the fact that $F$ is $L$-Lipschitz. Finally, we deduce that
	\begin{equation*}
		\frac{\d \E \left[ f(\ZZ(s), \ZZ'(s)) \right]}{\d s} \le \frac{t^2 L^2}{2} \norm{\WW - \WW'}_{\op}^2 \cdot \E \left[ f(\ZZ(s), \ZZ'(s)) \right].
	\end{equation*}
	Using Gr\"{o}nwall's inequality, it follows that
	\begin{equation*}
		\E \left[ \exp \left( t \left( F(\GG) - F(\GG') \right) \right) \right] = \E \left[ f(\ZZ(1), \ZZ'(1)) \right] \le \exp \left( \frac{t^2 L^2}{2} \norm{\WW - \WW'}_{\op}^2 \right).
	\end{equation*}
	This completes the proof of Lemma~\ref{lem:MGF_diff_Lip}.
\end{proof}

The following lemma deals with the continuity of $W_p$ distance with respect to $p$:

\begin{lem}\label{lem:Wp_continuity}
	Let $P$ and $Q$ be two probability measures with finite $q$-th moments ($q > 1$) on a Polish space $S$. Assume $M > 0$ is such that for all $p \in [1, q)$, $W_p (P, Q) \le M$. Then we have $W_q (P, Q) \le M$.
\end{lem}
\begin{proof}
	By assumption, we know that for any $\veps > 0$ and $p_k = q - 1 / k$, there exists a coupling $(P_k, Q_k) = \Law (U_k, V_k)$ of $P$ and $Q$ such that $\E [(U_k - V_k)^p]^{1/ p} \le M + \veps$. Note that the $q$-th moment of $(P_k, Q_k)$ is uniformly bounded, hence the sequence $\{ (P_k, Q_k) \}_{k \ge 1}$ is tight. Consequently, there exists a subsequence (still denoted as $\{ (P_k, Q_k) \}_{k \ge 1}$) which weakly converges to $(P, Q) = \Law(U, V)$. Using Skorokhod's representation theorem, we can assume with out loss of generality that $(U_k, V_k) \to (U, V)$ almost surely as $k \to \infty$. Applying Fatou's lemma implies that
	\begin{equation*}
		\E \left[ (U - V)^q \right] \le \liminf_{k \to \infty} \E \left[ (U_k - V_k)^{p_k} \right] \le (M + \veps)^q,
	\end{equation*}
	thus leading to $W_q (P, Q) \le \E \left[ (U - V)^q \right]^{1 / q} \le M + \veps$. Sending $\veps \to 0$ yields the desired result.
\end{proof}

For the sake of completeness, we include the closure property of the set of $(\alpha, m)$-feasible distributions. Note that the following lemma applies to $\cuF_{m, \alpha}^{\varphi}$ in the supervised case as well.

\begin{lem}\label{lem:closure_prop}
	The set $\cuF_{m, \alpha}$ is closed under weak limit in $\cuP (\R^m)$.
\end{lem}

\begin{proof}
	Let $\{ P_l \}_{l \ge 1}$ be a sequence of $(\alpha, m)$-feasible probability distributions on $\R^m$, and that $P_l$ weakly converges to $P \in \cuP (\R^m)$ as $l \to \infty$. We aim to show that $P$ is $(\alpha, m)$-feasible as well.
	
	By definition of $(\alpha, m)$-feasibility, for each fixed $l \ge 1$, there exists a sequence of random orthogonal matrices $\WW_{n, l} = \WW_{n, l} (\XX)$ such that
	\begin{equation*}
		\hat{P}_{n, l} = \frac{1}{n} \sum_{i=1}^{n} \delta_{ \left( \xx_i^\top \WW_{n, l} \right) } \stackrel{w}{\Rightarrow} P_{l} \ \text{in probability}.
	\end{equation*}
	According to Lemma~\ref{lem:BL_diatance_conv}, we know that $d_{\rm BL} (\hat{P}_{n, l}, P_l) \to 0$ in probability as $n \to \infty$, where $d_{\rm BL}$ denotes the bounded-Lipschitz distance which metrizes weak convergence.
	
	Consider $l = 1$, then there exists a subsequence $\{ n_{k_1} \}$ of $\mathbb{N}$ such that $d_{\rm BL} (\hat{P}_{n_{k_1}, 1}, P_1) \to 0$ almost surely. For $l = 2$, there exists a further subsequence $\{ n_{k_1 k_2} \}$ of $\{ n_{k_1} \}$ such that $d_{\rm BL} (\hat{P}_{n_{k_1 k_2}, 2}, P_2) \to 0$ almost surely. Proceeding similarly for $l \ge 3$ and using the diagonal argument, we finally deduce that there exists a subsequence $\{ n_k \}$ of $\mathbb{N}$ such that for all $l \ge 1$,
	\begin{equation*}
		\lim_{k \to \infty} d_{\rm BL} \left( \hat{P}_{n_k, l}, P_l \right) = 0 \ \text{almost surely}.
	\end{equation*}
	
	Now we know that almost surely for all $l \ge 1$, $\lim_{k \to \infty} d_{\rm BL} (\hat{P}_{n_k, l}, P_l) = 0$ and $\lim_{l \to \infty} d_{\rm BL} (P_l, P) = 0$. Fix $k \ge 1$, we can choose $l = l(k) \ge 1$ such that
	\begin{equation*}
		d_{\rm BL} \left( \hat{P}_{n_k, l(k)}, P \right) \le \inf_{l \ge 1} d_{\rm BL} \left( \hat{P}_{n_k, l}, P \right) + \frac{1}{k}.
	\end{equation*}
	Since $d_{\rm BL} (P_l, P) \to 0$, for any $\veps > 0$, there exists $l \ge 1$ satisfying $d_{\rm BL} (P_l, P) \le \veps/2$, and $N \ge 1$ such that $\forall k \ge N$, $d_{\rm BL} (\hat{P}_{n_k, l}, P_l) \le \veps/2$, thus leading to
	\begin{equation*}
		\inf_{l \ge 1} d_{\rm BL} \left( \hat{P}_{n_k, l}, P \right) \le d_{\rm BL} \left( \hat{P}_{n_k, l}, P \right) \le d_{\rm BL} \left( \hat{P}_{n_k, l}, P_l \right) + d_{\rm BL} \left( P_l, P \right) \le \veps.
	\end{equation*}
	Hence $\inf_{l \ge 1} d_{\rm BL} ( \hat{P}_{n_k, l}, P ) \to 0$ as $k \to \infty$, and as a consequence we obtain that
	\begin{equation*}
		\lim_{k \to \infty} d_{\rm BL} \left( \hat{P}_{n_k, l(k)}, P \right) = 0 \ \text{almost surely}.
	\end{equation*}
	Finally, let us define $\WW_{n_k} = \WW_{n_k, l(k)} \in \R^{d_k \times m}$ (note that $l(k)$ is also random). Then $\WW_{n_k}$ is a random orthogonal matrix. We further define
	\begin{equation*}
		\hat{P}_{n_k} = \frac{1}{n_k} \sum_{i=1}^{n_k} \delta_{\left( \xx_i^\top \WW_{n_k} \right)},
	\end{equation*}
	then it follows that $\hat{P}_{n_k} \stackrel{w}{\Rightarrow} P$ almost surely as $k \to \infty$. This will remain true if $\{ n_k \}$ is replaced by any subsequence as well, which implies that $P$ is $(\alpha, m)$-feasible, as desired. This completes the proof.
\end{proof}

The next lemmas collect some useful properties regarding the information projection onto an affine subspace of discrete probability measures:

\begin{lem}\label{lem:info_proj_cond}
	Let $\pp = (p_{ij})_{i, j \in [M]}$ be a probability distribution on $[M] \times [M]$ satisfying $p_{ij} > 0$ for all $i, j \in [M]$, and $(q_i)_{i \in [M]}$ be a probability vector such that $q_i > 0$ for all $i$. Define
	\begin{align}\label{eq:def_info_proj}
			\qq = (q_{ij})_{i, j \in [M]} :=  \argmin_{\qq} 
			\Big\{D_{\sKL} \left( \qq \Vert \pp \right) \;\;
			\mbox{\rm subj. to} \ \;\; \sum_{j=1}^{M} q_{ij} =  \sum_{j=1}^{M} q_{ji} = q_i, \ \forall i \in [M] \Big\}.
	\end{align}
    Then, we have
    \begin{itemize}
    	\item [(a)] There exist two vectors $\ff = (f_i)_{i \in [M]}$ and $\bgg = (g_i)_{i \in [M]}$ with positive components, such that $q_{ij} = f_i g_j p_{ij}$ for all $i, j \in [M]$. Moreover, $\ff$ and $\bgg$ are uniquely 
    	determined (up to a multiplicative constant) by the equations
    	    \begin{equation}\label{eq:cons_f_g}
    	    	q_i = f_i \sum_{j=1}^{M} g_j p_{ij}, \ q_j = g_j \sum_{i=1}^{M} f_i p_{ij}, \ \forall i, j \in [M].
    	    \end{equation} 
    	\item [(b)] For any $\qq' = (q_{ij}')_{i, j \in [M]}$ satisfying the marginalization constraints
    	$\sum_{j=1}^{M} q_{ij} =  \sum_{j=1}^{M} q_{ji} = q_i$
    	 in Eq.~\eqref{eq:def_info_proj}, we have
    	    \begin{equation*}
    	    	\sum_{i, j=1}^{M} (q_{ij}' - q_{ij}) \log \frac{q_{ij}}{p_{ij}} = 0.
    	    \end{equation*}
        \item [(c)] For any $i, j \in [M]$, $q_{ij} \ge q_i q_j (\inf_{i, j \in [M]} p_{ij})^3$.
    \end{itemize}
\end{lem}

\begin{proof}
	Note that the optimization problem~\eqref{eq:def_info_proj} is a convex one with affine constraints, and its feasibility set has non-empty interior. Hence, Slater's condition is satisfied. According to strong duality, there exist vectors $\blambda = (\lambda_i)_{i \in [M]}$ and $\bmu = (\mu_i)_{i \in [M]}$ such that
    \begin{equation*}
    	\qq = \argmin_{\qq} \left\{ \sum_{i, j = 1}^{M} q_{ij} \log \frac{q_{ij}}{p_{ij}} + \sum_{i=1}^{M} \lambda_i \left( \sum_{j=1}^{M} q_{ij} - q_i \right) + \sum_{j=1}^{M} \mu_j \left( \sum_{i=1}^{M} q_{ij} - q_j \right) \right\}.
    \end{equation*}
    Differentiating the objective function, we obtain that
    \begin{equation*}
    	\log \frac{q_{ij}}{p_{ij}} + 1 + \lambda_i + \mu_j = 0 \implies q_{ij} = p_{ij} \exp \left( - \lambda_i - \mu_j - 1 \right).
    \end{equation*}
    Setting $f_i = \exp(- \lambda_i - 1/2)$ and $g_j = \exp(- \mu_j - 1/2)$ then proves part (a). To see the equations satisfied by $\ff$ and $\bgg$, we note that
    \begin{equation*}
    	q_i = \sum_{j=1}^{M} q_{ij} = \sum_{j=1}^{M} f_i g_j p_{ij} = f_i \sum_{j=1}^{M} g_j p_{ij},
    \end{equation*}
    and similarly $q_j = g_j \sum_{i=1}^{M} f_i p_{ij}$. To prove uniqueness, fix $f_1 = 1$, then $g_j = f_1 g_j = q_{1j} / p_{1j}$ is uniquely determined, and each $f_i = f_i g_j / g_j = q_{ij} / (p_{ij} g_j)$ is uniquely determined. For part (b), by direct calculation, we get that
    \begin{equation*}
    	\sum_{i, j=1}^{M} (q_{ij}' - q_{ij}) \log \frac{q_{ij}}{p_{ij}} = \sum_{i, j=1}^{M} (q_{ij}' - q_{ij}) (\log f_i + \log g_j) = \sum_{i=1}^{M} \log f_i \sum_{j=1}^{M} (q_{ij}' - q_{ij}) + \sum_{j=1}^{M} \log g_j \sum_{i=1}^{M} (q_{ij}' - q_{ij}) = 0.
    \end{equation*}
    Now let us prove part (c). Denote $c = \min_{i,j \in [M]} p_{ij}$, for any $i \neq k$, one has
    \begin{equation*}
    	\frac{f_k}{f_i} = \frac{\sum_{j=1}^{M} f_k g_j}{\sum_{j=1}^{M} f_i g_j} = \frac{\sum_{j=1}^{M} q_{kj}/p_{kj}}{\sum_{j=1}^{M} q_{ij}/p_{ij}} \le \frac{q_k \cdot \max_{j \in [M]} p_{ij}}{q_i \cdot \min_{j \in [M]} p_{kj}} \le \frac{q_k}{q_i c}.
    \end{equation*}
    Similarly, we can show that $g_k / g_i \le (q_k \max_{j \in [M]} p_{ji})/(q_i \min_{j \in [M]} p_{jk}) \le q_k / (q_i c)$, which further implies
    \begin{equation*}
    	\frac{q_{kl}}{q_{ij}} = \frac{f_k}{f_i} \frac{g_l}{g_j} \frac{p_{kl}}{p_{ij}} \le \frac{q_k q_l}{q_i q_j c^3} \implies q_{kl} \le \frac{q_{ij}}{q_i q_j c^3} q_k q_l,
    \end{equation*}
    thus leading to
    \begin{equation*}
    	1 = \sum_{k, l = 1}^{M} q_{kl} \le \frac{q_{ij}}{q_i q_j c^3} \sum_{k, l = 1}^{M} q_k q_l = \frac{q_{ij}}{q_i q_j c^3} \implies q_{ij} \ge q_i q_j c^3.
    \end{equation*}
    This concludes the proof.
\end{proof}

\begin{lem}\label{lem:cont_info_proj}
	Under the assumptions of Lemma~\ref{lem:info_proj_cond}, define
	\begin{align*}
			F((q_i)_{i\le M}) :=  \min_{\qq} 
			\Big\{D_{\sKL} \left( \qq \Vert \pp \right) \;\;
			\mbox{\rm subj. to} \ \;\; \sum_{j=1}^{M} q_{ij} =  \sum_{j=1}^{M} q_{ji} = q_i, \ \forall i \in [M] \Big\}\, .
	\end{align*}
	Then $F((q_i)_{i\le M})$
	 is a Lipschitz function of $(q_i)_{i \in [M]}$. Moreover, the Lipschitz constant is upper bounded by a function of $c = \min_{i,j \in [M]} p_{ij}$.
\end{lem}
\begin{proof}
According to Lemma~\ref{lem:info_proj_cond} (a), we have $q_{ij} = f_i g_j p_{ij}$, $ \forall i, j \in [M]$ where $\ff$ and $\bgg$ satisfy Eq.~\eqref{eq:cons_f_g}, and
\begin{equation*}
	D_{\sKL} \left( \qq \Vert \pp \right) = \sum_{i, j=1}^{M} q_{ij} \log \frac{q_{ij}}{p_{ij}} = \sum_{i, j=1}^{M} q_{ij} \left( \log f_i + \log g_j \right) = \sum_{i=1}^{M} q_i \left( \log f_i + \log g_i \right)
\end{equation*}
is a function of $(q_i)_{i \in [M]}$. Hence, for each $i \in [M]$, we have
\begin{equation*}
	\frac{\partial D_{\sKL} (\qq \Vert \pp)}{\partial q_i} = \log f_i + \log g_i + \sum_{j=1}^{M} q_j \left( \frac{1}{f_j} \frac{\partial f_j}{\partial q_i} + \frac{1}{g_j} \frac{\partial g_j}{\partial q_i} \right).
\end{equation*}
According to Eq.~\eqref{eq:cons_f_g}, it follows that
\begin{equation*}
	1 = \frac{\partial f_i}{\partial q_i} \sum_{j=1}^{M} g_j p_{ij} + f_i \sum_{j=1}^{M} \frac{\partial g_j}{\partial q_i} p_{ij}, \ 0 = \frac{\partial f_k}{\partial q_i} \sum_{j=1}^{M} g_j p_{kj} + f_k \sum_{j=1}^{M} \frac{\partial g_j}{\partial q_i} p_{kj}, \ \text{for} \ k \neq i,
\end{equation*}
thus leading to
\begin{align*}
	1 = & \sum_{k=1}^{M} \frac{\partial f_k}{\partial q_i} \sum_{j=1}^{M} g_j p_{kj} + \sum_{k=1}^{M} f_k \sum_{j=1}^{M} \frac{\partial g_j}{\partial q_i} p_{kj} = \sum_{k=1}^{M} \frac{\partial f_k}{\partial q_i} \sum_{j=1}^{M} g_j p_{kj} + \sum_{j=1}^{M} \frac{\partial g_j}{\partial q_i} \sum_{k=1}^{M} f_k p_{kj} \\
	\stackrel{(i)}{=} & \sum_{k=1}^{M} \frac{\partial f_k}{\partial q_i} \frac{q_k}{f_k} + \sum_{j=1}^{M} \frac{\partial g_j}{\partial q_i} \frac{q_j}{g_j} = \sum_{j=1}^{M} q_j \left( \frac{1}{f_j} \frac{\partial f_j}{\partial q_i} + \frac{1}{g_j} \frac{\partial g_j}{\partial q_i} \right),
\end{align*}
where $(i)$ is due to Eq.~\eqref{eq:cons_f_g}. This further implies that
\begin{equation*}
	\frac{\partial D_{\sKL} (\qq \Vert \pp)}{\partial q_i} = \log f_i + \log g_i + 1 = \log \frac{q_{ii}}{p_{ii}} + 1 \le \log \frac{1}{p_{ii}} + 1 \le \log \frac{1}{c} + 1,
\end{equation*}
which completes the proof Lemma~\ref{lem:cont_info_proj}.
\end{proof}

The lemma below establishes the completeness of multivariate Hermite polynomials in $L^2 (\R^m, \gamma^m)$, whose proof can be found in \cite[Prop. 13]{rahman2017wiener}:
\begin{lem}\label{lem:mul_her_com}
	Let $\{ {\rm He}_n \}_{n \ge 0}$ be the sequence of univariate Hermite polynomials, then the functions
	\begin{equation*}
		\xx \mapsto \prod_{i = 1}^{m} {\rm He}_{n_i} (x_i), \ n_i \ge 0, \ \forall i \in [m]
	\end{equation*}
    consist a complete orthonormal basis of $L^2 (\R^m, \gamma^m)$.
\end{lem}

\end{document}